\documentclass[twoside,11pt]{article}

\usepackage{booktabs}
\usepackage{multirow}
\usepackage{jmlr2e}
\usepackage{xcolor}
\usepackage{bm}
\usepackage{mathtools}
\usepackage{algorithm}
\usepackage{algorithmic}
\usepackage{cancel}
\usepackage{subfigure}
\usepackage{amsmath}
\usepackage{nicefrac}

\DeclareMathOperator*{\argmin}{arg\,min}
\DeclareMathOperator*{\argmax}{arg\,max}
\DeclareMathOperator{\proj}{proj}

\newcommand{\ie}{\emph{i.e.}}
\newcommand{\eg}{\emph{e.g.}}
\newcommand{\opt}{^{\star}}

\ShortHeadings{Policy Gradient for Robust MDPs}{Wang, Ho, and Petrik}
\firstpageno{1}

\begin{document}

\title{Policy Gradient for Robust Markov Decision Processes}

\author{\name Qiuhao Wang \email jerrison.wang@my.cityu.edu.hk \\
       \name Shaohang Xu \email shaohanxu2@cityu.edu.hk\\
       \name Chin Pang Ho \email clint.ho@cityu.edu.hk \\
       \addr Department of Data Science\\
       City University of Hong Kong\\
       Kowloon Tong, Hong Kong
       \AND
       Marek Petrik \email mpetrik@cs.unh.edu\\
       \addr Department of Computer Science\\
       University of New Hampshire \\
       Durham, NH, USA, 03861\\
       }

\editor{EDITOR}

\maketitle

\begin{abstract}
We develop a generic policy gradient method with the global optimality guarantee for robust Markov Decision Processes (MDPs). While policy gradient methods are widely used for solving dynamic decision problems due to their scalable and efficient nature, adapting these methods to account for model ambiguity has been challenging, often making it impractical to learn robust policies. This paper introduces a novel policy gradient method, Double-Loop Robust Policy Mirror Descent (DRPMD), for solving robust MDPs. DRPMD employs a general mirror descent update rule for the policy optimization with adaptive tolerance per iteration, guaranteeing convergence to a globally optimal policy. We provide a comprehensive analysis of DRPMD, including new convergence results under both direct and softmax parameterizations, and provide novel insights into the inner problem solution through Transition Mirror Ascent (TMA). Additionally, we propose innovative parametric transition kernels for both discrete and continuous state-action spaces, broadening the applicability of our approach. Empirical results validate the robustness and global convergence of DRPMD across various challenging robust MDP settings.
\end{abstract}


\section{Introduction}\label{sec:intro}

Markov decision processes~(MDPs) serve as a fundamental model in dynamic decision-making and reinforcement learning~\citep{puterman2014markov,sutton2018reinforcement,meyn2022control}. Classical MDPs typically consider stochastic environments and assume that the model parameters, such as transition probabilities, are precisely known. However, in most real-world applications, such as finance~\citep{sun2023reinforcement} and healthcare~\citep{goh2018data}, model parameters are estimated from noisy and limited observation data, and these estimation errors may lead to policies that perform poorly in practice. Using the idea of robust optimization, Robust MDPs (RMDPs) allow one to compute policies that exhibit resilience to parameter errors~\citep {iyengar2005robust,nilim2005robust}. RMDPs do not assume precise knowledge of transition probabilities but instead allow these probabilities to vary within a specific \emph{ambiguity set}, aiming to seek a policy that is optimal under the worst-case plausible realization of the transition probabilities~\citep{xu2006robustness,mannor2012lightning,hanasusanto2013robust,tamar2014scaling,delgado2016real}. Compared to MDPs, the performance of RMDPs is less sensitive to the parameter errors that arise when estimating the transition kernel from empirical data, as is often the case in reinforcement learning~\citep{xu2009parametric,ICML2012Petrik_283,ghavamzadeh2016safe}. 

The prevailing literature on RMDPs predominantly assumes rectangular ambiguity sets, which constrain the errors in the transition probabilities independently for each state~\citep{wiesemann2013robust,ho2021partial,Panaganti2021}. This assumption is crucial for RMDPs to maintain traceability, ensuring that an optimal policy of the RMDP can be computed using dynamic programming~\citep{iyengar2005robust,nilim2005robust,kaufman2013robust,ho2021partial}. In contrast, RMDPs with general ambiguity sets are NP-hard to solve~\citep{wiesemann2013robust}. The simplest rectangularity assumption, known as $(s,a)$-rectangularity, permits the adversarial nature to independently select the worst transition probability for each state and action pair. However, this $(s,a)$-rectangularity assumption can be overly restrictive, often yielding conservative policies. This paper, therefore, explores the more general $s$-rectangular ambiguity set~\citep{le2007robust,wiesemann2013robust,derman2021twice,wang2022geometry}, which allows the adversarial nature to choose transition probabilities without observing the action. Our results also readily extend to other notions of rectangularity, including $k$-rectangular~\citep{mannor2016robust} and $r$-rectangular RMDPs~\citep{goyal2022robust}. While rectangularity assumptions have been extensively studied particularly in the tabular setting~\citep{iyengar2005robust,nilim2005robust,le2007robust,kaufman2013robust,wiesemann2013robust,ho2018fast,Behzadian2021,ho2021partial,ho2022robust}, such rectangular ambiguity sets can be still quite restrictive in modeling ambiguity, particularly in large-scale or continuous contexts.

Policy gradient techniques have gained significant popularity in reinforcement learning (RL) due to their remarkable empirical performance and flexibility in large and complex domains~\citep{silver2014deterministic,xu2014reinforcement}. By parameterizing policies, policy gradient methods exhibit scalability across large state and action spaces, and these methods can also easily leverage generic optimization techniques~\citep{konda1999actor,bhatnagar2009natural,petrik2014raam,pirotta2015policy,schulman2015trust,schulman2017proximal,behzadian2021fast}. Recent studies show that many policy gradient algorithms can guarantee global optimality in tabular MDPs, even when optimizing non-convex objectives~\citep{mei2020global,agarwal2021theory,bhandari2021linear,xiao2022convergence,yuan2022general}. Despite the importance of policy gradient methods in RL, the development of gradient-based approaches for solving RMDPs, with provable optimal robustness guarantees, is still an open area of research.

The main goal of this paper is to address the limitations of the existing RMDP literature by developing a computationally tractable algorithm for RMDPs with provable convergence guarantees, applicable to both the tabular and large-scale continuous settings. Our framework comprises four components, each of which represents a novel contribution.

Our first contribution is \emph{Double-Loop Robust Policy Mirror Descent} (DRPMD), a new generic policy gradient scheme for solving RMDPs. DRPMD accommodates various commonly used parameterized policy classes, such as softmax parameterization, Gaussian parameterization, and neural policy classes, allowing for more flexible policy representations. Inspired by double-loop algorithms designed for saddle point problems~\citep{jin2020local,luo2020stochastic,razaviyayn2020nonconvex,zhang2020single}, DRPMD features two nested loops: an \emph{outer loop} that updates policies, and an \emph{inner loop} that approximately computes the worst-case transition probabilities.

As our second contribution, we prove that DRPMD is guaranteed to converge to a globally optimal policy with linearly increasing step sizes. While assuming an oracle for solving the inner problem, our proposed algorithm achieves a fast global convergence rate of $\mathcal{O}(\epsilon^{-1})$ for general RMDPs under both direct parameterization and softmax parameterization, even in non-rectangular cases. To address the potentially high computational costs of solving the inner loop optimally, we propose an adaptive schedule for reducing approximation errors, which is sufficient to ensure convergence to the optimal solution and enhance the algorithm's efficiency. To the best of our knowledge, DRPMD is the first gradient-based method for RMDPs with softmax parameterization that offers a provable convergence guarantee. For special $s$-rectangular RMDPs with direct parameterization, DRPMD achieves an even faster rate of $\mathcal{O}(\log(\epsilon^{-1}))$, achieving the best iteration complexity among policy gradient methods for rectangular RMDPs.

Our third contribution is a novel gradient-based algorithm for solving the inner maximization problems in $s$-rectangular RMDPs, named \emph{Transition Mirror Ascent} (TMA). While the outer loop closely resembles exact policy mirror descent updates in classical MDPs~\citep{xiao2022convergence}, the inner loop must optimize over an infinite number of transition probabilities in the ambiguity set. TMA employs mirror ascent updates to compute the inner worst-case transition kernel and enjoys a proven fast linear convergence rate, comparable to the best-known convergence result provided by value-based inner solution methods that rely on the contraction operator~\citep{iyengar2005robust,nilim2005robust,wiesemann2013robust,ho2021partial}.

As our fourth contribution, we propose two innovative and complete transition parameterizations to enhance the scalability of TMA. The first is inspired by the analytical form of the worst-case transition in KL-divergence constrained $(s,a)$-rectangular RMDPs, scaling well in RMDPs with large-scale state spaces and function approximation. The second is characterized by the Gaussian Mixture model, adapting well to RMDPs with continuous state spaces and function approximation. To facilitate learning these transition parameters without expensive gradient computations, we develop a stochastic variant of the TMA method, named \emph{Monte-Carlo Transition Mirror Ascent} (MCTMA).

Together, our contributions comprise a comprehensive robust policy gradient framework for solving RMDPs. Compared to an earlier conference version of this work~\citep{wang2023policy}, we introduce several significant enhancements in this paper. First, the DRPMD method generalizes the outer projected gradient descent update from DRPG by employing a more flexible mirror descent update, which also ensures provably faster convergence rates through refined difference analysis techniques for various parameterized policy classes. Notably, the softmax policy parameterization explored in this paper has not been previously studied in the context of gradient-based algorithms for RMDPs. Second, we extend the inner solution method to a more versatile TMA, outperforming prior projected gradient techniques by providing a faster convergence guarantee. Third, while our previous work included a transition parameterization, we now introduce updated entropy parametric transitions and propose a new Gaussian mixture transition parameterization to improve the scalability of TMA. Lastly, we present MCTMA, a stochastic variant that enhances the practicality of our approach for real-world applications. These advancements collectively represent a significant improvement over the prior work.

The remainder of the paper is organized as follows. We summarize relevant prior work in Section~\ref{sec:RelaWork} and review the basic notation and fundamental results of nominal MDPs and RMDPs in Section~\ref{sec:setup}. Section~\ref{sec:DRPMD} describes the outer loop of our proposed DRPMD algorithm and demonstrates its global convergence guarantee. The algorithms and transition parameterizations for addressing the inner loop are detailed in Section~\ref{sec:Inner-Rec}. Finally, in Section~\ref{sec:Numerical}, we present experimental results that illustrate the effective empirical performance of our algorithms. 

We use the following notation throughout the paper. Regular lowercase letters (\eg, $p$) denote scalars, boldface lowercase letters (\eg, $\bm{p}$) represent vectors, and boldface uppercase letters (\eg, $\bm{X}$) denote matrices. Indexed values are printed in bold for vectors and in the regular font for scalars. Specifically, $p_{i}$ refers to the $i$-th component of a vector $\bm{p}$, whereas $\bm{p}_{i}$ is the $i$-th vector of a sequence of vectors. All vector inequalities are understood to hold component-wise. Calligraphic letters and uppercase Greek letters (\eg, $\mathcal{X}$ and $\Xi$) are reserved for sets. The symbol $\bm{e}$ denotes a vector of all ones of the size appropriate to the context. The set $\mathbb{R}$ represents real numbers, and the set $\mathbb{R}_{+}$ represents non-negative real numbers. The probability simplex in $\mathbb{R}^S_+$ is denoted as $\Delta^{S}$. 
 For vectors, we use $\|\cdot\|$ to denote the $l_{2}$-norm. For a differentiable function $h(\bm{x},\bm{y})$, we use $\nabla_{\bm{x}}h(\bm{x},\bm{y})$ to denote the partial derivative of $h$ with respect to $\bm{x}$.

\section{Related Work}\label{sec:RelaWork}

Robust MDPs with rectangular uncertainty sets are typically tackled using value-based methods, which compute the optimal policy's value function by solving the robust Bellman equation~\citep{iyengar2005robust,nilim2005robust,kaufman2013robust,wiesemann2013robust,ho2021partial}. Subsequent research has extended these methods to sample-based approaches~\citep{roy2017reinforcement,tessler2019action,badrinath2021robust,wang2021online,liu2022distributionally,panaganti2022robust,panaganti2022sample,zhou2024natural}. Additionally, approximate dynamic programming (ADP) techniques~\citep{powell2007approximate} have been extended to approximate robust Bellman updates in value-based methods~\citep{tamar2014scaling,zhou2021finite,ma2022distributionally}. The application of ADP to value-based methods facilitates the use of function approximation to address the curse of dimensionality~\citep{roy2017reinforcement,badrinath2021robust,kose2021risk}, offering an efficient approach for improving algorithm scalability.

In addition to advancements in value-based methods, there has been growing interest in developing gradient-based methods in non-robust settings. Policy gradient~\citep{williams1992simple,sutton1999policy} and their extensions~\citep{konda1999actor,kakade2001natural,schulman2015trust} have demonstrated success in various applications. Although the surge of interest in policy gradient methods in RL, the theoretical understanding of convergence behavior remains limited to local optima and stationary points. It was not until recently that the global optimality of various policy gradient methods was established~\citep{mei2020global,agarwal2021theory,bhandari2021linear,li2021softmax,cen2022fast,xiao2022convergence,yuan2022general,bhandari2024global}. Stochastic policy gradient methods, which estimate first-order information via samples, have also been proposed, with studies focusing on both sample and iteration complexity~\citep{shani2020adaptive,xu2020improving,lan2023policy}.

Despite these advancements, gradient-based methods for solving RMDPs remain largely unexplored. A recent work by~\cite{wang2022policy} proposes a policy gradient method for solving RMDPs with a particular $(s,a)$-rectangular linear contamination ambiguity set. While this algorithm is compellingly simple, it is limited to the R-contamination set, which has been shown to be equivalent to ordinary MDPs with a reduced discount factor~\citep{wang2023policy}. Another related work by~\cite{li2022first} develops an extended mirror descent method for solving RMDPs. Our DRPMD algorithm shares similarities with their outer policy update but improves upon it in several respects. Specifically, their results are confined to $(s,a)$-rectangular RMDPs and appear to be challenging to generalize to $s$-rectangular sets. In contrast, our algorithm converges to the globally optimal policy for $s$-rectangular RMDPs. Moreover, their method assumes solving the inner loop optimally in each policy update, which can be computationally expensive, whereas DRPMD introduces a decreasing adaptive tolerance sequence without compromising convergence. Another noteworthy work by~\cite{zhou2024natural} explores a data-driven robust natural actor-critic algorithm to address $(s,a)$-rectangular RMDPs. Their approach introduces two well-constructed ambiguity sets and leverages function approximation techniques to efficiently manage large-scale robust RL problems.

Regarding RMDPs with the $s$-rectangular ambiguity set, \cite{kumar2023towards} propose an algorithm achieving the $\mathcal{O}(\epsilon^{-1})$ convergence rate; however, it relies on a strong assumption that the robust objective function is smooth, which may not hold for many ambiguity sets~\citep{lin2024single}. Another related work by~\cite{kumar2024policy} considers the ball-constrained ambiguity set, provides a closed-form expression of the worst-case transition kernel, and proposes a robust policy gradient method. Concurrently, \cite{li2023policy} introduces a double-loop algorithm for solving $s$-rectangular RMDPs, achieving the same $\mathcal{O}(\epsilon^{-4})$ convergence rate as~\cite{wang2023policy}. This work~\citep{li2023policy} also introduces an inner solution method for the $s$-rectangular RMDPs inner problem, though it converges at a slower global rate of $\mathcal{O}(\epsilon^{-2}_{\bm{\pi}})$ compared to TMA. Interestingly, \cite{li2023policy} also addresses RMDPs with a non-rectangular ambiguity set by employing projected Langevin dynamics, a Monte Carlo method for solving the inner problem. 

It is worth noting that all aforementioned related works feature a double-loop structure, alternately updating the outer policy and solving the inner problem optimally, whether by using analytical worst-case transition kernel for a specific ambiguity set~\citep{wang2022policy,kumar2024policy} or by assuming the existence of an oracle for solving the inner problem~\citep{li2022first,kumar2023towards}. From this point of view, our DRPMD method generalizes and extends these approaches. Apart from the double-loop methods, \cite{lin2024single} proposes the first single-loop robust policy gradient method for $s$-rectangular RMDPs with a global optimality guarantee. Compared to existing methods with a double-loop structure, their single-loop method avoids the costly computation required for the search of an inner approximation, offering better computational efficiency.

While our present paper exclusively focuses on RMDPs, it is worth mentioning that there is an active line of research studying a related model, called distributionally robust MDPs, which assumes the transition kernel is random and governed by an unknown probability distribution that lies in an ambiguity set \citep{ruszczynski2010risk,xu2010distributionally,shapiro2016rectangular,chen2019distributionally,clement2021first,shapiro2021distributionally,liu2022distributionally,yu2024fast}.

\section{The Model}\label{sec:setup}

This section reviews MDPs and RMDPs. We summarize their necessary notations and fundamental concepts that will be used throughout the paper.

\subsection{Markov Decision Processes}

A nominal MDP is specified by a tuple $\langle\mathcal{S},\mathcal{A},\bm{p},\bm{c}, \bm{\rho}\rangle$, where $\mathcal{S}=\{1,2,\cdots,S\}$ and $\mathcal{A}=\{1,2,\cdots, A\}$ are the finite state and action sets, respectively. The probability distribution of transiting from the current state $s$ to the next state $s'$ after taking an action $a$ is denoted as a vector $\bm{p}_{sa}\in\Delta^{S}$ and in a part of the transition kernel $\bm{p}:=(\bm{p}_{sa})_{s\in\mathcal{S},a\in\mathcal{A}}\in(\Delta^{S})^{S\times A}$. The cost of the aforementioned transition is denoted as $c_{sas'}$ for each $(s,a,s')\in\mathcal{S}\times\mathcal{A}\times\mathcal{S}$. We assume that $c_{s a s'} \in [0,1]$ for each $s,s'\in \mathcal{S}$ and $a\in \mathcal{A}$. This is without loss of generality because translating the costs by a constant or multiplying them by a positive scalar does not change the set of optimal policies~\citep{puterman2014markov}. The initial state is selected randomly according to the initial state distribution $\bm{\rho}\in\Delta^{S}$. 

A (stationary) randomized policy $\bm{\pi}:=(\bm{\pi}_s)_{s\in\mathcal{S}}, \bm{\pi}_s \in \Delta^A$ is a probability density function that
 prescribes taking action $a\in\mathcal{A}$ with probability $\pi_{sa}$ whenever the MDP is in state $s\in\mathcal{S}$. We use $\Pi = (\Delta^A)^S$ to denote the set of all randomized stationary policies. The total expected discounted cost of this MDP is defined as
 \begin{equation} \label{eq:return}
  J_{\bm{\rho}}(\bm{\pi},\bm{p})
  ~:=~
\mathbb{E}_{\bm{\pi},\bm{p},\tilde{s}_0 \sim\bm{\rho}}\left[\sum_{t=0}^{\infty} \gamma^{t}\cdot  c_{\tilde{s}_{t} \tilde{a}_{t} \tilde{s}_{t+1}}\right],
\end{equation}
where $\gamma\in (0,1)$ is the discount factor, reflecting how costs are weighted over time. The random variables are denoted by a tilde here and in the remainder of the paper. Here, $\mathbb{E}_{\bm{\pi},\bm{p},\tilde{s}_0 = s}$ represents the expectation of a dynamic where the action $\tilde{a}_t$ follows the distribution $\bm{\pi}_{\tilde{s}_t}$, the state $\tilde{s}_{t+1}$ follows the distribution $\bm{p}_{\tilde{s}_t \tilde{a}_t}$ and the initial state is taken as $s\in\mathcal{S}$. 

For each $s\in\mathcal{S}$, the \emph{value function} of the MDP is 
 \begin{equation*}
  v^{\bm{\pi},\bm{p}}_{s}  ~:=~
  \mathbb{E}_{\bm{\pi},\bm{p},\tilde{s}_0 = s}\left[\sum_{t=0}^{\infty} \gamma^{t}\cdot  c_{\tilde{s}_{t} \tilde{a}_{t} \tilde{s}_{t+1}}\right],
\end{equation*}
which represents the total expected discounted cost once the MDP starts from state $s$. Given the definition of the value function, we have $J_{\bm{\rho}}(\bm{\pi},\bm{p}) = \mathbb{E}_{\tilde{s}_{0}\sim\bm{\rho}}\left[v^{\bm{\pi},\bm{p}}_{\tilde{s}_{0}}\right]$. Similarly, we define the total expected discounted cost while the MDP takes an action $a$ at the initial state $\tilde{s}_0 = s$ as the \emph{action value function}, that is,
\begin{equation*}
  q^{\bm{\pi},\bm{p}}_{sa}  ~:=~
  \mathbb{E}_{\bm{\pi},\bm{p},\tilde{s}_{0}=s, \tilde{a}_{0} = a}\left[\sum_{t=0}^{\infty} \gamma^t \cdot c_{\tilde{s}_t \tilde{a}_t \tilde{s}_{t+1}}\right].
\end{equation*}
It is straightforward to compute the value function from the action value function as $v^{\bm{\pi},\bm{p}}_{s} = \sum_{a\in \mathcal{A}}\pi_{sa}q^{\bm{\pi},\bm{p}}_{sa}$~\citep{puterman2014markov,sutton2018reinforcement}. For analytical convenience, we also define the \emph{advantage function}, for each $s\in \mathcal{S}$ and $a\in \mathcal{S}$, as
\begin{equation*}
\psi^{\bm{\pi},\bm{p}}_{sa} ~:=~ q^{\bm{\pi},\bm{p}}_{sa} - v^{\bm{\pi},\bm{p}}_{s}.
\end{equation*}
To compute a policy $\bm{\pi}^\star$ that minimizes the expected sum of discounted costs $J_{\bm{\rho}}(\bm{\pi},\bm{p})$, policy search methods have been extensively studied in recent decades~\citep{williams1992simple,kakade2001natural,silver2014deterministic,schulman2015trust}. In policy search, the policy space $\Pi$ is typically parameterized by introducing a finite-dimensional vector $\bm{\theta}$, reducing the direct search for a good policy to a search over a chosen parameter set $\Theta$, resulting in the corresponding parameterized policy $\bm{\pi}^{\bm{\theta}}$ as we describe below. To facilitate our analysis, we overload notation and refer $J_{\bm{\rho}}(\bm{\pi}^{\bm{\theta}},\bm{p})$ as our general objective function, allowing us to frame the problem of solving the MDP as
\begin{equation}\label{eq:NMdps_para}
\min_{\bm{\theta} \in \Theta} \, J_{\bm{\rho}}(\bm{\pi}^{\bm{\theta}},\bm{p}).
\end{equation}
We now summarize several common policy parameterizations. In the basic \emph{direct parametrization}~\citep{shani2020adaptive,agarwal2021theory,bhandari2021linear,bhandari2024global}, we set for each $s\in \mathcal{S}$ and $a\in \mathcal{A}$,
\begin{equation}\label{def:dire-policy}
    \pi^{\bm{\theta}}_{sa} ~:=~ \theta_{sa},
\end{equation}
where $\bm{\theta}\in\Theta=\Pi=(\Delta^{A})^{S}$. For the finite action space $\mathcal{A}$, \emph{softmax parametrization}~\citep{mei2020global,agarwal2021theory,li2021softmax} is a natural option, where the softmax policy is defined for $\bm{\theta}\in\Theta = \mathbb{R}^{S\times A}$ as
\begin{equation}\label{def:Soft-policy}
    \pi^{\bm{\theta}}_{sa} ~:=~ \frac{\exp\left(\theta_{sa}\right)}{\sum_{a^{\prime} \in \mathcal{A}} \exp\left(\theta_{sa'}\right)}.
\end{equation}
For continuous action spaces $ \mathcal{A}$ with an infinite number of possible actions, the \emph{Gaussian parametrization}~\citep{zhao2011analysis, pirotta2013adaptive, papini2022smoothing} is widely used, where the Gaussian policies are defined as
\begin{equation}
\pi^{\bm{\theta}}_{sa} ~:=~ \frac{1}{\sigma \sqrt{2 \pi}} \exp \left(-\frac{(a-\mu(s;\bm{\theta}))^2}{2 \sigma^2}\right).
\end{equation}
Here, $\mu(s;\bm{\theta})$ is a state-dependent mean function of $\bm{\theta}\in\mathbb{R}^{d}$. Another widely used class of policy is the \emph{neural policies}, also called \emph{deep policies}~\citep{duan2016benchmarking}, often applied to large-scale RL problems. For example, a neural network can be used to parameterize the mean of a Gaussian policy, leading to $a \sim\mathcal{N}(g_{\bm{\theta}}(s),\sigma)$,
where $g_{\bm{\theta}}(s)$ is a neural network with weights $\bm{\theta}$.

Policy gradient methods apply first-order optimization techniques directly to update the policy parameter $\bm{\theta}$. The gradient $\nabla_{\bm{\theta}} J_{\bm{\rho}}(\bm{\pi}^{\bm{\theta}},\bm{p})$ of the objective in~\eqref{eq:NMdps_para} can be expressed analytically as~\citep{sutton1999policy,sutton2018reinforcement}:
\begin{equation}\label{eq:PG_them}
\frac{\partial J_{\bm{\rho}}(\bm{\pi}^{\bm{\theta}},\bm{p})}{\partial\bm{\theta}} ~=~ 
\mathbb{E}_{\tilde{s}\sim\bm{d}^{\bm{\pi}^{\bm{\theta}},\bm{p}}_{\rho},\tilde{a}\sim\bm{\pi}^{\bm{\theta}}_{s}}\left[\frac{\partial\log \pi^{\bm{\theta}}_{\tilde{s}\tilde{a}}}{\partial\bm{\theta}}\cdot q^{\bm{\pi}^{\bm{\theta}},\bm{p}}_{\tilde{s}\tilde{a}}\right].
\end{equation}
Here, $\bm{d}_{\bm{\rho}}^{\bm{\pi},\bm{p}}\in\Delta^{S}$ represents the (discounted) \emph{state occupancy measure}~\citep{puterman2014markov}, defined for $s'\in \mathcal{S}$ as
\begin{equation}\label{def:occu}
d_{\bm{\rho}}^{\bm{\pi},\bm{p}}(s') ~:=~ 
(1-\gamma) \cdot \mathbb{E}_{\bm{\pi},\bm{p}, \tilde{s}_0 \sim \bm{\rho}}\left[\sum_{t=0}^{\infty} \gamma^t \cdot \bm{1}\left\{ \tilde{s}_t = s' \right\} \right].
\end{equation}
Intuitively, the discounted state occupancy measure is interpreted as the expected total discounted visits of a particular state $s$ over a trajectory.

In the standard policy gradient method, the optimal policy $\bm{\pi}^{\star}$ can be computed approximately by iteratively updating the policy parameter $\bm{\theta}$.
Specifically, at the $(t+1)$-th step, the policy parameter is updated via projected gradient descent~\citep{Bertsekas2016}, that is,
\begin{align}\label{eq:VPG}
  \bm{\theta}_{t+1} 
  ~=~
\proj_{\Theta}\left(\bm{\theta}_{t} - \alpha_{t}\nabla_{\bm{\theta}}J_{\bm{\rho}}(\bm{\pi}^{\bm{\theta}_{t}},\bm{p}_{t})\right),
\end{align}
where $\proj_{\Theta}$ is the projection operator onto $\Theta$, and $\alpha_{t}$ is the step size. To implement the update rule~\eqref{eq:VPG}, one requires the exact gradient computation with full knowledge of the transition kernel and cost function. However, in most domains, the exact transition kernel and cost function are not known precisely and must be estimated from data. Unfortunately, with limited data, these estimation errors often result in policies that perform poorly when deployed.

\subsection{Robust Markov Decision Processes}

RMDPs generalize MDPs to account for model ambiguity, aiming to find policies that are resilient to ambiguity. More specifically, in an RMDP $\langle\mathcal{S},\mathcal{A},\mathcal{P},\bm{c},\bm{\rho}\rangle$, the transition kernel $\bm{p}$ is assumed to be adversarially chosen from an \emph{ambiguity set} of plausible values $\mathcal{P}\subseteq(\Delta^{S})^{S\times A}$~\citep{hanasusanto2013robust,wiesemann2013robust,petrik2014raam,Russell2019a,ho2021partial}. Our end goal is to find a robust policy that minimizes the expected total cost under the worst-case transition kernel from $\mathcal{P}$:
\begin{equation}\label{prob_RMDP}
\min_{\bm{\pi}\in\Pi}\max_{\bm{p}\in\mathcal{P}} \, J_{\bm{\rho}}(\bm{\pi},\bm{p}).
\end{equation}
Here, the outer minimization in~\eqref{prob_RMDP} reflects the agent’s objective, while the inner maximization represents the objective of the adversarial nature. By appropriately calibrating $\mathcal{P}$ to include the unknown true transition kernel, the optimal policy derived from~\eqref{prob_RMDP} can ensure reliable performance~\citep{Russell2019a,Behzadian2021,Panaganti2022}. 

Most standard methods for solving~\eqref{prob_RMDP}, such as robust value iteration~\citep{iyengar2005robust,nilim2005robust,wiesemann2013robust}, modified policy iteration~\citep{kaufman2013robust} and partial policy iteration~\citep{ho2021partial}, focus on estimating the robust values of policies and selecting policies based on these estimates. These method often assume the that the RMDP is rectangular~\citep{iyengar2005robust,nilim2005robust,wiesemann2013robust,ho2021partial}, where the ambiguity on transitions related to different states (state-action pairs) is uncoupled, and the adversary is allowed to select the worst possible realization for each state (state-action pair) unrelated to others. 

Two common classes of ambiguity sets are considered in this paper. We say an ambiguity set $\mathcal{P}$ is $(s,a)$-rectangular~\citep{iyengar2005robust,nilim2005robust,le2007robust} if it is a Cartesian product of sets $\mathcal{P}_{s,a}\subseteq\Delta^{S}$ for each state $s\in\mathcal{S}$ and action $a\in\mathcal{A}$, \ie,
\begin{equation*}
\mathcal{P}~:=~\left\{\bm{p}\in(\Delta^{S})^{S\times A}\mid\bm{p}_{s,a}\in\mathcal{P}_{s,a},\;\forall s\in\mathcal{S},a\in\mathcal{A}\right\},
\end{equation*}
whereas an ambiguity set $\mathcal{P}$ is $s$-rectangular~\citep{wiesemann2013robust} if it is defined as a Cartesian product of
sets $\mathcal{P}_{s}\subseteq(\Delta^{S})^{A}$, \ie,
\begin{equation*}
\mathcal{P}~:=~\left\{\bm{p}\in(\Delta^{S})^{S\times A}\mid\bm{p}_{s}=(\bm{p}_{s,a})_{a\in\mathcal{A}}\in\mathcal{P}_{s},\;\forall s\in\mathcal{S}\right\}.
\end{equation*}
Although the rectangularity is a standard assumption in most prior works on RMDPs,
it is not essential for describing or analyzing the proposed method DRPMD. We only require $\mathcal{P}$ to be compact to guarantee the existence of an inner maximum. This condition is satisfied by the majority of ambiguity sets considered in prior research, including $L_{1}$-ambiguity sets~\citep{ho2021partial}, $L_{\infty}$-ambiguity sets~\citep{givan2000bounded,Behzadian2021}, $L_{2}$-ambiguity sets~\citep{nilim2005robust}, and KL-ambiguity sets~\citep{iyengar2005robust,nilim2005robust}. 
While rectangularity benefits the development of algorithms for solving the inner maximization, it is not strictly necessary (see~\cite{li2023policy}).

From an optimization perspective, the optimal policy $\bm{\pi}^{\star}$ for this RMDP is the solution $(\bm{\pi}^{\star},\bm{p}^{\star})$ of the global minimax problem~\eqref{prob_RMDP}, where $\bm{\pi}^{\star}$ minimizes the function $\Phi(\bm{\pi}):= \max_{\bm{p}\in\mathcal{P}}J_{\bm{\rho}}(\bm{\pi},\bm{p})$, and $\bm{p}^{\star}$ 
is the worst-case transition kernel that maximizes $J_{\bm{\rho}}(\bm{\pi}^{\star},\bm{p})$~\citep{jin2020local,luo2020stochastic,razaviyayn2020nonconvex,zhang2020single}. Thus, the problem of solving the RMDP is allowed also to be considered as solving the following equivalent problem
\begin{equation}\label{prob_RMDP2}
\min_{\bm{\pi}\in\Pi} \; \Big\{ \Phi(\bm{\pi})
  ~:=~
\max_{\bm{p}\in\mathcal{P}}J_{\bm{\rho}}(\bm{\pi},\bm{p}) \Big\}.
\end{equation}
When using parameterized policies, we overload the notation $\Phi$ to represent the inner maximization for parameterized policies as well, denoting $\Phi(\bm{\theta}):=\max_{\bm{p}\in\mathcal{P}}J_{\bm{\rho}}(\bm{\pi}^{\bm{\theta}},\bm{p})$, which captures the worst-case performance of the RMDP under a parameterized policy.

A natural generalization of policy gradient methods to the robust setting would be to simply solve~\eqref{prob_RMDP2} by gradient descent on the function $\Phi$. However, this is challenging because the function $\Phi$ may not be differentiable due to the inner maximization problem. Also, since $\Phi$ is neither convex nor concave, its subgradient might not exist either~\citep{nouiehed2019solving,lin2020gradient}. These complications motivate the need for the double-loop iterative scheme, which we propose for solving RMDPs in Section~\ref{sec:DRPMD}.

\section{Double-Loop Robust Policy Mirror Descent}\label{sec:DRPMD}

In this section, we present a policy gradient approach for solving RMDPs. As the main contribution of this section, we demonstrate that our algorithm guarantees a globally optimal solution to the optimization problem in~\eqref{prob_RMDP2}, despite the objective function $\Phi$ being neither convex nor concave. Additionally, we establish that our algorithm offers the first global convergence guarantee for addressing RMDPs using the more practical and widely adopted softmax parameterization. This contrasts with prior works that focus exclusively on direct parameterization~\citep{li2022first,wang2022policy,li2023policy,kumar2024policy}. For the time being, we assume the existence of an oracle capable of solving the inner maximization problem. We will provide further discussions and algorithms for addressing the inner problem in Section~\ref{sec:Inner-Rec}.

We refer to our proposed policy gradient scheme as \emph{Double-Loop Robust Policy Gradient} (DRPMD), summarized in Algorithm~\ref{alg:DRPMD}. The term ``double-loop'' aligns with established terminology in the game theory literature~\citep{nouiehed2019solving,thekumparampil2019efficient,jin2020local,zhang2020single}, indicating that the algorithm operates within two nested loops. Specifically, DRPMD iteratively searches for an optimal policy in~\eqref{prob_RMDP2} by taking steps along the policy gradient. At each iteration $t$, Algorithm~\ref{alg:DRPMD} performs an inner update to seek an approximate worst-case transition kernel $\bm{p}_{t}$, solving the inner maximization problem to a specified precision $\epsilon_t$. Once $\bm{p}_t$ is computed, DRPMD proceeds to minimize $J_{\bm{\rho}}(\bm{\pi}^{\bm{\theta}}, \bm{p}_t)$ while incorporating a distance term $D(\bm{\theta},\bm{\theta}_{t})$, ensuring that the updated parameter $\bm{\theta}_{t+1}$ remains close to $\bm{\theta}_{t}$. It is worth emphasizing that the only first-order information utilized by DRPMD to solve~\eqref{prob_RMDP2} is the partial derivative $\nabla_{\bm{\theta}}J_{\bm{\rho}}(\bm{\pi}^{\bm{\theta}},\bm{p})$. As a result, the non-differentiability of $\Phi(\bm{\theta})$ does not impede the implementation of DRPMD.

\begin{algorithm}[t]
\caption{Double-Loop Robust Policy Mirror Descent (DRPMD)}
\label{alg:DRPMD}
\begin{algorithmic}
\STATE {\bfseries Input:} initial policy parameters $\bm{\theta}_{0}$, step size sequence $\{\alpha_{t}\}_{t\geq0}$, number of iterations $T$, tolerance sequence  $\{\epsilon_{t}\}_{t\geq0}$ such that $\epsilon_{t+1}\leq\gamma\epsilon_{t}$
\FOR{$t = 0,1,\dots,T-1$}
\STATE Find $\bm{p}_{t}$ so that $\displaystyle J_{\bm{\rho}}(\bm{\pi}^{\bm{\theta}_{t}},\bm{p}_{t}) \geq \max_{\bm{p}\in\mathcal{P}} J_{\bm{\rho}}(\bm{\pi}^{\bm{\theta}_{t}},\bm{p}) - \epsilon_{t}$.
\STATE Set $\displaystyle\bm{\theta}_{t+1} \gets  \argmin_{\bm{\theta}\in\Theta}\left\{ \alpha_{t}\langle \nabla_{\bm{\theta}} J_{\bm{\rho}}(\bm{\pi}^{\bm{\theta}_{t}},\bm{p}_{t}), \bm{\theta}\rangle + D(\bm{\theta},\bm{\theta}_{t})\right\}$.
\ENDFOR
\STATE {\bfseries Output:} $\bm{\pi}^{\bm{\theta}_{t^{\star}}} \in \{\bm{\pi}^{\bm{\theta}_{0}},\dots,\bm{\pi}^{\bm{\theta}_{T-1}}\}$ s.t. $\displaystyle J_{\bm{\rho}}(\bm{\pi}^{\bm{\theta}_{t^{\star}}},\bm{p}_{t^{\star}}) = \min_{t' \in \{0,\dots,T-1\}} J_{\bm{\rho}}(\bm{\pi}^{\bm{\theta}_{t'}},\bm{p}_{t'})$.
\end{algorithmic}
\end{algorithm}

When the tolerances $\{\epsilon_{t}\}_{t\geq0}$ are appropriately chosen, DRPMD can efficiently compute a near-optimal policy with guaranteed global convergence. The adaptive tolerance sequence $\{\epsilon_{t}\}_{t\geq0}$ is inspired by prior work on algorithms for RMDPs \citep{ho2021partial}. The subsequent convergence analysis offers further guidance on suitable choices for $\epsilon_t$. This adaptive approach not only ensures convergence but also significantly accelerates the algorithm, as demonstrated by our experimental results in Section~\ref{sec:Numerical}. In contrast, such an adaptive scheme is not considered in other policy gradient algorithms studied within the context of zero-sum games and robust MDPs~\citep{nouiehed2019solving,thekumparampil2019efficient}. 

It is important to note that single-loop algorithms have been explored in the game theory literature as an alternative to double-loop methods. These single-loop algorithms interleave updates for the inner and outer optimization problems~\citep{mokhtari2020unified,zhang2020single}. While this interleaving can enhance speed, it may lead to instability and oscillations, which one can mitigate by using two-scale step size updates~\citep{heusel2017gans,daskalakis2020independent,russel2020robust,lin2024single}. Our focus is on double-loop algorithms due to their conceptual clarity and favorable empirical performance.

In Section~\ref{subsec:opt-DP}, we show that our scheme is guaranteed to converge to the global solution when direct parameterization is considered, and then, in Section~\ref{subsec:opt-SP}, the global optimality of Algorithm~\ref{alg:DRPMD} for softmax parameterization is further reported. To the best of our knowledge, this is the first generic robust policy gradient algorithm with global convergence guarantees.   

\subsection{Global Optimality: Direct
Parameterization} \label{subsec:opt-DP}

We commence our analysis with the direct policy parametrization, where the parameters are the policies themselves, \ie, $\theta_{sa} = \pi^{\bm{\theta}}_{sa}$ (see~\eqref{def:dire-policy}), and constrained in the simplex, \ie, $\Theta=\Pi$. Our algorithm employs a form of mirror descent based on proximal minimization with respect to a Bregman divergence~\citep{nemirovskij1983problem,beck2003mirror} as the update rule of outer parameters, termed the \emph{policy mirror descent step}.

Following the derivations of~\cite{shani2020adaptive}, DRPMD updates the policy at each iteration using a distance function $D$ defined in terms of dynamically weighted Bregman divergences:
\begin{equation} \label{eq:md-definition}
\begin{aligned} 
  \bm{\pi}_{t+1}
  &~\in~
    \argmin_{\bm{\pi}\in\Pi}\left\{ \alpha_{t}\langle \nabla_{\bm{\pi}}J_{\bm{\rho}}(\bm{\pi}_{t},\bm{p}_{t}), \bm{\pi}\rangle + D(\bm{\pi},\bm{\pi}_{t})\right\}\\
  &~=~ \argmin_{\bm{\pi}\in\Pi}\left\{ \alpha_{t}\langle \nabla_{\bm{\pi}}J_{\bm{\rho}}(\bm{\pi}_{t},\bm{p}_{t}), \bm{\pi}\rangle + \frac{1}{1-\gamma}B_{\bm{d}_{\bm{\rho}}^{\bm{\pi}_{t},\bm{p}_{t}}}(\bm{\pi},\bm{\pi}_{t})\right\},
\end{aligned}
\end{equation}
where $\alpha_{t}>0$ is the step size and $B_{\bm{\mu}}(\bm{\pi},\bm{\pi}')$ is defined as a weighted Bregman divergence function, that is for any $\bm{\mu}\in\Delta^{S}$,
\begin{equation*}
B_{\bm{\mu}}(\bm{\pi},\bm{\pi}')~:=~\mathbb{E}_{s\sim\bm{\mu}}\left[B(\bm{\pi}_{s}, \bm{\pi}^{\prime}_{s})\right] ~=~ \sum_{s\in\mathcal{S}}\mu_{s}B(\bm{\pi}_{s}, \bm{\pi}^{\prime}_{s}).
\end{equation*}
Here $B(\bm{\pi}_{s}, \bm{\pi}^{\prime}_{s})$ denotes the Bregman divergence between policy $\bm{\pi}_{s}$ and $\bm{\pi}'_{s}$, defined as
\begin{equation}\label{def:bregman}
 B(\bm{\pi}_{s},\bm{\pi}'_{s}) ~:=~ h(\bm{\pi}_{s}) - h(\bm{\pi}'_{s}) -\langle\nabla h(\bm{\pi}'_{s}), \bm{\pi}_{s}-\bm{\pi}'_{s}\rangle,
\end{equation}
for some $h \colon \Delta^{A} \rightarrow \mathbb{R}$, a continuously differentiable and strictly convex function known as the distance-generating function. The two widely used Bregman divergence include:
\begin{itemize}
\item \emph{Squared Euclidean (SE) distance:} generated by the squared $l_{2}$-norm, \ie,
\begin{equation*}
    h(\bm{\pi}_{s}) ~:=~ \frac{1}{2}\left\|\bm{\pi}_{s}\right\|^{2}, \quad B(\bm{\pi}_{s},\bm{\pi}'_{s}) ~:=~  \frac{1}{2}\left\|\bm{\pi}_{s} - \bm{\pi}'_{s}\right\|^{2}.
\end{equation*}
\item \emph{Kullback-Leibler (KL) divergence:} generated by the negative entropy, \ie,
\begin{equation*}
    h(\bm{\pi}_{s}) ~:=~ \sum_{a\in\mathcal{A}}\pi_{sa}\log \pi_{sa}, \quad B(\bm{\pi}_{s},\bm{\pi}'_{s}) ~:=~  \sum_{a\in\mathcal{A}}\pi_{sa}\log \frac{\pi_{sa}}{\pi'_{sa}}.
\end{equation*}
\end{itemize}
In the direct parametrization, the gradient $\nabla_{\bm{\pi}}J_{\bm{\rho}}(\bm{\pi}_{t},\bm{p}_{t})$ has a well-known analytical form~\citep{agarwal2021theory,bhandari2021linear} given for $s\in \mathcal{S}, a\in \mathcal{A}$ by
\begin{equation}\label{eq:sec3_lem3.1}
    \frac{\partial J_{\bm{\rho}}(\bm{\pi},\bm{p})}{\partial \pi_{sa}} ~=~ \frac{1}{1-\gamma} \cdot d_{\bm{\rho}}^{\bm{\pi},\bm{p}}(s)\cdot q^{\bm{\pi},\bm{p}}_{sa}.
\end{equation}
Using the notation defined above, the policy mirror descent update on $\bm{\pi}_{t}:= (\bm{\pi}_{t,s})_{s\in\mathcal{S}} \in(\Delta^{A})^{S}$ can be represented as
\begin{equation*}
\bm{\pi}_{t+1} ~=~ \argmin_{\bm{\pi}\in\Pi}\left\{ \frac{1}{1-\gamma}\sum_{s\in\mathcal{S}}d_{\bm{\rho}}^{\bm{\pi}_{t},\bm{p}_{t}}(s)\left(\alpha_{t}\langle \bm{q}^{\bm{\pi}_{t},\bm{p}_{t}}_{s}, \bm{\pi}_{s}\rangle + B(\bm{\pi}_{s
},\bm{\pi}_{t,s})\right)\right\},
\end{equation*}
which can be further decoupled to multiple independent mirror descent updates across states, taking the form as
\begin{equation*}
\bm{\pi}_{t+1,s} ~=~ \argmin_{\bm{\pi}_{s}\in\Delta^{A}}\left\{ \alpha_{t}\langle \bm{q}^{\bm{\pi}_{t},\bm{p}_{t}}_{s}, \bm{\pi}_{s}\rangle + B(\bm{\pi}_{s
},\bm{\pi}_{t,s})\right\},\quad\forall s\in\mathcal{S}.
\end{equation*}
It is worth noting that while SE distance is chosen, the mirror descent updates on the policy reduce to projected gradient updates, as applied in the work by~\cite{wang2023policy}:
\begin{equation}\label{eq:PolicyUpdate-DRPG}
    \bm{\pi}_{t+1,s} ~=~ \proj_{\Delta^{A}}\left(\bm{\pi}_{t,s} - \alpha_{t}\nabla_{\bm{\pi}_{s}}J_{\bm{\rho}}(\bm{\pi}_{t},\bm{p}_{t})\right),\quad \forall s\in\mathcal{S}.
\end{equation}
In the rest of this subsection, our focus then shifts towards analyzing the convergence behavior of DRPMD with this more general mirror descent outer updates. 

The following lemma characterizes the performance of each policy update of DRPMD, encompassing both the SE distance and KL divergence without sacrificing rigor, serving as a crucial step towards establishing global convergence. 
\begin{lemma}\label{lem:3point}
For any $\bm{y}\in\Delta^{A}$ and any $s\in\mathcal{S}$, we have
\begin{equation}\label{eq:3-point-descent}
\alpha_{t}\langle \bm{q}^{\bm{\pi}_{t},\bm{p}_{t}}_{s}, \bm{\pi}_{t+1,s}-\bm{y}\rangle + B(\bm{\pi}_{t+1,s},\bm{\pi}_{t,s}) \leq B(\bm{y},\bm{\pi}_{t,s}) - B(\bm{y},\bm{\pi}_{t+1,s}). 
\end{equation}
\end{lemma}
\begin{proof}
    By the definition of $\bm{\pi}_{t+1,s}$ in~\eqref{eq:md-definition} and its derivative in~\eqref{eq:sec3_lem3.1}, we know for any $s\in\mathcal{S}$ that
    \begin{align*}
    \bm{\pi}_{t+1,s} &=\argmin_{\bm{\pi}_{s}\in\Delta^{A}}\left\{ \alpha_t\langle \bm{q}^{\bm{\pi}_{t},\bm{p}_{t}}_{s}, \bm{\pi}_{s}\rangle + B(\bm{\pi}_{s
},\bm{\pi}_{t,s})\right\}\\
&= \argmin_ {\bm{z}\in\mathbb{R}^{A}}\left\{ \alpha_{t}\langle \bm{q}^{\bm{\pi}_{t},\bm{p}_{t}}_{s}, \bm{z}\rangle + \mathbb{I}_{\Delta^{A}}(\bm{z})+  B(\bm{z},\bm{\pi}_{t,s})\right\},
    \end{align*}
where $\mathbb{I}$ is the indicator function (see Definition~\ref{def:indi}). From the optimality of $\bm{\pi}_{t+1,s}$ in~\eqref{eq:md-definition} and using~\cite[Theorem~6.12]{rockafellar2009variational}, we have
\begin{align*}
    0 &\in \partial\left\{\alpha_{t}\langle \bm{q}^{\bm{\pi}_{t},\bm{p}_{t}}_{s}, \bm{z}\rangle + \mathbb{I}_{\Delta^{A}}(\bm{z})+  B(\bm{z},\bm{\pi}_{t,s})\right\}\Big\vert_{\bm{z}= \bm{\pi}_{t+1,s}}\\
    \Longleftrightarrow\;\;\;\; 0 &\in \alpha_{t} \bm{q}^{\bm{\pi}_{t},\bm{p}_{t}}_{s} +  \mathcal{N}_{\Delta^{A}}(\bm{\pi}_{t+1,s})+  \nabla_{\bm{z}}B(\bm{z},\bm{\pi}_{t,s})\Big\vert_{\bm{z}= \bm{\pi}_{t+1,s}}.
\end{align*}
Here, $\mathcal{N}_{\mathcal{X}}(\bm{x})$ denotes the normal cone of $\mathcal{X}$ at $\bm{x}$. Therefore, the above equation implies that, for any $\bm{y}\in\Delta^{A}$,
\begin{align*}
\alpha_{t}\langle\bm{q}^{\bm{\pi}_{t},\bm{p}_{t}}_{s}, \bm{y}-\bm{\pi}_{t+1,s}\rangle + \langle\nabla_{\bm{z}}B(\bm{z},\bm{\pi}_{t,s})\Big\vert_{\bm{z}= \bm{\pi}_{t+1,s}},\bm{y}-\bm{\pi}_{t+1,s}\rangle\geq0.
\end{align*}
From the definition of Bregman divergence~\eqref{def:bregman}, we have that
\begin{align*}
\nabla_{\bm{z}}B(\bm{z},\bm{\pi}_{t,s}) = \nabla h(\bm{z}) - \nabla h(\bm{\pi}_{t,s}). 
\end{align*}
Thus, we can obtain the following identity
\begin{align*}
B(\bm{y},\bm{\pi}_{t,s}) - B(\bm{\pi}_{t+1,s},\bm{\pi}_{t,s}) - B(\bm{y},\bm{\pi}_{t+1,s}) &= \langle \nabla h(\bm{\pi}_{t+1,s}), \bm{y}-\bm{\pi}_{t+1,s}\rangle - \langle \nabla h(\bm{\pi}_{t,s}), \bm{y}-\bm{\pi}_{t+1,s}\rangle\\
&= \langle\nabla_{\bm{z}}B(\bm{z},\bm{\pi}_{t,s})\Big\vert_{\bm{z}= \bm{\pi}_{t+1,s}}, \bm{y}-\bm{\pi}_{t+1,s}\rangle,
\end{align*}
which leads us to the desired result.
\end{proof} 
We are now ready to demonstrate the global convergence of DRPMD through the use of linearly increasing step sizes. To this end, we show that our algorithm indeed exhibits a faster convergence rate of $\mathcal{O}(\epsilon^{-1})$, even when incorporating general Bregman divergences. 
\begin{theorem}\label{the:sublinear-dir-para}
Denote $\bm{\pi}_{t^{\star}}$ as the policy that Algorithm~\ref{alg:DRPMD} outputs. Suppose the step sizes $\{\alpha_{t}\}_{t\geq0}$ satisfy
\begin{equation}\label{def:sublinear-stepsize}
    \alpha_{t}\geq\frac{M\alpha_{t-1}}{1-\gamma},\quad\forall t\geq1,
\end{equation}
where $M:=\sup_{\bm{\pi}\in\Pi,\bm{p}\in\mathcal{P}}\left\|\nicefrac{\bm{d}_{\bm{\rho}}^{\bm{\pi},\bm{p}}}{\bm{\rho}}\right\|_{\infty} <\infty$ is finite whenever $\min_{s\in\mathcal{S}}\rho_{s}>0$. Then, for an initial tolerance $\epsilon_{0}\geq0$ and an initial policy $\bm{\pi}_{0}\in\Pi$, we have 
\begin{equation*}
    \Phi(\bm{\pi}_{t^\star}) - \min_{\bm{\pi}\in\Pi}\Phi(\bm{\pi})\leq \frac{1}{T}\left(\frac{2S}{(1-\gamma)^{2}}+ \frac{B_{\bm{d}^{\bm{\pi}^{\star},\bm{p}_{0}}_{\bm{\rho}}}(\bm{\pi}^{\star},\bm{\pi}_{0})}{\alpha_{0}(1-\gamma)}+\frac{\epsilon_{0}}{1-\gamma}\right).
\end{equation*}
\end{theorem}
Before we provide the proof, it is worth mentioning that the term $\left\|\nicefrac{\bm{d}_{\bm{\rho}}^{\bm{\pi},\bm{p}}}{\bm{\rho}}\right\|_{\infty}$ required in the theorem is known as \emph{distribution mismatch coefficient} and often assumed to be bounded~\citep{scherrer2014approximate,chen2019information,mei2020global,agarwal2021theory,leonardos2021global,xiao2022convergence}. 
\begin{proof}
We note that, by taking $\bm{y}$ in~\eqref{eq:3-point-descent} as the globally optimal policy $\bm{\pi}^{\star}_{s}$, we can obtain that for any $s\in\mathcal{S}$,
\begin{align}\label{eq:the-sub-mid1}
\langle \bm{q}^{\bm{\pi}_{t},\bm{p}_{t}}_{s}, \bm{\pi}_{t+1,s}-\bm{\pi}^{\star}_{s}\rangle + \frac{1}{\alpha_{t}}B(\bm{\pi}_{t+1,s},\bm{\pi}_{t,s}) &\leq  \frac{1}{\alpha_{t}}B(\bm{\pi}^{\star}_{s},\bm{\pi}_{t,s}) - \frac{1}{\alpha_{t}}B(\bm{\pi}^{\star}_{s},
\bm{\pi}_{t+1,s})\notag\\
\Longrightarrow\; \underbrace{\langle \bm{q}^{\bm{\pi}_{t},\bm{p}_{t}}_{s}, \bm{\pi}_{t,s}-\bm{\pi}^{\star}_{s}\rangle}_{\text{(A)}}+\underbrace{\langle \bm{q}^{\bm{\pi}_{t},\bm{p}_{t}}_{s}, \bm{\pi}_{t+1,s}-\bm{\pi}_{t,s}\rangle}_{\text{(B)}}  &\leq  \frac{1}{\alpha_{t}}B(\bm{\pi}^{\star}_{s},\bm{\pi}_{t,s}) - \frac{1}{\alpha_{t}}B(\bm{\pi}^{\star}_{s},\bm{\pi}_{t+1,s}).
\end{align}
To obtain the second inequality in~\eqref{eq:the-sub-mid1}, we drop the nonnegative term $B(\bm{\pi}_{t+1,s},\bm{\pi}_{t,s})$ on the left side of the first inequality and add $0$. For term (A), the definition of $\bm{p}_{t}$ in Algorithm~\ref{alg:DRPMD} implies that
    \begin{align*}
        \Phi(\bm{\pi}_{t}) - \Phi(\bm{\pi}^{\star}) &= 
   \max_{\bm{p}\in\mathcal{P}}J_{\bm{\rho}}(\bm{\pi}_{t},\bm{p})-\max_{\bm{p}\in\mathcal{P}}J_{\bm{\rho}}(\bm{\pi}^{\star},\bm{p})\\
   &\leq J_{\bm{\rho}}(\bm{\pi}_{t},\bm{p}_{t})-\max_{\bm{p}\in\mathcal{P}}J_{\bm{\rho}}(\bm{\pi}^{\star},\bm{p}) + \epsilon_{t}\\
   &\leq J_{\bm{\rho}}(\bm{\pi}_{t},\bm{p}_{t})-J_{\bm{\rho}}(\bm{\pi}^{\star},\bm{p}_{t}) + \epsilon_{t}.
    \end{align*}
Then, by further applying Lemma~\ref{lem:2nd-performance-diff}, we obtain that
\begin{equation}\label{eq:the-sub-mid2}
\Phi(\bm{\pi}_{t}) - \Phi(\bm{\pi}^{\star}) \leq \frac{1}{1-\gamma} \sum_{s\in\mathcal{S}} d_{\bm{\rho}}^{\bm{\pi}^{\star},\bm{p}_{t}}(s)\underbrace{\sum_{a\in\mathcal{A}}\left(\pi_{t,sa}-\pi^{\star}_{sa}\right) q^{\bm{\pi}_{t},\bm{p}_{t}}_{sa}}_{\text{(A)}} +\epsilon_{t}.
\end{equation}
As a result of Lemma~\ref{lem:3point}, while taking $\bm{y}$ as $\bm{\pi}_{t}$ in~\eqref{eq:3-point-descent}, we have for any $s\in\mathcal{S}$,
\begin{equation*}
\alpha_{t}\langle \bm{q}^{\bm{\pi}_{t},\bm{p}_{t}}_{s}, \bm{\pi}_{t+1,s}-\bm{\pi}_{t}\rangle \leq - B(\bm{\pi}_{t+1,s},\bm{\pi}_{t,s}) - \bm{\pi}_{t},\bm{\pi}_{t+1,s})\leq0.  
\end{equation*}
This observation leads to the following result that term (B) satisfies
\begin{equation}\label{eq:the-sub-mid3}
   \langle \bm{q}^{\bm{\pi}_{t},\bm{p}_{t}}_{s}, \bm{\pi}_{t+1,s}-\bm{\pi}_{t,s}\rangle\geq \frac{1}{1-\gamma}\langle \bm{e}, \bm{\pi}_{t+1,s}-\bm{\pi}_{t,s}\rangle, 
\end{equation}
where the fact $ 0\leq q^{\bm{\pi},\bm{p}}_{sa}\leq 1/(1-\gamma)$~\citep{wang2023policy} is applied. Then, by aggregating~\eqref{eq:the-sub-mid1} across states with weights set as the specific occupancy measure $d^{\bm{\pi}^{\star},\bm{p}_{t}}_{\bm{\rho}}(s)$, and making use of~\eqref{eq:the-sub-mid2} and~\eqref{eq:the-sub-mid3}, we obtain that
\begin{align*}
    \Phi(\bm{\pi}_{t}) - \Phi(\bm{\pi}^{\star})&\leq \frac{1}{1-\gamma}\mathbb{E}_{s\sim \bm{d}^{\bm{\pi}^{\star},\bm{p}_{t}}_{\bm{\rho}}}\left[\langle \bm{q}^{\bm{\pi}_{t},\bm{p}_{t}}_{s}, \bm{\pi}_{t,s}-\bm{\pi}_{t+1,s}\rangle\right]\\ 
    &\;\;\;\;+ \frac{1}{\alpha_{t}(1-\gamma)}\left(B_{\bm{d}^{\bm{\pi}^{\star},\bm{p}_{t}}_{\bm{\rho}}}(\bm{\pi}^{\star},\bm{\pi}_{t}) - B_{\bm{d}^{\bm{\pi}^{\star},\bm{p}_{t}}_{\bm{\rho}}}(\bm{\pi}^{\star},\bm{\pi}_{t+1})\right) +\epsilon_{t}\\
    &\leq \frac{1}{(1-\gamma)^{2}}\cdot\mathbb{E}_{s\sim \bm{d}^{\bm{\pi}^{\star},\bm{p}_{t}}_{\bm{\rho}}}\left[\langle \bm{e}, \bm{\pi}_{t,s}-\bm{\pi}_{t+1,s}\rangle\right]\\
    &\;\;\;\;+ \frac{1}{\alpha_{t}(1-\gamma)}\left(B_{\bm{d}^{\bm{\pi}^{\star},\bm{p}_{t}}_{\bm{\rho}}}(\bm{\pi}^{\star},\bm{\pi}_{t}) - B_{\bm{d}^{\bm{\pi}^{\star},\bm{p}_{t}}_{\bm{\rho}}}(\bm{\pi}^{\star},\bm{\pi}_{t+1})\right)+\epsilon_{t}.
\end{align*}
By summing the above inequality up over $t$, and we have
\begin{align*}
\sum^{T-1}_{t=0}\left(\Phi(\bm{\pi}_{t}) - \Phi(\bm{\pi}^{\star})\right)&\leq \frac{1}{(1-\gamma)^{2}}\cdot\sum^{T-1}_{t=0}\mathbb{E}_{s\sim \bm{d}^{\bm{\pi}^{\star},\bm{p}_{t}}_{\bm{\rho}}}\left[\langle \bm{e}, \bm{\pi}_{t,s}-\bm{\pi}_{t+1,s}\rangle\right]\\
&\;\;\;\;+ \sum^{T-1}_{t=0}\frac{1}{\alpha_{t}(1-\gamma)}\left(B_{\bm{d}^{\bm{\pi}^{\star},\bm{p}_{t}}_{\bm{\rho}}}(\bm{\pi}^{\star},\bm{\pi}_{t}) - B_{\bm{d}^{\bm{\pi}^{\star},\bm{p}_{t}}_{\bm{\rho}}}(\bm{\pi}^{\star},\bm{\pi}_{t+1})\right) +\sum^{T-1}_{t=0}\epsilon_{t}\\
&\leq \frac{1}{(1-\gamma)^{2}}\cdot\sum^{T-1}_{t=0}\mathbb{E}_{s\sim \bm{d}^{\bm{\pi}^{\star},\bm{p}_{t}}_{\bm{\rho}}}\left[\langle \bm{e}, \bm{\pi}_{t,s}-\bm{\pi}_{t+1,s}\rangle\right] + \frac{1}{\alpha_{0}(1-\gamma)}B_{\bm{d}^{\bm{\pi}^{\star},\bm{p}_{0}}_{\bm{\rho}}}(\bm{\pi}^{\star},\bm{\pi}_{0})\\
&\;\;\;\;+ \frac{1}{1-\gamma}\underbrace{\sum^{T-1}_{t=1}\left(\frac{1}{\alpha_{t}}B_{\bm{d}^{\bm{\pi}^{\star},\bm{p}_{t}}_{\bm{\rho}}}(\bm{\pi}^{\star},\bm{\pi}_{t}) - \frac{1}{\alpha_{t-1}}B_{\bm{d}^{\bm{\pi}^{\star},\bm{p}_{t-1}}_{\bm{\rho}}}(\bm{\pi}^{\star},\bm{\pi}_{t})\right)}_{\text{(C)}} +\sum^{T-1}_{t=0}\epsilon_{t}.
\end{align*}
For the step sizes $\{\alpha_{t}\}_{t\geq0}$ defined in~\eqref{def:sublinear-stepsize}, we have that
\begin{equation*}
    \alpha_{t}\geq \frac{M\alpha_{t-1}}{1-\gamma}\geq\frac{\alpha_{t-1}}{1-\gamma}\left\|\frac{\bm{d}^{\bm{\pi}^{\star},\bm{p}_{t}}_{\bm{\rho}}}{\bm{\rho}}\right\|_{\infty}\overset{(a)}{\geq}\alpha_{t-1}\left\|\frac{\bm{d}^{\bm{\pi}^{\star},\bm{p}_{t}}_{\bm{\rho}}}{\bm{d}^{\bm{\pi}^{\star},\bm{p}_{t-1}}_{\bm{\rho}}}\right\|_{\infty}\geq \alpha_{t-1}\left(\frac{d^{\bm{\pi}^{\star},\bm{p}_{t}}_{\bm{\rho}}(s)}{d^{\bm{\pi}^{\star},\bm{p}_{t-1}}_{\bm{\rho}}(s)}\right),
\end{equation*}
where the inequality (a) holds due to the fact that $d^{\bm{\pi},\bm{p}}_{\rho}(s)\geq(1-\gamma)\rho_{s}$~\citep{agarwal2021theory}. This also implies that term (C) should be smaller or equal to 0, since
\begin{align*}
    \frac{1}{\alpha_{t}}B_{\bm{d}^{\bm{\pi}^{\star},\bm{p}_{t}}_{\bm{\rho}}}(\bm{\pi}^{\star},\bm{\pi}_{t}) - \frac{1}{\alpha_{t-1}}B_{\bm{d}^{\bm{\pi}^{\star},\bm{p}_{t-1}}_{\bm{\rho}}}(\bm{\pi}^{\star},\bm{\pi}_{t}) = \sum_{s\in\mathcal{S}}\left(\frac{d^{\bm{\pi}^{\star},\bm{p}_{t}}_{\bm{\rho}}(s)}{\alpha_{t}} - \frac{d^{\bm{\pi}^{\star},\bm{p}_{t-1}}_{\bm{\rho}}(s)}{\alpha_{t-1}}\right)B(\bm{\pi}^{\star}_{s},\bm{\pi}_{t,s})\leq 0.
\end{align*}
We also note that 
\begin{equation*}
  \sum^{T-1}_{t=0}\mathbb{E}_{s\sim \bm{d}^{\bm{\pi}^{\star},\bm{p}_{t}}_{\bm{\rho}}}\left[\langle \bm{e}, \bm{\pi}_{t,s}-\bm{\pi}_{t+1,s}\rangle\right]\leq   \sum^{T-1}_{t=0}\sum_{s\in\mathcal{S}}\langle \bm{e}, \bm{\pi}_{t,s}-\bm{\pi}_{t+1,s}\rangle = \sum_{s\in\mathcal{S}}\langle \bm{e}, \bm{\pi}_{0,s}-\bm{\pi}_{T,s}\rangle\leq 2S,
\end{equation*}
where that last inequality holds due to facts that
\begin{equation*}
    \langle \bm{e}, \bm{\pi}_{0,s}-\bm{\pi}_{T,s}\rangle\leq\left\|\bm{e}\right\|_{\infty}\cdot\left\|\bm{\pi}_{0,s}-\bm{\pi}_{T,s}\right\|_{1}\leq2.
\end{equation*}
Therefore, we obtain that
\begin{align*}
   \sum^{T-1}_{t=0}\left(\Phi(\bm{\pi}_{t}) - \Phi(\bm{\pi}^{\star})\right)&\leq \frac{2S}{(1-\gamma)^{2}}+\frac{1}{\alpha_{0}(1-\gamma)}B_{\bm{d}^{\bm{\pi}^{\star},\bm{p}_{0}}_{\bm{\rho}}}(\bm{\pi}^{\star},\bm{\pi}_{0})+ \sum^{T-1}_{t=0}\epsilon_{t}.
\end{align*}
Since $\epsilon_{t+1}\leq\gamma\epsilon_{t}$ is assumed in~Algorithm~\ref{alg:DRPMD}, we have
\begin{equation*}
    \sum^{T-1}_{t=0}\epsilon_{t} \leq \sum^{\infty}_{t=0}\epsilon_{t} \leq \epsilon_{0}\cdot\left(1+\gamma+\gamma^{2}+\cdots\right) \leq\frac{\epsilon_{0}}{1-\gamma}.
\end{equation*}
Therefore, by combining the inequalities above, we can reach the desired result, that is, for the output $\bm{\pi}_{t^\star}$ of Algorithm~\ref{alg:DRPMD}, it satisfies
\begin{equation*}
    \Phi(\bm{\pi}_{t^\star}) - \min_{\bm{\pi}\in\Pi}\Phi(\bm{\pi})\leq \frac{1}{T}\sum^{T-1}_{t=0}\left(\Phi(\bm{\pi}_{t}) - \Phi(\bm{\pi}^{\star})\right)\leq \frac{1}{T}\left(\frac{2S}{(1-\gamma)^{2}}+ \frac{B_{\bm{d}^{\bm{\pi}^{\star},\bm{p}_{0}}_{\bm{\rho}}}(\bm{\pi}^{\star},\bm{\pi}_{0})}{\alpha_{0}(1-\gamma)}+\frac{\epsilon_{0}}{1-\gamma}\right).
\end{equation*}
\end{proof}
As a result of Theorem~\ref{the:sublinear-dir-para}, whenever carefully taking $\alpha_{0}\geq(1-\gamma)B_{\bm{d}^{\bm{\pi}^{\star},\bm{p}_{0}}_{\bm{\rho}}}(\bm{\pi}^{\star},\bm{\pi}_{0})$, we have
\begin{equation*}
\Phi(\bm{\pi}_{t^\star}) - \min_{\bm{\pi}\in\Pi}\Phi(\bm{\pi})\leq \frac{1}{T}\left(\frac{2S+1}{(1-\gamma)^{2}}+\frac{\epsilon_{0}}{1-\gamma}\right). 
\end{equation*}
In other words, the number of iterations required to achieve $\Phi(\bm{\pi}_{t^\star}) - \min_{\bm{\pi}\in\Pi}\Phi(\bm{\pi})\leq \epsilon$ is at most:
\begin{equation*}
    \frac{1}{\epsilon}\left(\frac{2S+1}{(1-\gamma)^{2}}+\frac{\epsilon_{0}}{1-\gamma}\right).
\end{equation*}
More specifically, for DRPMD with SE distance, \ie, $B(\bm{\pi}'_{s},\bm{\pi}_{s}) = \frac{1}{2}\|\bm{\pi}'_{s}-\bm{\pi}_{s}\|^{2}$, we have $B_{\bm{\mu}}(\bm{\pi}'_{s},\bm{\pi}_{s})\leq1$ for any $\bm{\mu}\in\Delta^{S}$, allowing for $\alpha_{0}\geq(1-\gamma)$; for DRPMD with KL divergence, if we choose the uniform initial policy, then $B_{\bm{\mu}}(\bm{\pi}'_{s},\bm{\pi}_{s})\leq\log A$ for any $\bm{\mu}\in\Delta^{S}$, so one could set $\alpha_{0}\geq(1-\gamma)\log A$.

Theorem~\ref{the:sublinear-dir-para} states that achieving an $\epsilon$-optimal robust policy using DRPMD requires iteration complexity of $\mathcal{O}(\epsilon^{-1})$. This result appears to be unique within the existing literature on first-order policy-based methods for solving generic RMDPs. To the best of our knowledge, Double-Loop Robust Policy Gradient (DRPG)~\citep{wang2023policy} was the first policy-based approach that has been demonstrated to solve $s$-rectangular RMDPs with general ambiguity sets. While DRPG employs a straightforward algorithmic structure and utilizes the projected policy gradient~\eqref{eq:PolicyUpdate-DRPG} for policy updates, its convergence analysis yields a slower rate of $\mathcal{O}(S^{4}A^{2}\epsilon^{-4})$. It should be noted that the convergence guarantee of their approach significantly depends on the state and action space sizes. In contrast, our result of $\mathcal{O}(S\epsilon^{-1})$ (with the choice of $\epsilon_{0}\ll S$) improves upon this, as it effectively reduces the dependence of convergence rate on the size of the state and action spaces. Moreover, while taking Bregman divergence as SE distance, DRPMD reduces to DRPG. 

Notably, a recent work by~\cite{li2022first} also explored the policy mirror descent method for solving robust MDPs; however, their analysis is limited to the setting of restrictive $(s,a)$-rectangular RMDPs only. Our analysis parallels that of~\cite{li2022first} without imposing any rectangularity assumption on the ambiguity set. Furthermore, we proceed to show that, by incorporating the more general $s$-rectangularity assumption, our DRPMD method can further exhibit a faster linear convergence rate, covering the results of~\cite{li2022first} under the $(s,a)$-rectangularity condition. Specifically, assuming the worst-case inner transition kernel can be computed exactly (\ie, $\bm{p}_{t}$ in DRPMD here maximizes the inner problem $J_{\bm{\rho}}(\bm{\pi}_{t},\bm{p})$), we show that policies generated by DRPMD approaches to the globally optimal policy with an $\mathcal{O}(\log(\epsilon^{-1}))$ convergence rate.
\begin{theorem}(Linear convergence for $s$-rectangular RMDPs)\label{the:linear-srec}
Suppose the step sizes $\{\alpha_{t}\}_{t\geq0}$ satisfy
\begin{equation}\label{def:linear-stepsize-s}
    \alpha_{t}\geq\frac{M}{1-\gamma}\cdot\left(1-\frac{(1-\gamma)^{2}}{M}\right)^{-1}\cdot\alpha_{t-1},\quad\forall t\geq1.
\end{equation}
Then, for any iteration $k$, DRPMD produces policy $\bm{\pi}_{k}$ for $s$-rectangular RMDPs satisfying
\begin{align*}
\Phi(\bm{\pi}_{k}) - \Phi(\bm{\pi}^{\star})\leq (1-\frac{(1-\gamma)^{2}}{M})^{k}\left(\Phi(\bm{\pi}_{0}) - \Phi(\bm{\pi}^{\star}) + \frac{(1-\gamma)B_{\bm{d}^{\bm{\pi}^{\star},\bm{p}_{0}}_{\bm{\rho}}}(\bm{\pi}^{\star},\bm{\pi}_{0})}{M-(1-\gamma)^{2}}\right).
\end{align*} 
\end{theorem}
The complete proof and a more detailed discussion are provided in the Appendix~\ref{subapp:DRPMD-linear}. Theorem~\ref{the:linear-srec} demonstrates a fast $\mathcal{O}(\log(\epsilon^{-1}))$ convergence rate of DRPMD for solving $s$-rectangular RMDPs, extending the existing result of~\cite{li2022first} from $(s,a)$-rectangular case to more general $s$-rectangular scenario. Our results for $s$-rectangular RMDPs improve upon the convergence rate of $\mathcal{O}(\epsilon^{-1})$ obtained by~\cite{kumar2023towards} and attain the same best-known iteration complexity as Partial Policy Iteration provided by~\cite{ho2021partial}.

\subsection{Global Optimality: Softmax
Parameterization} \label{subsec:opt-SP}
We now consider the softmax policy parameterization (see~\eqref{def:Soft-policy}). In practice, softmax parameterization offers a significant advantage over direct parameterization. This advantage arises from the explicit enforcement of the simplex constraint through the exponential mapping. As a result, the parameters $\bm{\theta}$ in softmax polices remain unconstrained, \ie, $\Theta=\mathbb{R}^{S\times A}$, enabling the use of standard unconstrained optimization algorithms. Specifically, for the softmax parameterization, DRPMD employs the SE distance in the policy update:
\begin{align}\label{eq:Softmax_outer_mirror}
    \bm{\theta}_{t+1} &~\in~
    \argmin_ {\bm{\theta}\in\mathbb{R}^{S\times A}}\left\{ \alpha_{t}\langle \nabla_{\bm{\theta}}J_{\bm{\rho}}(\bm{\pi}^{\bm{\theta}_{t}},\bm{p}_{t}), \bm{\theta}\rangle + D(\bm{\theta},\bm{\theta}_{t})\right\}\notag\\
    &~=~\argmin_ {\bm{\theta}\in\mathbb{R}^{S\times A}}\left\{ \alpha_{t}\langle \nabla_{\bm{\theta}}J_{\bm{\rho}}(\bm{\pi}^{\bm{\theta}_{t}},\bm{p}_{t}), \bm{\theta}\rangle + \frac{1}{2}\left\|\bm{\theta} - \bm{\theta}_{t}\right\|^{2}\right\}\notag\\
    &~=~ \bm{\theta}_{t} - \alpha_{t}\nabla_{\bm{\theta}}J_{\bm{\rho}}(\bm{\pi}^{\bm{\theta}_{t}},\bm{p}_{t}).
\end{align}
With softmax parameterization, DRPMD facilitates a more straightforward computation of the policy update step, as it involves taking simpler gradient descent steps compared to direct parameterization. This update rule aligns with existing literature on the convergence behavior of softmax parameterization in non-robust RL contexts~\citep{mei2020global,agarwal2021theory,li2021softmax,papini2022smoothing}. For softmax parameterization, the gradient $\nabla_{\bm{\theta}} J_{\bm{\rho}}(\bm{\pi}^{\bm{\theta}},\bm{p})$ takes the following particular form:
\begin{equation*}
    \frac{\partial J_{\bm{\rho}}(\bm{\pi}^{\bm{\theta}},\bm{p})}{\partial \theta_{sa}} ~=~ \frac{1}{1-\gamma} \cdot d_{\bm{\rho}}^{\bm{\pi}^{\bm{\theta}},\bm{p}}(s) \cdot\pi^{\bm{\theta}}_{sa}\cdot \psi^{\bm{\pi}^{\bm{\theta}},\bm{p}}_{sa},
\end{equation*}    
which has been well established in prior works~\citep{mei2020global,agarwal2021theory}; see also Lemma~\ref{lem:sec4_2_1}.

Despite the substantial empirical popularity of policy gradient methods with softmax parameterization in modern RL, there has been limited exploration of their performance in robust settings. Here, we demonstrate a positive result that our proposed DRPMD with softmax
parameterization indeed converges to the globally optimal policy for RMDPs. Before delving into the convergence behavior of DRPMD with this parameterization, we first introduce a fundamental result that outlines several desirable properties satisfied by $J_{\bm{\rho}}(\bm{\pi}^{\bm{\theta}},\bm{p})$.
\begin{lemma}\label{lem:sec4_2_2}
For the softmax parameterization, the objective function $J_{\bm{\rho}}(\bm{\pi}^{\bm{\theta}},\bm{p})$ is $L_{\bm{\theta}}$-Lipschitz and $\ell_{\bm{\theta}}$-smooth in $\bm{\theta}$ with $L_{\bm{\theta}}:=  \sqrt{2}/(1-\gamma)^{2}$ and $\ell_{\bm{\theta}}:= 8/(1-\gamma)^{3}$. The smoothness definitions are provided in Appendix~\ref{app:Opt-results} for completeness.
\end{lemma}
\begin{proof}
By the policy gradient theorem (Theorem~\ref{the:appA1-1}), we have that
    \begin{align*}
\nabla_{\bm{\theta}}J_{\bm{\rho}}(\bm{\pi}^{\bm{\theta}},\bm{p}) = \frac{1}{1-\gamma}    \mathbb{E}_{s\sim\bm{d}^{\bm{\pi}^{\bm{\theta}},\bm{p}}_{\rho}}\left[\sum_{a\in\mathcal{A}}\nabla_{\bm{\theta}}\pi^{\bm{\theta}}_{sa}\cdot q^{\bm{\pi}^{\bm{\theta}},\bm{p}}_{sa}\right]=\frac{1}{1-\gamma}\mathbb{E}_{s\sim\bm{d}^{\bm{\pi}^{\bm{\theta}},\bm{p}}_{\rho},a\sim\bm{\pi}^{\bm{\theta}}_{s}}\left[\nabla_{\bm{\theta}}\log\pi^{\bm{\theta}}_{sa}\cdot q^{\bm{\pi}^{\bm{\theta}},\bm{p}}_{sa}\right].
\end{align*}
We then turn to derive the bound of $\nabla_{\bm{\theta}}\log\pi^{\bm{\theta}}_{sa}$ for the softmax parameterization. We notice that for $\hat{s}\neq s$, $\frac{\partial\log\pi^{\bm{\theta}}_{sa}}{\partial\theta(\hat{s},\hat{a})} = 0$, and for $\hat{s}= s$, we have
\begin{align*}
\frac{\partial\log\pi^{\bm{\theta}}_{sa}}{\partial\theta(s,\hat{a})} &= \frac{\partial\left(\log\frac{\exp(\theta(s,a))}{\sum_{a^{\prime} \in \mathcal{A}} \exp(\theta(s,a))}\right)}{\partial\theta(s,\hat{a})}\\
&= \frac{\partial\left(\log(\exp(\theta(s,a))) - \log(\sum_{a^{\prime} \in \mathcal{A}} \exp(\theta(s,a)))\right)}{\partial\theta(s,\hat{a})}\\
&= \frac{\partial\log(\exp(\theta(s,a)))}{\partial\theta(s,\hat{a})} - \frac{\partial\log(\sum_{a^{\prime} \in \mathcal{A}} \exp(\theta(s,a)))}{\partial\theta(s,\hat{a})}\\
&=\left\{
\begin{aligned}
1 - \frac{\exp(\theta(s,a))}{\sum_{a^{\prime} \in \mathcal{A}} \exp(\theta(s,a))} & , & \hat{a} = a, \\
\frac{\exp(\theta(s,\hat{a}))}{\sum_{a^{\prime} \in \mathcal{A}} \exp(\theta(s,a))}\;\;\; & , & \hat{a} \neq a.
\end{aligned}
\right.\\
&= \left\{
\begin{aligned}
1 - \pi^{\bm{\theta}}_{sa} & , & \hat{a} = a, \\
- \pi^{\bm{\theta}}_{s\hat{a}}\;\;\; & , & \hat{a} \neq a.
\end{aligned}
\right.
\end{align*}
Therefore, we obtain that
\begin{equation*}
\left\|\nabla\log\pi^{\bm{\theta}}_{sa}\right\| = \left[\sum_{\hat{a}\in\mathcal{A}}\left(\frac{\partial\log\pi^{\bm{\theta}}_{sa}}{\partial\theta(s,\hat{a})}\right)^{2}\right]^{\frac{1}{2}} = \left[1 - 2\pi^{\bm{\theta}}_{sa} +\sum_{\hat{a}\in\mathcal{A}}(\pi^{\bm{\theta}}_{s\hat{a}})^{2}\right]^{\frac{1}{2}}\leq \sqrt{2},
\end{equation*}
and therefore we can obtain
\begin{equation*} \left\|\nabla_{\bm{\theta}}J_{\bm{\rho}}(\bm{\pi}^{\bm{\theta}},\bm{p})\right\|\leq\frac{\sqrt{2}}{(1-\gamma)^{2}}
\end{equation*}
by using the fact that $0\leq q^{\bm{\pi}^{\bm{\theta}},\bm{p}}_{sa}\leq 1/(1-\gamma)$. Regarding the smoothness property of $J_{\bm{\rho}}(\bm{\pi}^{\bm{\theta}},\bm{p})$, recent result on softmax parameterization~\citep[Lemma 7]{mei2020global} can directly show that $J_{\bm{\rho}}(\bm{\pi}^{\bm{\theta}},\bm{p})$ is smooth in $\bm{\theta}$ with $\ell_{\bm{\theta}}:=8/(1-\gamma)^{3}$.
\end{proof}
Next, we turn to derive the global convergence guarantee of DRPMD with softmax parameterization. To be specific, we aim to characterize and upper bound the difference in robust values between the parametric softmax policy $\Phi(\bm{\pi}^{\bm{\theta}_{t}})$ and the optimal robust policy $\Phi(\bm{\pi}^{\star})$. For simplicity, we assume that the inner worst-case transition kernel is attained exactly, though our results can also be extended to cases that include the tolerance sequence. The main result below demonstrates that Algorithm~\ref{alg:DRPMD} converges to an $\epsilon$-optimal policy within $\mathcal{O}(\epsilon^{-1})$ steps. 
\begin{theorem}\label{the:sublinear-soft-para}
    Consider the softmax parameterization~\eqref{def:Soft-policy} and let $\bm{\pi}^{\star}$ as the globally optimal policy of RMDPs with $\Phi^{\star}:=\Phi(\bm{\pi}^{\star})$. Denote $\{\bm{\theta}_{t}\}_{t\geq0}$ as the parameter sequence generated by Algorithm~\ref{alg:DRPMD} and introduce some constant $U>0$. Then, for the constant step size $\alpha:=1/\ell_{\bm{\theta}}$, we have
    \begin{equation}
        \Phi(\bm{\pi}^{\bm{\theta}_{t}}) - \Phi^{\star} \leq \frac{16M^{2}S}{(1-\gamma)^{5}U^{2}t}, \quad t\geq 0.
    \end{equation}
\end{theorem}
At a high level, the analysis of DRPMD with softmax parameterization differs from that of direct parameterization. Our approach is as follows. First, thanks to key technical results (\ie, Lemma~\ref{lem:sec4_2_1}), we establish a connection between the parameter sequence $\{\bm{\theta}_{t}\}_{t\geq0}$, generated by DRPMD, and a separate sequence $\{\hat{\bm{\theta}}_{k}\}_{k\geq0}$, generated by non-robust softmax policy gradient~\citep{li2021softmax}. Specifically, we observe that for any parameter $\bm{\theta}_{t}$ at iteration $t$, we can identify a corresponding point in the sequence $\{\hat{\bm{\theta}}_{k}\}_{k\geq0}$ such that this point also occurs at the same iteration $t$. Using this insight, we incorporate the convergence result of the softmax policy gradient, derived by~\cite{mei2020global} to derive our desired convergence result of DRPMD. In this analysis, $M>0$ is the upper bound of the distribution mismatch coefficient, and $U>0$ is a constant that provides a lower bound for the non-uniform Łojasiewicz degree, as is typical in the analysis of policy gradient methods~\citep{mei2020global,mei2021leveraging}. With these elements in place, we are now ready to present our proofs.
\begin{proof}
From the proof of Theorem~\ref{the:sublinear-dir-para}, we can bound the difference between the robust values $\Phi(\bm{\pi}^{\bm{\theta}_{t}})$ and $\Phi^{\star}$ as follow:
\begin{align*}
\Phi(\bm{\pi}^{\bm{\theta}_{t}}) - \Phi^{\star}
\leq J_{\bm{\rho}}(\bm{\pi}^{\bm{\theta}_{t}},\bm{p}_{t}) - J_{\bm{\rho}}(\bm{\pi}^{\star},\bm{p}_{t}),
\end{align*}
where $\bm{p}_{t}:= \operatorname{argmax}_{\bm{p}\in\mathcal{P}}J_{\bm{\rho}}(\bm{\pi}^{\bm{\theta}_{t}},\bm{p})$. Recall that, for any $0\leq t\leq T$, the $t$-th outer update in Algorithm~\ref{alg:DRPMD} with a constant step size is represented as
\begin{equation*}
      \bm{\theta}_{t+1}
  ~=~
  \bm{\theta}_{t} - \alpha\cdot\nabla_{\bm{\theta}} J_{\bm{\rho}}(\bm{\pi}^{\bm{\theta}_{t}},\bm{p}_{t}).
\end{equation*}
This update can be viewed as a specific step within the gradient descent algorithm applied to solve the non-robust MDP with the transition kernel $\bm{p}_{t}$. Specifically, while Lemma~\ref{lem:sec4_2_2} implies that $\nabla_{\bm{\theta}} J_{\bm{\rho}}(\bm{\pi}^{\bm{\theta}},\bm{p})$ is a $L_{\bm{\theta}}$-Lipschitz
and $\ell_{\bm{\theta}}$-Lipschitz continuous function, we apply Lemma~\ref{lem:PG_seq} to conclude that there exists a sequence of softmax policy parameters $\{\hat{\bm{\theta}}_{k}\}_{k\geq0}$, generated by using the update rule 
\begin{equation*}
      \hat{\bm{\theta}}_{k+1}
  ~=~
  \hat{\bm{\theta}}_{k} - \alpha\cdot\nabla_{\bm{\theta}} J_{\bm{\rho}}(\bm{\pi}^{\hat{\bm{\theta}}_{k}},\hat{\bm{p}}),
\end{equation*}
to solve a nominal MDP with the fixed $\hat{\bm{p}}:=\bm{p}_{t}$, while the $t$-th iterative point $\hat{\bm{\theta}}_{t}$ exactly equals to the $t$-th iterative point $\bm{\theta}_{t}$ from the parameter sequence $\{\bm{\theta}_{t}\}_{t\geq0}$ generated by DRPMD. By applying Theorem 4 in~\cite{mei2020global}, we find that using a constant step size $\alpha=1/\ell_{\bm{\theta}}$, the sequence $\{\hat{\bm{\theta}}_{k}\}_{k\geq0}$ generated by the softmax policy gradient satisfies 
\begin{equation*}
J_{\bm{\rho}}(\bm{\pi}^{\bm{\theta}_{t}},\bm{p}_{t}) - \min_{\bm{\pi}\in\Pi}J_{\bm{\rho}}(\bm{\pi},\bm{p}_{t})  = J_{\bm{\rho}}(\bm{\pi}^{\hat{\bm{\theta}}_{t}},\hat{\bm{p}}) - \min_{\bm{\pi}\in\Pi}J_{\bm{\rho}}(\bm{\pi},\hat{\bm{p}})  \leq \frac{16M^{2}S}{(1-\gamma)^{5}U_{\bm{p}_{t}}^{2}t}.
\end{equation*}
Here, $U_{\bm{p}_{t}}$ provides a lower bound of the non-uniform Łojasiewicz degree of the nominal MDP with transition kernel $\bm{p}_{t}$, defined as $U_{\bm{p}_{t}}:=\inf_{s\in\mathcal{S},k\geq1}\pi^{\hat{\bm{\theta}}_{k}}_{sa^{\star}_{\bm{p}_{t}}(s)}>0$, where $a^{\star}_{\bm{p}_{t}}(s)$ is simply defined as the unique action that the deterministic optimal policy of the nominal MDP selects at state $s\in\mathcal{S}$~\citep{mei2020global}. When the optimal action is non-unique, the definition of $U_{\bm{p}_{t}}$ need to be slightly modified. More details are refereed to~\cite{mei2020global}. Thus, by introducing $U:=\min_{\bm{p}\in\mathcal{P}}U_{\bm{p}}>0$, we can obtain the desired result, that is for all $0\leq t \leq T$,
\begin{align*}
\Phi(\bm{\pi}^{\bm{\theta}_{t}}) - \Phi^{\star}\leq J_{\bm{\rho}}(\bm{\pi}^{\bm{\theta}_{t}},\bm{p}_{t}) - \min_{\bm{\pi}\in\Pi}J_{\bm{\rho}}(\bm{\pi},\bm{p}_{t})\leq \frac{16M^{2}S}{(1-\gamma)^{5}U^{2}t},
\end{align*}
which completes the proof.
\end{proof}
To the best of our knowledge, DRPMD is the first policy gradient method to provide a convergence analysis specifically tailored to RMDPs with softmax parameterization. Thus, we conclude that DRPMD is a versatile algorithm, applicable to both direct and softmax parameterized policies. In the upcoming numerical section, we will further demonstrate the performance of DRPMD across the remaining two classes mentioned earlier, namely, the Gaussian policy and the neural policy classes. It is also important to highlight that DRPMD relies on an oracle that outputs at least one worst-case transition kernel for any given $\bm{\pi}^{\bm{\theta}}$. Solving the inner loop problem could still be NP-hard for non-rectangular setting~\citep{wiesemann2013robust}. In the subsequent section, we introduce several gradient-based algorithms to address the inner loop problem.

\section{Solving the Inner Loop}\label{sec:Inner-Rec}

So far, we have described the outline of DRPMD and established its global convergence for two common parameterizations. For $k$-th iteration in Algorithm~\ref{alg:DRPMD}, the transition kernel $\bm{p}_{k}$ is obtained by approximately solving the inner maximization problem with a fixed policy $\bm{\pi}^{\bm{\theta}_{k}}$:
\begin{equation}\label{prob_inRMDP}
    \max_{\bm{p}\in\mathcal{P}}  \left\{J_{\bm{\rho}}(\bm{\pi}^{\bm{\theta}_{k}},\bm{p}) =  \bm{\rho}^{\top}\bm{v}^{\bm{\pi}^{\bm{\theta}_{k}},\bm{p}}\right\}.
\end{equation} 
While the assumptions of boundedness and compactness are used to ensure the existence of maximum in problem~\eqref{prob_inRMDP}, solving this inner maximization problem still remains a significant computational challenge due to the non-convexity~\citep{wiesemann2013robust}. This section discusses several gradient-based approaches to solve the inner maximization problem, which we shall refer to as the \emph{robust policy evaluation problem}. It is important to emphasize that the convergence results outlined in Section~\ref{sec:DRPMD} are independent of the specific method chosen to address this robust policy evaluation problem. 

The rest of this section is structured as follows. Section~\ref{subsec:Inner_Global_conv} describes a deterministic gradient-based scheme for solving the robust policy evaluation problem and discusses the global optimality under the wildly-used rectangularity setting. Then, we briefly introduce the value-based approach and establish the connection between value-based and our gradient-based methods in Section~\ref{subsec:Connect_to_VB}. Besides, we further explore the utilization of parametric models for transition probabilities in Section~\ref{subsec:Transition_para}, allowing us to model the ambiguity of transitions beyond the rectangularity and handle large-scale problems by parameterizing the transition probabilities in lower dimension. Lastly, Section~\ref{subsec:StocTMA} develops the stochastic variants of the proposed gradient-based approach, where the exact information on
the transition gradient is not required.

\subsection{Transition Mirror Ascent with Global Optimality}\label{subsec:Inner_Global_conv}
In this subsection, a generic gradient-based algorithm is proposed in Algorithm~\ref{alg:transition_gradient}, named the Transition Mirror Ascent Method (TMA). This method solves the inner-loop robust policy evaluation problem with a global convergence guarantee, assuming a rectangular and convex ambiguity set. Note that the inner problem~\eqref{prob_inRMDP} could be viewed as a constrained non-concave maximization problem. To maximize $J_{\bm{\rho}}(\bm{\pi}^{\bm{\theta}_{k}},\bm{p})$, TMA intuitively updates the parameter by following its ascent direction within the feasible set through the mirror ascent step
\begin{align}
    \bm{p}_{t+1} ~\in~ \argmax_ {\bm{p}\in\mathcal{P}}\left\{\beta_{t}\langle \nabla_{\bm{p}}J_{\bm{\rho}}(\bm{\pi}^{\bm{\theta}_{k}},\bm{p}_{t}), \bm{p}\rangle - D(\bm{p},\bm{p}_{t})\right\}.
\end{align}
where distance-like function $D(\bm{p},\bm{p}_{t})$ generally is considered as the Bregman divergence.
\begin{algorithm}[t]
\caption{Transition Mirror Ascent Method (TMA)}
\label{alg:transition_gradient}
\begin{algorithmic}
\STATE {\bfseries Input:} Target fixed policy with parameter $\bm{\theta}_{k}$, initial transition $\bm{p}_{0}$, iteration time $T_{k}$, step size sequence $\{\beta_{t}\}_{t\geq0}$
\FOR{$t = 0,1,\dots,T_{k}-1$}
\STATE Set $\displaystyle\bm{p}_{t+1} \gets \argmax_{\bm{p}\in\mathcal{P}}\left\{\beta_{t}\langle \nabla_{\bm{p}}J_{\bm{\rho}}(\bm{\pi}^{\bm{\theta}_{k}},\bm{p}_{t}), \bm{p}\rangle - D(\bm{p},\bm{p}_{t})\right\}$.
\ENDFOR
\end{algorithmic}
\end{algorithm}
To implement TMA, one needs to obtain an exact gradient of the transition in every iteration, which is provided in the following lemma.
\begin{lemma}\label{lem:inner_gradient}
The partial derivative of $J_{\bm{\rho}}(\bm{\pi},\bm{p})$ has the explicit form for any $(s,a,s')\in\mathcal{S}\times\mathcal{A}\times\mathcal{S}$,  
\begin{equation*}
    \frac{\partial J_{\bm{\rho}}(\bm{\pi},\bm{p})}{\partial p_{sas'}} = \frac{1}{1-\gamma}\cdot d_{\bm{\rho}}^{\bm{\pi},\bm{p}}(s)\cdot\pi_{sa}\cdot\left(c_{sas'}+\gamma v^{\bm{\pi},\bm{p}}_{s'}\right).
\end{equation*}
\end{lemma}
\begin{proof}
    Notice that
\begin{equation*}
    \frac{\partial J_{\bm{\rho}}(\bm{\pi},\bm{p})}{\partial p_{sas'}} = \sum_{\hat{s} \in \mathcal{S}} \frac{\partial v^{\bm{\pi},\bm{p}}_{\hat{s}}}{\partial p_{sas'}} \rho_{\hat{s}}.
\end{equation*}
Then, we discuss $\frac{\partial v^{\bm{\pi},\bm{p}}_{\hat{s}}}{\partial p_{sas'}}$ over two cases: $\hat{s}\neq s$ and $\hat{s}= s$
\begin{align}
    \frac{\partial v^{\bm{\pi},\bm{p}}_{\hat{s}}}{\partial p_{sas'}}\Big\vert_{\hat{s}\neq s} &= \frac{\partial}{\partial p_{sas'}}\left[\sum_{\hat{a}}\pi_{\hat{s}\hat{a}}\sum_{\hat{s}' \in \mathcal{S}} p_{\hat{s}\hat{a}\hat{s}'}\left(c_{\hat{s}\hat{a}\hat{s}'}+\gamma v^{\bm{\pi},\bm{p}}_{\hat{s}'}\right)\right] = \gamma\sum_{\hat{a}}\pi_{\hat{s}\hat{a}}\sum_{\hat{s}' \in \mathcal{S}} p_{\hat{s}\hat{a}\hat{s}'} \frac{\partial v^{\bm{\pi},\bm{p}}_{\hat{s}'}}{\partial p_{sas'}},\label{eq:sec5-mid1}\\
    \frac{\partial v^{\bm{\pi},\bm{p}}_{\hat{s}}}{\partial p_{sas'}}\Big\vert_{\hat{s}= s} &= \gamma\sum_{\hat{a}}\pi_{s\hat{a}}\sum_{\hat{s}' \in \mathcal{S}} p_{s\hat{a}\hat{s}'} \frac{\partial v^{\bm{\pi},\bm{p}}_{\hat{s}'}}{\partial p_{sas'}} + \pi_{sa}\left(c_{sas'}+\gamma v^{\bm{\pi},\bm{p}}_{s'}\right)\label{eq:sec5-mid2}.
\end{align}
By condensing notations
\begin{align}\label{eq:sec5-mid3}
\sum_{\hat{a}}\pi_{\hat{s}\hat{a}} p_{\hat{s}\hat{a}\hat{s}'}=p^{\bm{\pi}}_{\hat{s}\hat{s}'}(1),\quad
    \sum_{\hat{s}'}p^{\bm{\pi}}_{\hat{s}\hat{s}'}(t-1)\cdot\sum_{\hat{a}}\pi_{\hat{s}'\hat{a}} p_{\hat{s}'\hat{a}\hat{s}''} = p^{\bm{\pi}}_{\hat{s}\hat{s}''}(t),
\end{align}
we notice that for $\hat{s}\neq s$,
\begin{align*}
     \frac{\partial v^{\bm{\pi},\bm{p}}_{\hat{s}}}{\partial p_{sas'}}\Big\vert_{\hat{s}\neq s} &= \gamma\sum_{\hat{s}' \neq s } p^{\bm{\pi}}_{\hat{s}\hat{s}'}(1) \frac{\partial v^{\bm{\pi},\bm{p}}_{\hat{s}'}}{\partial p_{sas'}} + \gamma\sum_{\hat{s}' = s } p^{\bm{\pi}}_{\hat{s}\hat{s}'}(1) \frac{\partial v^{\bm{\pi},\bm{p}}_{\hat{s}'}}{\partial p_{sas'}}\\
     &= \gamma\sum_{\hat{s}' \neq s } p^{\bm{\pi}}_{\hat{s}\hat{s}'}(1) \cdot \gamma\sum_{\hat{a}}\pi_{\hat{s}'\hat{a}}\sum_{\hat{s}'' \in \mathcal{S}} p_{\hat{s}'\hat{a}\hat{s}''} \frac{\partial v^{\bm{\pi},\bm{p}}_{\hat{s}''}}{\partial p_{sas'}} \\
     &\;\;\;\;+ \gamma p^{\bm{\pi}}_{\hat{s}s}(1) \cdot \left( \gamma\sum_{\hat{a}}\pi_{s\hat{a}}\sum_{\hat{s}' \in \mathcal{S}} p_{s\hat{a}\hat{s}'} \frac{\partial v^{\bm{\pi},\bm{p}}_{\hat{s}'}}{\partial p_{sas'}} + \pi_{sa}\left(c_{sas'}+\gamma v^{\bm{\pi},\bm{p}}_{s'}\right)\right)\\
     &= \gamma p^{\bm{\pi}}_{\hat{s}s}(1)\pi_{sa}\left(c_{sas'}+\gamma v^{\bm{\pi},\bm{p}}_{s'}\right) + \gamma^{2}\sum_{\hat{s}'} p^{\bm{\pi}}_{\hat{s}\hat{s}'}(2)\frac{\partial v^{\bm{\pi},\bm{p}}_{\hat{s}'}}{\partial p_{sas'}}.
\end{align*}
By recursively making use of~\eqref{eq:sec5-mid1},~\eqref{eq:sec5-mid2}, and~\eqref{eq:sec5-mid3}, we can further obtain that
\begin{equation*}
\frac{\partial v^{\bm{\pi},\bm{p}}_{\hat{s}}}{\partial p_{sas'}}\Big\vert_{\hat{s}\neq s} =   \sum_{t=1}^{\infty}\gamma^{t}p^{\bm{\pi}}_{\hat{s}s}(t)\pi_{sa}\left(c_{sas'}+\gamma v^{\bm{\pi},\bm{p}}_{s'}\right) = \sum_{t=0}^{\infty}\gamma^{t}p^{\bm{\pi}}_{\hat{s}s}(t)\pi_{sa}\left(c_{sas'}+\gamma v^{\bm{\pi},\bm{p}}_{s'}\right)  
\end{equation*}
The last equality is from the initial assumption $\hat{s} \neq s$, \ie, $p^{\bm{\pi}}_{\hat{s}s}(0)=0$. Similarly, for the case, $\hat{s}= s$ we have
\begin{equation*}
    \frac{\partial v^{\bm{\pi},\bm{p}}_{\hat{s}}}{\partial p_{sas'}}\Big\vert_{\hat{s}= s} = \sum_{t=0}^{\infty}\gamma^{t}p^{\bm{\pi}}_{ss}(t)\pi_{sa}\left(c_{sas'}+\gamma v^{\bm{\pi},\bm{p}}_{s'}\right).
\end{equation*}
Hence, the partial derivative for transition probability is obtained
\begin{align*}
    \frac{\partial J_{\bm{\rho}}(\bm{\pi},\bm{p})}{\partial p_{sas'}}  &= \frac{1}{1-\gamma}\left(\underbrace{(1-\gamma) \sum_{\hat{s}\in\mathcal{S}}\sum_{t=0}^{\infty} \gamma^{t} \rho_{\hat{s}}p^{\bm{\pi}}_{\hat{s}s}(t)}_{d_{\bm{\rho}}^{\bm{\pi},\bm{p}}(s)}\right)\pi_{sa}\left(c_{sas'}+\gamma v^{\bm{\pi},\bm{p}}_{s'}\right) \\
    &= \frac{1}{1-\gamma}d_{\bm{\rho}}^{\bm{\pi},\bm{p}}(s)\pi_{sa}\left(c_{sas'}+\gamma v^{\bm{\pi},\bm{p}}_{s'}\right).
\end{align*}
\end{proof}
Inspired by~\cite{wang2023policy}, the mirror ascent update can be simplified by incorporating different rectangularity conditions with appropriate weighted Bregman divergence $D(\bm{p},\bm{p}_{t})$.  Specifically, while introducing $g^{\bm{\pi},\bm{p}}_{sas'}:=c_{sas'}+\gamma v^{\bm{\pi},\bm{p}}_{s'}$, which is formally defined as the \emph{action-next-state value function} in~\cite{li2023policy}, we define the transition gradient update for $(s,a)$-rectangular RMDPs as
\begin{align*}
\bm{p}_{t+1} ~=~ \mathop{\arg\max}\limits_ {\bm{p}\in\mathcal{P}}\left\{\beta_{t}\langle \nabla_{\bm{p}}J_{\bm{\rho}}(\bm{\pi},\bm{p}_{t}), \bm{p}\rangle - \frac{1}{1-\gamma}B_{\langle\bm{\pi},\bm{d}^{\bm{\pi},\bm{p}_{t}}_{\bm{\rho}}\rangle}(\bm{p},\bm{p}_{t})\right\}.
\end{align*}
Here, $B_{\langle\bm{\pi},\bm{d}^{\bm{\pi},\bm{p}_{t}}_{\bm{\rho}}\rangle}(\bm{p},\bm{p}_{t})$ is a special weighted Bregman divergence, defined as
\begin{equation*}
B_{\langle\bm{\pi},\bm{d}^{\bm{\pi},\bm{p}_{t}}_{\bm{\rho}}\rangle}(\bm{p},\bm{p}_{t}) ~:=~ \sum_{s\in\mathcal{S}}d^{\bm{\pi},\bm{p}_{t}}_{\bm{\rho}}(s)\sum_{a\in\mathcal{A}}\pi_{sa}B (\bm{p}_{sa},\bm{p}_{t,sa}),  
\end{equation*}
where $B (\bm{p}_{sa},\bm{p}_{t,sa})$ is the Bregman divergence between $\bm{p}_{sa}$ and $\bm{p}_{t,sa}$, defined in~\eqref{def:bregman}.
Then, we can further decouple the above update across state-action pairs and represent it as
\begin{equation*}
    \bm{p}_{t+1,sa} ~=~ \mathop{\arg\max}\limits_ {\bm{p}_{sa}\in\mathcal{P}_{sa}}\left\{\beta_{t}\langle \bm{g}^{\bm{\pi},\bm{p}_{t}}_{sa}, \bm{p}_{sa}\rangle -B(\bm{p}_{sa},\bm{p}_{t,sa})\right\},\quad\forall s\in\mathcal{S},a\in\mathcal{A}.
\end{equation*}
Similarly, for $s$-rectangular RMDPs, the inner update is defined as
\begin{align*}
\bm{p}_{t+1} ~=~ \mathop{\arg\max}\limits_ {\bm{p}\in\mathcal{P}}\left\{\beta_{t}\langle \nabla_{\bm{p}}J_{\bm{\rho}}(\bm{\pi},\bm{p}_{t}), \bm{p}\rangle - \frac{1}{1-\gamma}B_{\bm{d}^{\bm{\pi},\bm{p}_{t}}_{\bm{\rho}}}(\bm{p},\bm{p}_{t})\right\},
\end{align*}
where
\begin{equation*}
B_{\bm{d}^{\bm{\pi},\bm{p}_{t}}_{\bm{\rho}}}(\bm{p},\bm{p}_{t}) ~=~ \sum_{s\in\mathcal{S}}d^{\bm{\pi},\bm{p}_{t}}_{\bm{\rho}}(s)B (\bm{p}_{s},\bm{p}_{t,s}).  
\end{equation*}
The above update can be decoupled across different states as
\begin{equation*}
    \bm{p}_{t+1,s} ~=~ \mathop{\arg\max}\limits_ {\bm{p}_{s}\in\mathcal{P}_{s}}\left\{\beta_{t}\langle \hat{\bm{g}}^{\bm{\pi},\bm{p}_{t}}_{s}, \bm{p}_{s}\rangle -B(\bm{p}_{s},\bm{p}_{t,s})\right\},\quad\forall s\in\mathcal{S},
\end{equation*}
where
\begin{equation*}
\hat{\bm{g}}^{\bm{\pi},\bm{p}_{t}}_{s} ~:=~ \left(\pi_{sa}\bm{g}^{\bm{\pi},\bm{p}_{t}}_{sa}\right)_{a\in\mathcal{A}}\subseteq \mathbb{R}^{S\times A}.    
\end{equation*}
We now discuss the convergence behavior of Algorithm~\ref{alg:transition_gradient}. Before we proceed to establish some general convergence properties of TMA, we first introduce the following lemmas, which can be viewed as variants of the celebrated performance difference lemma for non-robust MDPs.
\begin{lemma}[First performance difference lemma across transition kernels]\label{lem:1st-diff-perf-inner}
    For any transition kernels $\bm{p},\bm{p}'\in\mathcal{P}$ and policy $\bm{\pi}\in\Pi$, we have
    \begin{equation*}
J_{\bm{\rho}}(\bm{\pi},\bm{p})-J_{\bm{\rho}}(\bm{\pi},\bm{p}') =  \frac{1}{1-\gamma} \sum_{s\in\mathcal{S}} d_{\bm{\rho}}^{\bm{\pi},\bm{p}'}(s) \left(\sum_{a\in\mathcal{A}}\pi_{sa}\sum_{s'}\left(p_{sas'}-p'_{sas'}\right)g^{\bm{\pi},\bm{p}}_{sas'}\right).      
    \end{equation*}
\end{lemma}
\begin{proof}
    According to the definition of the value function, we have
\begin{align*} 
&v^{\bm{\pi},\bm{p}}_{s}-v^{\bm{\pi},\bm{p}'}_{s} \\
&= v^{\bm{\pi},\bm{p}}_{s}
-  \sum_{a\in\mathcal{A}}\pi_{sa}\sum_{s'}p'_{sas'}\left(c_{sas'}+\gamma v^{\bm{\pi},\bm{p}}_{s'}\right)
+\sum_{a\in\mathcal{A}}\pi_{sa}\sum_{s'}p'_{sas'}\left(c_{sas'}+\gamma v^{\bm{\pi},\bm{p}}_{s'}\right)-
v^{\bm{\pi},\bm{p}'}_{s}\\
&= \sum_{a\in\mathcal{A}}\pi_{sa}\sum_{s'}\left(p_{sas'}-p'_{sas'}\right)\left(c_{sas'}+\gamma v^{\bm{\pi},\bm{p}}_{s'}\right) + \gamma\sum_{a\in\mathcal{A}}\pi_{sa}\sum_{s'}p'_{sas'}\left(v^{\bm{\pi},\bm{p}}_{s'}-v^{\bm{\pi},\bm{p}'}_{s'}\right)\\
&= \sum_{t=0}^{\infty}\gamma^{t}\sum_{s'}p'^{\bm{\pi}}_{ss'}(t)\left(\sum_{a'}\pi_{s'a'}\sum_{s''}\left(p_{s'a's''}-p'_{s'a's''}\right)\left(c_{s'a's''}+\gamma v^{\bm{\pi},\bm{p}}_{s''}\right)\right).
\end{align*}
Then, we can obtain that
\begin{align*}
    J_{\bm{\rho}}(\bm{\pi},\bm{p})-J_{\bm{\rho}}(\bm{\pi},\bm{p}') &= \sum_{s\in\mathcal{S}}\rho_{s}\left( v^{\bm{\pi},\bm{p}}_{s}-v^{\bm{\pi},\bm{p}'}_{s}\right)\\
    &= \sum_{s\in\mathcal{S}}\rho_{s}\sum_{t=0}^{\infty}\gamma^{t}\sum_{s'}p'^{\bm{\pi}}_{ss'}(t)\left(\sum_{a'}\pi_{s'a'}\sum_{s''}\left(p_{s'a's''}-p'_{s'a's''}\right)\left(c_{s'a's''}+\gamma v^{\bm{\pi},\bm{p}}_{s''}\right)\right)\\
    &= \frac{1}{1-\gamma} \sum_{s\in\mathcal{S}} d_{\bm{\rho}}^{\bm{\pi},\bm{p}'}(s) \left(\sum_{a\in\mathcal{A}}\pi_{sa}\sum_{s'}\left(p_{sas'}-p'_{sas'}\right)\left(c_{sas'}+\gamma v^{\bm{\pi},\bm{p}}_{s'}\right)\right).
\end{align*}    
\end{proof}
\begin{lemma}[Second performance difference lemma across transition kernels]\label{lem:2nd-diff-perf-inner}
    For any transition kernels $\bm{p},\bm{p}'\in\mathcal{P}$ and policy $\bm{\pi}\in\Pi$, we have   \begin{equation*}
        J_{\bm{\rho}}(\bm{\pi},\bm{p})-J_{\bm{\rho}}(\bm{\pi},\bm{p}') = \frac{1}{1-\gamma} \sum_{s\in\mathcal{S}} d_{\bm{\rho}}^{\bm{\pi},\bm{p}}(s) \left(\sum_{a\in\mathcal{A}}\pi_{sa}\sum_{s'}\left(p_{sas'}-p'_{sas'}\right)g^{\bm{\pi},\bm{p}'}_{sas'}\right)
    \end{equation*}
\end{lemma}
\begin{proof}
    According to the definition of the value function, we have
\begin{align*} 
&v^{\bm{\pi},\bm{p}}_{s}-v^{\bm{\pi},\bm{p}'}_{s} \\
&= v^{\bm{\pi},\bm{p}}_{s}
-  \sum_{a\in\mathcal{A}}\pi_{sa}\sum_{s'}p_{sas'}\left(c_{sas'}+\gamma v^{\bm{\pi},\bm{p}'}_{s'}\right)
+\sum_{a\in\mathcal{A}}\pi_{sa}\sum_{s'}p_{sas'}\left(c_{sas'}+\gamma v^{\bm{\pi},\bm{p}'}_{s'}\right)-
v^{\bm{\pi},\bm{p}'}_{s}\\
&= \sum_{a\in\mathcal{A}}\pi_{sa}\sum_{s'}\left(p_{sas'}-p'_{sas'}\right)\left(c_{sas'}+\gamma v^{\bm{\pi},\bm{p}'}_{s'}\right) + \gamma\sum_{a\in\mathcal{A}}\pi_{sa}\sum_{s'}p_{sas'}\left(v^{\bm{\pi},\bm{p}}_{s'}-v^{\bm{\pi},\bm{p}'}_{s'}\right)\\
&= \sum_{t=0}^{\infty}\gamma^{t}\sum_{s'}p^{\bm{\pi}}_{ss'}(t)\left(\sum_{a'}\pi_{s'a'}\sum_{s''}\left(p_{s'a's''}-p'_{s'a's''}\right)\left(c_{s'a's''}+\gamma v^{\bm{\pi},\bm{p}'}_{s''}\right)\right).
\end{align*}
Then, we can obtain that
\begin{align*}
    J_{\bm{\rho}}(\bm{\pi},\bm{p})-J_{\bm{\rho}}(\bm{\pi},\bm{p}') &= \sum_{s\in\mathcal{S}}\rho_{s}\left( v^{\bm{\pi},\bm{p}}_{s}-v^{\bm{\pi},\bm{p}'}_{s}\right)\\
    &= \sum_{s\in\mathcal{S}}\rho_{s}\sum_{t=0}^{\infty}\gamma^{t}\sum_{s'}p^{\bm{\pi}}_{ss'}(t)\left(\sum_{a'}\pi_{s'a'}\sum_{s''}\left(p_{s'a's''}-p'_{s'a's''}\right)\left(c_{s'a's''}+\gamma v^{\bm{\pi},\bm{p}'}_{s''}\right)\right)\\
    &= \frac{1}{1-\gamma} \sum_{s\in\mathcal{S}} d_{\bm{\rho}}^{\bm{\pi},\bm{p}}(s) \left(\sum_{a\in\mathcal{A}}\pi_{sa}\sum_{s'}\left(p_{sas'}-p'_{sas'}\right)\left(c_{sas'}+\gamma v^{\bm{\pi},\bm{p}'}_{s'}\right)\right).
\end{align*}
\end{proof}
We note that the analysis for both $(s,a)$-rectangular and $s$-rectangular RMDPs are similar, so we only discuss the convergence guarantee for more general $s$-rectangular RMDPs in this paper. We first apply the idea from Lemma~\ref{lem:3point} to estimate the difference between each transition update generated by TMA and prove the following ascent property of TMA.
\begin{lemma}[Ascent property of TMA]\label{lem:3point_inner}
For any $\bm{y}\in\mathcal{P}_{s}$ and any $s\in\mathcal{S}$, we have
\begin{equation}\label{eq:3-point-ascent}
\beta_{t}\langle \hat{\bm{g}}^{\bm{\pi},\bm{p}_{t}}_{s},\bm{y}- \bm{p}_{t+1,s}\rangle + B(\bm{p}_{t+1,s},\bm{p}_{t,s}) \leq B(\bm{y},\bm{p}_{t,s}) - B(\bm{y},\bm{p}_{t+1,s}), 
\end{equation}
and for any $\bm{\rho}\in\Delta^{S}$,
\begin{equation*}
    J_{\bm{\rho}}(\bm{\pi},\bm{p}_{t})\leq J_{\bm{\rho}}(\bm{\pi},\bm{p}_{t+1}),\quad \forall t\geq0.
\end{equation*}
\end{lemma}
\begin{proof}
    For any $s\in\mathcal{S}$, by the definition of $\bm{p}_{t+1,s}$, we know that
\begin{align*}
    \bm{p}_{t+1,s}
    &=\mathop{\arg\max}\limits_ {\bm{p}_{s}\in\mathcal{P}_{s}}\left\{ \beta_{t}\langle \hat{\bm{g}}^{\bm{\pi},\bm{p}_{t}}_{s}, \bm{p}_{s}\rangle -B(\bm{p}_{s},\bm{p}_{t,s})\right\}\\
    &= \mathop{\arg\min}\limits_ {\bm{p}_{s}\in\mathcal{P}_{s}}\left\{ -\beta_{t}\langle \hat{\bm{g}}^{\bm{\pi},\bm{p}_{t}}_{s}, \bm{p}_{s}\rangle + B(\bm{p}_{s},\bm{p}_{t,s})\right\}\\
    &= \mathop{\arg\min}\limits_ {\bm{z}\in\mathbb{R}^{S\times A}}\left\{ -\beta_{t}\langle \hat{\bm{g}}^{\bm{\pi},\bm{p}_{t}}_{s}, \bm{z}\rangle + \mathbb{I}_{\mathcal{P}_{s}}(\bm{z})+B(\bm{z},\bm{p}_{t,s})\right\}.
\end{align*}
From the optimality condition that $\bm{p}_{t+1,s}$ satisfies, we have
\begin{align*}
    0 &\in \partial\left\{-\beta_{t}\langle \hat{\bm{g}}^{\bm{\pi},\bm{p}_{t}}_{s}, \bm{z}\rangle + \mathbb{I}_{\mathcal{P}_{s}}(\bm{z})+B(\bm{z},\bm{p}_{t,s})\right\}\Big\vert_{\bm{z}= \bm{p}_{t+1,s}}\\
\Longleftrightarrow\;\;\;\; 0 &\in -\beta_{t} \hat{\bm{g}}^{\bm{\pi},\bm{p}_{t}}_{s} +  \mathcal{N}_{\mathcal{P}_{s}}(\bm{p}_{t+1,s})+  \nabla_{\bm{z}}B(\bm{z},\bm{p}_{t,s})\Big\vert_{\bm{z}= \bm{p}_{t+1,s}}.
\end{align*}
Then, the above equation implies that, for any $\bm{y}\in\mathcal{P}_{s}$,
\begin{align}\label{eqsec5:mid}
\beta_{t}\langle \hat{\bm{g}}^{\bm{\pi},\bm{p}_{t}}_{s},\bm{p}_{t+1,s}-\bm{y}\rangle + \langle\nabla_{\bm{z}}B(\bm{z},\bm{p}_{t,s})\Big\vert_{\bm{z}= \bm{p}_{t+1,s}},\bm{y}-\bm{p}_{t+1,s}\rangle\geq0.
\end{align}
From the definition of Bregman divergence~\eqref{def:bregman}, we have
\begin{align*}
\nabla_{\bm{z}}B(\bm{z},\bm{p}_{t,s})= \nabla h(\bm{z}) - \nabla h(\bm{p}_{t,s}), 
\end{align*}
Thus, we have the following identity
\begin{align*}
B(\bm{y},\bm{p}_{t,s}) - B(\bm{p}_{t+1,s},\bm{p}_{t,s}) - B(\bm{y},\bm{p}_{t+1,s})=\langle\nabla_{\bm{z}}B(\bm{z},\bm{p}_{t,s})\Big\vert_{\bm{z}= \bm{p}_{t+1,s}}, \bm{y}-\bm{p}_{t+1,s}\rangle,
\end{align*}
We combine this identity with~\eqref{eqsec5:mid} and obtain the desired equation~\eqref{eq:3-point-ascent}. While taking $\bm{y} = \bm{p}_{t,s}$ in~\eqref{eq:3-point-ascent}, we have
\begin{equation*}
\beta_{t}\langle \hat{\bm{g}}^{\bm{\pi},\bm{p}_{t}}_{s},\bm{p}_{t,s}- \bm{p}_{t+1,s}\rangle \leq- B(\bm{p}_{t+1,s},\bm{p}_{t,s}) - B(\bm{p}_{t,s},\bm{p}_{t+1,s})  \leq0.  
\end{equation*}
By applying the first performance difference lemma (Lemma~\ref{lem:1st-diff-perf-inner}), we have
\begin{equation*}
   J_{\bm{\rho}}(\bm{\pi},\bm{p}_{t})-J_{\bm{\rho}}(\bm{\pi},\bm{p}_{t+1}) = \mathbb{E}_{s\sim \bm{d}^{\bm{\pi},\bm{p}_{t+1}}_{\bm{\rho}}}\left[\langle \hat{\bm{g}}^{\bm{\pi},\bm{p}_{t}}_{s},\bm{p}_{t,s}- \bm{p}_{t+1,s}\rangle\right] \leq 0.
\end{equation*}
\end{proof}
We notice that the inner problem can be interpreted as having an adversarial nature to maximize the total reward (decision maker's cost) by selecting a proper transition kernel from the ambiguity set $\mathcal{P}$~\citep{lim2013reinforcement,goyal2022robust}. Hence, we leverage the idea from the convergence analysis of the classical policy mirror descent~\citep{xiao2022convergence} and derive our global convergence guarantee in the following theorem.
\begin{theorem}\label{the:inner-optimal-sublinear}
    Consider the transition mirror ascent method with a fixed policy $\bm{\pi}\in\Pi$ and constant step size $\beta_{t}=\beta\geq 0$ for all $t\geq0$. For any $\bm{\rho}\in\Delta^{S}$, we have for each $t\geq 0$,
\begin{align*}
J_{\bm{\rho}}(\bm{\pi},\bm{p}^{\star})-J_{\bm{\rho}}(\bm{\pi},\bm{p}_{t})\leq \frac{1}{t}\left(\frac{M}{(1-\gamma)^{2}}+ \frac{1}{\beta(1-\gamma)}B_{\bm{d}_{\bm{\rho}}^{\bm{\pi},\bm{p}^{\star}}}(\bm{p}^{\star}_{s},\bm{p}_{0,s})\right)    
\end{align*}
\end{theorem}
\begin{proof}
Consider the inequality~\eqref{eq:3-point-ascent} in Lemma~\ref{lem:3point_inner} and let $\bm{y}=\bm{p}^{\star}_{s}$, where $\bm{p}^{\star}$ is the worst-case inner transition kernel under the policy $\bm{\pi}$. We then lead to the following result with a constant step size $\beta\geq0$
\begin{align}\label{eq:InnLinear-mid1}
\beta\underbrace{\langle \hat{\bm{g}}^{\bm{\pi},\bm{p}_{t}}_{s}, \bm{p}^{\star}_{s}-\bm{p}_{t,s}\rangle}_{\text{(A)}} + \beta\underbrace{\langle \hat{\bm{g}}^{\bm{\pi},\bm{p}_{t}}_{s}, \bm{p}_{t,s}-\bm{p}_{t+1,s}\rangle}_{\text{(B)}}  \leq  B(\bm{p}^{\star}_{s},\bm{p}_{t,s}) - B(\bm{p}^{\star}_{s},\bm{p}_{t+1,s}).
\end{align}
We first notice that, by applying the second performance difference lemma (Lemma~\ref{lem:2nd-diff-perf-inner}), we can obtain the following result
\begin{align*}
J_{\bm{\rho}}(\bm{\pi},\bm{p}^{\star})-J_{\bm{\rho}}(\bm{\pi},\bm{p}_{t}) &=  \frac{1}{1-\gamma} \sum_{s\in\mathcal{S}} d_{\bm{\rho}}^{\bm{\pi},\bm{p}^{\star}}(s) \left(\sum_{a\in\mathcal{A}}\pi_{sa}\sum_{s'}\left(p^{\star}_{sas'}-p_{t,sas'}\right)g^{\bm{\pi},\bm{p}_{t}}_{sas'}\right)\\
&= \frac{1}{1-\gamma} \sum_{s\in\mathcal{S}} d_{\bm{\rho}}^{\bm{\pi},\bm{p}^{\star}}(s)\underbrace{\langle \hat{\bm{g}}^{\bm{\pi},\bm{p}_{t}}_{s}, \bm{p}^{\star}_{s}-\bm{p}_{t,s}\rangle}_{\text{(A)}}.
\end{align*}
Similarly, by applying the first transition kernel performance difference lemma  (Lemma~\ref{lem:1st-diff-perf-inner}), we can obtain that
\begin{align*}
J_{\bm{\rho}}(\bm{\pi},\bm{p}_{t})-J_{\bm{\rho}}(\bm{\pi},\bm{p}_{t+1}) &= \frac{1}{1-\gamma} \sum_{s\in\mathcal{S}}d_{\bm{\rho}}^{\bm{\pi},\bm{p}_{t+1}}(s)\left(\sum_{a\in\mathcal{A}}\pi_{sa}\sum_{s'}\left(p_{t,sas'}-p_{t+1,sas'}\right)g^{\bm{\pi},\bm{p}_{t}}_{sas'}\right).
\end{align*}
Then, we notice that Lemma~\ref{lem:3point_inner} implies $\beta_{t}\langle \hat{\bm{g}}^{\bm{\pi},\bm{p}_{t}}_{s},\bm{p}_{t,s}- \bm{p}_{t+1,s}\rangle \leq0$, which can lead to the following result
\begin{align*}
J_{\bm{\rho}}(\bm{\pi},\bm{p}_{t})-J_{\bm{\rho}}(\bm{\pi},\bm{p}_{t+1})
&= \frac{1}{1-\gamma} \sum_{s\in\mathcal{S}} \frac{d_{\bm{\rho}}^{\bm{\pi},\bm{p}_{t+1}}(s)}{d_{\bm{\rho}}^{\bm{\pi},\bm{p}^{\star}}(s)} d_{\bm{\rho}}^{\bm{\pi},\bm{p}^{\star}}(s)\underbrace{\left(\sum_{a\in\mathcal{A}}\pi_{sa}\sum_{s'}\left(p_{t,sas'}-p_{t+1,sas'}\right)g^{\bm{\pi},\bm{p}_{t}}_{sas'}\right)}_{\leq0}\\
&\leq
\left\|\frac{\bm{d}^{\bm{\pi},\bm{p}^{\star}}_{\bm{\rho}}}{\bm{\rho}}\right\|^{-1}_{\infty}\sum_{s\in\mathcal{S}} d_{\bm{\rho}}^{\bm{\pi},\bm{p}^{\star}}(s)\left(\sum_{a\in\mathcal{A}}\pi_{sa}\sum_{s'}\left(p_{t,sas'}-p_{t+1,sas'}\right)g^{\bm{\pi},\bm{p}_{t}}_{sas'}\right)\\
&\leq\frac{1}{M}\sum_{s\in\mathcal{S}} d_{\bm{\rho}}^{\bm{\pi},\bm{p}^{\star}}(s)\left(\sum_{a\in\mathcal{A}}\pi_{sa}\sum_{s'}\left(p_{t,sas'}-p_{t+1,sas'}\right)g^{\bm{\pi},\bm{p}_{t}}_{sas'}\right).
\end{align*}
Here, the first inequality is due to the fact that $d^{\bm{\pi},\bm{p}}_{\rho}(s)\geq(1-\gamma)\rho_{s}$ for all $s\in\mathcal{S}$. Hence, using the above results, we take expectation with respect to the distribution $\bm{d}_{\bm{\rho}}^{\bm{\pi},\bm{p}^{\star}}$ on both sides of the inequality~\eqref{eq:InnLinear-mid1} and obtain
\begin{multline*}
(1-\gamma)\left(J_{\bm{\rho}}(\bm{\pi},\bm{p}^{\star})-J_{\bm{\rho}}(\bm{\pi},\bm{p}_{t})\right) + M\left(J_{\bm{\rho}}(\bm{\pi},\bm{p}_{t})-J_{\bm{\rho}}(\bm{\pi},\bm{p}_{t+1})\right)\\
\leq  \frac{1}{\beta}B_{\bm{d}_{\bm{\rho}}^{\bm{\pi},\bm{p}^{\star}}}(\bm{p}^{\star}_{s},\bm{p}_{t,s}) - \frac{1}{\beta}B_{\bm{d}_{\bm{\rho}}^{\bm{\pi},\bm{p}^{\star}}}(\bm{p}^{\star}_{s},\bm{p}_{t+1,s}),   
\end{multline*}
which further leads to 
\begin{multline*}
(1-\gamma)\sum^{t}_{i=0}\left(J_{\bm{\rho}}(\bm{\pi},\bm{p}^{\star})-J_{\bm{\rho}}(\bm{\pi},\bm{p}_{i})\right)\leq M\left(J_{\bm{\rho}}(\bm{\pi},\bm{p}_{t})-J_{\bm{\rho}}(\bm{\pi},\bm{p}_{0})\right)\\
+\frac{1}{\beta}B_{\bm{d}_{\bm{\rho}}^{\bm{\pi},\bm{p}^{\star}}}(\bm{p}^{\star}_{s},\bm{p}_{0,s}) - \frac{1}{\beta}B_{\bm{d}_{\bm{\rho}}^{\bm{\pi},\bm{p}^{\star}}}(\bm{p}^{\star}_{s},\bm{p}_{t,s}).
\end{multline*}
Since $J_{\bm{\rho}}(\bm{\pi},\bm{p}_{t})$ is monotonically non-decreasing in $t$ (see Lemma~\ref{lem:3point_inner}), we conclude that
\begin{align*}
J_{\bm{\rho}}(\bm{\pi},\bm{p}^{\star})-J_{\bm{\rho}}(\bm{\pi},\bm{p}_{t})\leq \frac{1}{t}\left(\frac{M}{(1-\gamma)^{2}}+ \frac{1}{\beta(1-\gamma)}B_{\bm{d}_{\bm{\rho}}^{\bm{\pi},\bm{p}^{\star}}}(\bm{p}^{\star}_{s},\bm{p}_{0,s})\right)    
\end{align*}
\end{proof}
As a result of Theorem~\ref{the:inner-optimal-sublinear}, whenever $\beta\geq (1-\gamma)B_{\bm{d}_{\bm{\rho}}^{\bm{\pi},\bm{p}^{\star}}}(\bm{p}^{\star}_{s},\bm{p}_{0,s})$, the number of iterations required to guarantee $J_{\bm{\rho}}(\bm{\pi},\bm{p}^{\star})-J_{\bm{\rho}}(\bm{\pi},\bm{p}_{t})\leq\epsilon$ is at most
\begin{equation*}
    \frac{M+1}{(1-\gamma)^{2}\epsilon},
\end{equation*}
which is independent of the dimensions of state and action space. Note that we don't impose any upper bound on $\beta$ in the above analysis, so the step size could be chosen to be arbitrarily large. The choice of $\beta$ in our analysis is different from that in the classical smooth optimization, where the step size is usually highly related to various Lipschitz constants. This observation is aligned with the discussion of standard policy mirror descent and may be explained by the unique structure of discounted MDPs~\citep{xiao2022convergence}. Indeed, we show next that TMA enjoys a faster $\mathcal{O}(\log(\epsilon^{-1}))$ if the step size grows linearly. 
\begin{corollary}\label{the:inner-optimal-linear}
    Consider the transition mirror ascent method with a fixed policy $\bm{\pi}\in\Pi$. Suppose the step sizes $\{\beta_{t}\}_{t\geq0}$ satisfy $\beta_{0}\geq0$ and 
\begin{equation*}
    \left(1-\frac{1-\gamma}{M}\right)\beta_{t}\geq \beta_{t-1}, \quad\forall t\geq 1,
\end{equation*}
    then we have for each $t\geq 0$,
\begin{equation*}
J_{\bm{\rho}}(\bm{\pi},\bm{p}^{\star})-J_{\bm{\rho}}(\bm{\pi},\bm{p}_{t}) \leq \left(1-\frac{1-\gamma}{M}\right)^{t} \left(J_{\bm{\rho}}(\bm{\pi},\bm{p}^{\star})-J_{\bm{\rho}}(\bm{\pi},\bm{p}_{0})+ \frac{B_{\bm{d}_{\bm{\rho}}^{\bm{\pi},\bm{p}^{\star}}}(\bm{p}^{\star}_{s},\bm{p}_{0,s})}{(M-1+\gamma)\beta_{0}}\right).    
\end{equation*}
\end{corollary}
\begin{proof}
We use the same techniques from Theorem~\ref{the:inner-optimal-sublinear} to obtain that
\begin{multline*}
(1-\gamma)\left(J_{\bm{\rho}}(\bm{\pi},\bm{p}^{\star})-J_{\bm{\rho}}(\bm{\pi},\bm{p}_{t})\right) + M\left(J_{\bm{\rho}}(\bm{\pi},\bm{p}_{t})-J_{\bm{\rho}}(\bm{\pi},\bm{p}_{t+1})\right)
  \\ \leq \frac{1}{\beta_{t}}B_{\bm{d}_{\bm{\rho}}^{\bm{\pi},\bm{p}^{\star}}}(\bm{p}^{\star}_{s},\bm{p}_{t,s}) - \frac{1}{\beta_{t}}B_{\bm{d}_{\bm{\rho}}^{\bm{\pi},\bm{p}^{\star}}}(\bm{p}^{\star}_{s},\bm{p}_{t+1,s}),   
\end{multline*}
which is equivalent to 
\begin{multline*}
J_{\bm{\rho}}(\bm{\pi},\bm{p}^{\star})-J_{\bm{\rho}}(\bm{\pi},\bm{p}_{t+1})\leq\left(1-\frac{1-\gamma}{M}\right)\left(J_{\bm{\rho}}(\bm{\pi},\bm{p}^{\star})-J_{\bm{\rho}}(\bm{\pi},\bm{p}_{t})\right)\\
+\frac{1}{M\beta_{t}}B_{\bm{d}_{\bm{\rho}}^{\bm{\pi},\bm{p}^{\star}}}(\bm{p}^{\star}_{s},\bm{p}_{t,s}) - \frac{1}{M\beta_{t}}B_{\bm{d}_{\bm{\rho}}^{\bm{\pi},\bm{p}^{\star}}}(\bm{p}^{\star}_{s},\bm{p}_{t+1,s}).
\end{multline*}
By applying the above inequality recursively, we have
\begin{gather*}
    J_{\bm{\rho}}(\bm{\pi},\bm{p}^{\star})-J_{\bm{\rho}}(\bm{\pi},\bm{p}_{t})
    \leq \\ \left(1-\frac{1-\gamma}{M}\right)^{t}\left(J_{\bm{\rho}}(\bm{\pi},\bm{p}^{\star})-J_{\bm{\rho}}(\bm{\pi},\bm{p}_{0})\right) + \left(1-\frac{1-\gamma}{M}\right)^{t-1}\frac{B_{\bm{d}_{\bm{\rho}}^{\bm{\pi},\bm{p}^{\star}}}(\bm{p}^{\star}_{s},\bm{p}_{0,s})}{M\beta_{0}}\\
     + \frac{1}{M}\sum^{t-1}_{i=1}\underbrace{\left(\left(1-\frac{1-\gamma}{M}\right)^{t-1-i}\frac{B_{\bm{d}_{\bm{\rho}}^{\bm{\pi},\bm{p}^{\star}}}(\bm{p}^{\star}_{s},\bm{p}_{i,s})}{\beta_{i}} -\left(1-\frac{1-\gamma}{M}\right)^{t-i}\frac{B_{\bm{d}_{\bm{\rho}}^{\bm{\pi},\bm{p}^{\star}}}(\bm{p}^{\star}_{s},\bm{p}_{i,s})}{\beta_{i-1}}\right)}_{\text{(C)}}.
\end{gather*}
Given that 
\begin{equation*}
    \left(1-\frac{1-\gamma}{M}\right)\beta_{i}\geq \beta_{i-1}, \quad\forall i\geq 1,
\end{equation*}
the term (C) is non-positive, and we obtain the desired result
\begin{equation*}
J_{\bm{\rho}}(\bm{\pi},\bm{p}^{\star})-J_{\bm{\rho}}(\bm{\pi},\bm{p}_{t}) \leq \left(1-\frac{1-\gamma}{M}\right)^{t} \left(J_{\bm{\rho}}(\bm{\pi},\bm{p}^{\star})-J_{\bm{\rho}}(\bm{\pi},\bm{p}_{0})+ \frac{B_{\bm{d}_{\bm{\rho}}^{\bm{\pi},\bm{p}^{\star}}}(\bm{p}^{\star}_{s},\bm{p}_{0,s})}{(M-1+\gamma)\beta_{0}}\right).    
\end{equation*}
\end{proof}
\subsection{Value-Based Approach}\label{subsec:Connect_to_VB}
Solving the inner maximization problem~\eqref{prob_inRMDP} is equivalent to computing the optimum of $\bm{v}^{\bm{\pi}^{\bm{\theta}_{k}}}:=\max_{\bm{p}\in\mathcal{P}}\bm{v}^{\bm{\pi}^{\bm{\theta}_{k}},\bm{p}}$, which is commonly defined as the \emph{robust value function}~\citep{iyengar2005robust,nilim2005robust,wiesemann2013robust}. There has been extensive study of efficient value-based methods to compute the robust value function~\citep{iyengar2005robust,nilim2005robust,wiesemann2013robust,petrik2014raam,ho2018fast,behzadian2021fast}. A standard class of methods to compute the robust value function $\bm{v}^{\bm{\pi}}$ of a rectangular RMDP for a policy $\bm{\pi}\in\Pi$ is the value-based approach, which commonly utilize the robust Bellman policy update $\mathcal{T}_{\bm{\pi}}:\mathbb{R}^{S}\rightarrow\mathbb{R}^{S}$~\citep{ho2021partial}. Specifically, for $(s,a)$-rectangular RMDPs, the operator $\mathcal{T}_{\bm{\pi}}$ is
defined for each state $s\in\mathcal{S}$
\begin{equation*}
    (\mathcal{T}_{\bm{\pi}}\bm{v})_{s} ~:=~ \sum_{a\in\mathcal{A}} \left(\pi_{sa}\cdot\max_{\bm{p}_{sa}\in\mathcal{P}_{sa}}\bm{p}_{sa}^{\top}(\bm{c}_{sa}+\gamma \bm{v})\right),
\end{equation*}
while for $s$-rectangular RMDPs, the the operator $\mathcal{T}_{\bm{\pi}}$ is defined as
\begin{equation*}
    (\mathcal{T}_{\bm{\pi}}\bm{v})_{s} ~:=~ \max_{\bm{p}_{s}\in\mathcal{P}_{s}}\sum_{a\in\mathcal{A}}\pi_{sa}\cdot\bm{p}_{sa}^{\top}(\bm{c}_{sa}+\gamma \bm{v}).
\end{equation*}
For rectangular RMDPs, $\mathcal{T}_{\bm{\pi}}$ is a $\gamma$-contraction~\citep{nilim2005robust,wiesemann2013robust,ho2021partial} and the robust value function is the unique solution of $\bm{v}^{\bm{\pi}} = \mathcal{T}_{\bm{\pi}}\bm{v}^{\bm{\pi}}$. To solve the robust value function, the value-based approach commonly computes the sequence $\bm{v}_{t+1}^{\bm{\pi}} = \mathcal{T}_{\bm{\pi}}\bm{v}_{t}^{\bm{\pi}}$ with any initial $\bm{v}_{0}^{\bm{\pi}}$.

As discussed in Section~\ref{subsec:Inner_Global_conv}, our analysis of the TMA method does not impose any upper bound on the step sizes: they can be either arbitrarily large constant or geometrically increasing. We take $s$-rectangular RMDPs as the example and consider $\beta_{t}\rightarrow\infty$ for all iterations. Then, the ascent update of TMA becomes
\begin{align*}
\bm{p}_{t+1,s} ~=~\mathop{\arg\max}\limits_ {\bm{p}_{s}\in\mathcal{P}_{s}}\left(\sum_{a\in\mathcal{A}}\pi_{sa}\sum_{s'}p_{sas'}\left(c_{sas'}+\gamma v^{\bm{\pi},\bm{p}_{t}}_{s'}\right)\right),
\end{align*}
which is precisely same as applying contraction mapping $\mathcal{T}_{\bm{\pi}}$ once on $\bm{v}^{\bm{\pi},\bm{p}_{t}}$ to compute the worst transition kernel. In fact, our analysis still holds when considering the extreme situation ($\beta_{t}\rightarrow\infty$) and the result corresponding to Corollary~\ref{the:inner-optimal-linear} is
\begin{equation*}
J_{\bm{\rho}}(\bm{\pi},\bm{p}^{\star})-J_{\bm{\rho}}(\bm{\pi},\bm{p}_{t}) \leq \left(1-\frac{1-\gamma}{M}\right)^{t} \left(J_{\bm{\rho}}(\bm{\pi},\bm{p}^{\star})-J_{\bm{\rho}}(\bm{\pi},\bm{p}_{0})\right).    
\end{equation*}
This convergence behavior of TMA closely resembles the convergence result of the value-based approaches, achieving the same $\mathcal{O}(\log(\epsilon^{-1}))$ rate. It should be noted that all convergence results of TMA in Theorem~\ref{the:inner-optimal-sublinear} and Corollary~\ref{the:inner-optimal-linear} are superior to the convergence result $\mathcal{O}(\epsilon^{-2})$ obtained by using general gradient descent analysis techniques on smooth objectives~\citep{wang2023policy}. 

\subsection{Scalability of Parametric Transition}\label{subsec:Transition_para}
In standard policy-gradient methods, one considers a family of policies parametrized by lower-dimensional parameter vectors to limit the number of variables when scaling to large problems. So far, the mirror ascent step in TMA requires directly updating each $p_{sas'}$, which is challenging with large state and action spaces. In this subsection, we apply the idea of policy parameterization and explore the utilization of parametric models for transition probabilities. We provide two novel approaches to parameterize transition kernel, generally denoted as $\bm{p}^{\bm{\xi}}$, covering both discrete and continuous settings. As far as we know, comparable parameterizations for the inner problem have not been studied previously.

\paragraph{Entropy Parametric Transition Kernel}
For the RMDP with discrete state and action spaces, we parameterize the transition with the following form for any $(s,a,s')\in\mathcal{S}\times\mathcal{A}\times\mathcal{S}$,
\begin{equation}\label{inner_para1}
p^{\bm{\xi}}_{sas'} ~:=~ \frac{\bar{p}_{sas'}\cdot\exp\left(\frac{\bm{\eta}^{\top}\bm{\phi}(s')}{\bm{\lambda}^{\top}\bm{\zeta}(s,a)})\right)}{\sum_{k}\bar{p}_{sak}\cdot\exp\left({\frac{\bm{\eta}^{\top}\bm{\phi}(k)}{\bm{\lambda}^{\top}\bm{\zeta}(s,a)}}\right)},
\end{equation}
where $\bm{\phi}(s):=\left[\phi_{1}(s),\cdots,\phi_{l}(s)\right]$ is a $l$-dimensional feature vector corresponding to the state $s\in\mathcal{S}$ and $\bm{\zeta}(s,a):=\left[\zeta_{1}(s,a),\cdots,\zeta_{n}(s,a)\right]$ is introduced as an $n$-dimensional state-action feature vector. We denote $\bm{\xi}:=(\bm{\eta},\bm{\lambda
})$ as the collection of parameters, consisting of the parameter $\bm{\lambda}\in\mathbb{R}^{n}$ with the same dimension as the feature $\bm{\zeta}(s,a)$ and the parameter $\bm{\eta}:=\left[\eta_{1},\cdots,\eta_{l}\right]$. The symbol $\bar{\bm{p}}$ represents the nominal transition kernel, which is typically estimated from the empirical sample of state transitions. Notably, \cite{wang2023policy} proposed a similar parameterization that simply used $\lambda_{sa}, s\in\mathcal{S},a\in\mathcal{A}$ to characterize the information of $(s,a)$ pair, making the number of parameters still heavily dependent on the dimensions of the state and action spaces. In contrast, the parameterization form~\eqref{inner_para1} improves the scalability by introducing a linear approximation on $\lambda_{sa}$, thereby mitigating the reliance on the dimensionality of the state and action spaces.

The parameterization~\eqref{inner_para1} is motivated by the form of the worst-case transition probabilities in RMDPs with KL-divergence constrained $(s,a)$-rectangular ambiguity sets~\citep{nilim2005robust}. In fact, the worst-case transition has an identical form to~\eqref{inner_para1} when linear approximations $\bm{\eta}^{\top}\bm{\phi}(s)$ and $\bm{\lambda}^{\top}\bm{\zeta}(s,a)$ are applied (see Appendix~\ref{app:entropy_trans} for more detailed explanation). Given~\eqref{inner_para1}, the inner policy evaluation problem becomes
\begin{equation}\label{prob_inRMDP_para}
\max_{\bm{\xi}\in\Xi} \,J_{\bm{\rho}}(\bm{\pi}, \bm{p}^{\bm{\xi}}),
\end{equation}
where $\Xi$ is the ambiguity set for the parameter $\bm{\xi}$. In practice, $\Xi$ could be constructed via distance-type constraint; that is, we consider
\begin{equation*}
    \Xi~:=~\{\bm{\xi} \mid D\left(\bm{\xi}\|\bm{\xi}_{c}\right) \leq \kappa\},
\end{equation*}
where $D(\cdot\|\cdot)$ represents a distance function, such as $L_{1}$-norm and $L_{\infty}$-norm, $\bm{\xi}_{c}$ is the user-specified empirical estimation of $\bm{\xi}$, and $\kappa\in\mathbb{R}_{++}$ is a given radius.

\paragraph{Gaussian Mixture Parametric Transition Kernel}
Finite mixtures of distributions are
being widely used to estimate complex
distributions of data on random phenomena~\citep{mclachlan2019finite}. For large, even continuous state spaces, we typically consider the normal mixture model, named Gaussian Mixture model, to parameterize transition kernel. To be specific, this parameterization assumes for a given $(s,a)\in\mathcal{S}\times\mathcal{A}$, the next state $s'\in\mathcal{S}$ is selected from a convex combination of certain component normal
distributions
\begin{equation*}
    s' ~\sim~ \sum_{m=1}^{M} \omega_{m}\mathcal{N}\left(\mu_{m}, \sigma_{m}\right),
\end{equation*}
where the transition function is defined as 
\begin{equation}\label{inner_para2}
p^{\bm{\xi}}_{sas'} ~:=~\sum^{M}_{m=1}\omega_{m}\left(\frac{1}{\sigma_{m} \sqrt{2 \pi}} \exp \left(-\frac{(s'-\bm{\eta}^{\top}_{m} \bm{\zeta}(s,a))^2}{2 \sigma_{m}^2}\right)\right)
\end{equation}
with the parameter collection $\bm{\xi}:=(\bm{\omega},\bm{\eta}_{1},\dots,\bm{\eta}_{M})$. For simplicity, in this form of parameterization, we only employ the linear function approximation on the means $\mu_{m} := \bm{\eta}^{\top}_{m} \bm{\zeta}(s,a)$, and the standard deviations are also readily approximated via similar techniques in the future work.

The parameter $\bm{\eta}_{m}$ for each $m=1,2,\dots,M$ could lie in the distance-type constrained ambiguity sets
\begin{equation*}
    \Xi_{\bm{\eta}_{m}}~:=~\left\{ \bm{\eta}_{m} \Big\vert D\left(\bm{\eta}_{m}\|\hat{\bm{\eta}}_{m}\right) \leq \kappa_{\bm{\eta}}\right\},
\end{equation*}
with $\hat{\bm{\eta}}_{m}$ being the $m$-th user-specified central point and $\kappa_{\bm{\eta}}$ is a given tolerance. The weights $\omega_{m}'s$ are also unknown and may be selected from a general ambiguity set
\begin{equation*}
\Xi_{\bm{\omega}}~:=~\left\{\bm{\omega} \in \mathbb{R}^M ~\Big\vert~ \sum^{M}_{m=1}\omega_{m} = 1,\;\omega_{m} \geq 0, \;\underline{\omega}^{\delta}_{m} \leq \omega_{m} \leq \overline{\omega}^{\delta}_{m},\forall m=1,\dots,M\right\}.
\end{equation*}
with $\delta$ being a confidence level and $[\underline{\omega}^{\delta}_{m}, \overline{\omega}^{\delta}_{m}]$ being the $\delta$-confidence interval.

\subsection{Stochastic Transition Gradient}\label{subsec:StocTMA}
We are now ready to explore gradient-based algorithms to learn the transition parameters without knowing the exact transition kernel.  These methods seek to maximize the inner objective, so their updates (approximate) gradient ascent in $J_{\bm{\rho}}(\bm{\pi},\bm{p}^{\bm{\xi}})$ with respect to $\bm{\xi}\in\Xi$. We begin with the following \textit{transition gradient theorem}, which provides the inner transition gradient, similar to the classical policy gradient theorem~\citep{sutton1999policy}.
\begin{theorem}[Transition Gradient Theorem]\label{the_sec6_1}
    Consider a map $\bm{\xi}\mapsto p^{\bm{\xi}}_{sas'}$ that is differentiable for any $(s,a,s')$ . Then, the partial gradient of $J_{\rho}(\bm{\pi},\bm{p}^{\bm{\xi}})$ on $\bm{\xi}$ is
    \begin{align}\label{eq:sec6-TGT}
    \frac{\partial J_{\rho}(\bm{\pi},\bm{p}^{\bm{\xi}})}{\partial \bm{\xi}} ~=~ \frac{1}{1-\gamma}    \mathbb{E}_{s\sim\bm{d}^{\bm{\pi},\bm{\xi}}_{\rho},a\sim\bm{\pi}_{s},s'\sim\bm{p}_{sa}}\left[\frac{\partial\log p^{\bm{\xi}}_{sas'}}{\partial \bm{\xi}}\left(c_{sas'}+\gamma v^{\bm{\pi},\bm{p}^{\bm{\xi}}}_{s'}\right)\right].
\end{align}
\end{theorem}
\begin{proof}
    We notice that
\begin{align}\label{eq_sec5.4_mid1}
    \frac{\partial v^{\bm{\pi},\bm{p}^{\bm{\xi}}}_{s}}{\partial \bm{\xi}} &=\sum_{a\in\mathcal{A}}\pi_{sa}\frac{\partial}{\partial \bm{\xi}}\left[\sum_{s'}p^{\bm{\xi}}_{sas'}\left(c_{sas'}+\gamma v^{\bm{\pi},\bm{p}^{\bm{\xi}}}_{s'}\right)\right]\notag\\
    &=\sum_{a\in\mathcal{A}}\pi_{sa}\sum_{s'}\frac{\partial p^{\bm{\xi}}_{sas'}}{\partial \bm{\xi}}\left(c_{sas'}+\gamma v^{\bm{\pi},\bm{p}^{\bm{\xi}}}_{s'}\right) + \gamma\sum_{a\in\mathcal{A}}\pi_{sa}\sum_{s'}p^{\bm{\xi}}_{sas'}\frac{\partial v^{\bm{\pi},\bm{p}^{\bm{\xi}}}_{s'}}{\partial \bm{\xi}}.
\end{align}
By condensing $\sum_{a\in\mathcal{A}}\pi_{sa}p^{\bm{\xi}}_{sas'}=p^{\bm{\pi},\bm{\xi}}_{ss'}(1)$ as the probability of the transition from $s$ to $s'$ within one step and introducing $\sum_{s'}p^{\bm{\pi},\bm{\xi}}_{ss'}(t-1)\cdot\sum_{a\in\mathcal{A}}\pi_{s'a} p^{\bm{\xi}}_{s'as''} = p^{\bm{\pi},\bm{\xi}}_{ss''}(t)$, we can obtain,
\begin{align*}
     \frac{\partial v^{\bm{\pi},\bm{p}^{\bm{\xi}}}_{s}}{\partial \bm{\xi}} &= \sum_{a\in\mathcal{A}}\pi_{sa}\sum_{s'}\frac{\partial p^{\bm{\xi}}_{sas'}}{\partial \bm{\xi}}\left(c_{sas'}+\gamma v^{\bm{\pi},\bm{p}^{\bm{\xi}}}_{s'}\right) + \gamma\sum_{s'}p^{\bm{\pi},\bm{\xi}}_{ss'}(1)\frac{\partial v^{\bm{\pi},\bm{p}^{\bm{\xi}}}_{s'}}{\partial \bm{\xi}}\\
     &= \sum_{a\in\mathcal{A}}\pi_{sa}\sum_{s'}\frac{\partial p^{\bm{\xi}}_{sas'}}{\partial \bm{\xi}}\left(c_{sas'}+\gamma v^{\bm{\pi},\bm{p}^{\bm{\xi}}}_{s'}\right)\\
     &\;\;\;\;+ \gamma\sum_{s'}p^{\bm{\pi},\bm{\xi}}_{ss'}(1)\left[\sum_{a'}\pi_{s'a'}\sum_{s''}\frac{\partial p^{\bm{\xi}}_{s'a's''}}{\partial \bm{\xi}}\left(c_{s'a's''}+\gamma v^{\bm{\pi},\bm{p}^{\bm{\xi}}}_{s''}\right)+\gamma\sum_{s''}p^{\bm{\pi},\bm{\xi}}_{s's''}(1)\frac{\partial v^{\bm{\pi},\bm{p}^{\bm{\xi}}}_{s''}}{\partial \bm{\xi}}\right]\\
     &=\sum^{2}_{k=0}\gamma^{k}\sum_{s'}p^{\bm{\pi},\bm{\xi}}_{ss'}(k)\sum_{a'}\pi_{s'a'}\left[\sum_{s''}\frac{\partial p^{\bm{\xi}}_{s'a's''}}{\partial \bm{\xi}}\left(c_{s'a's''}+\gamma v^{\bm{\pi},\bm{p}^{\bm{\xi}}}_{s''}\right)\right] + \gamma^{3}\sum_{s'}p^{\bm{\pi},\bm{\xi}}_{ss'}(3)\frac{\partial v^{\bm{\pi},\bm{p}^{\bm{\xi}}}_{s'}}{\partial \bm{\xi}}.
\end{align*}
By recursively applying~\eqref{eq_sec5.4_mid1}, we can further obtain that
\begin{equation*}
    \frac{\partial v^{\bm{\pi},\bm{p}^{\bm{\xi}}}_{s}}{\partial \bm{\xi}}=\sum^{\infty}_{k=0}\gamma^{k}\sum_{s'}p^{\bm{\pi},\bm{\xi}}_{ss'}(k)\sum_{a'}\pi_{s'a'}\left[\sum_{s''}\frac{\partial p^{\bm{\xi}}_{s'a's''}}{\partial \bm{\xi}}\left(c_{s'a's''}+\gamma v^{\bm{\pi},\bm{p}^{\bm{\xi}}}_{s''}\right)\right].
\end{equation*}
Finally, we can have our desired result, that is,
\begin{align*}
    \frac{\partial J_{\rho}(\bm{\pi},\bm{p}^{\bm{\xi}})}{\partial \bm{\xi}}&= \sum_{s\in\mathcal{S}}\rho_{s}\sum^{\infty}_{k=0}\gamma^{k}\sum_{s'}p^{\bm{\pi},\bm{\xi}}_{ss'}(k)\sum_{a'}\pi_{s'a'}\left[\sum_{s''}\frac{\partial p^{\bm{\xi}}_{s'a's''}}{\partial \bm{\xi}}\left(c_{s'a's''}+\gamma v^{\bm{\pi},\bm{p}^{\bm{\xi}}}_{s''}\right)\right]\\
    &= \frac{1}{1-\gamma}\sum_{s\in\mathcal{S}}d^{\bm{\pi},\bm{p}^{\bm{\xi}}}_{\bm{\rho}}(s)\sum_{a\in\mathcal{A}}\pi_{sa}\left[\sum_{s'}\frac{\partial p^{\bm{\xi}}_{sas'}}{\partial \bm{\xi}}\left(c_{sas'}+\gamma v^{\bm{\pi},\bm{p}^{\bm{\xi}}}_{s'}\right)\right]\\
    &= \frac{1}{1-\gamma}\sum_{s\in\mathcal{S}}d^{\bm{\pi},\bm{p}^{\bm{\xi}}}_{\bm{\rho}}(s)\sum_{a\in\mathcal{A}}\pi_{sa}\sum_{s'}p^{\bm{\xi}}_{sas'}\left[\frac{\partial \log p^{\bm{\xi}}_{sas'}}{\partial \bm{\xi}}\left(c_{sas'}+\gamma v^{\bm{\pi},\bm{p}^{\bm{\xi}}}_{s'}\right)\right].
\end{align*}
\end{proof}
The term $\nabla_{\bm{\xi}}\log p^{\bm{\xi}}_{sas'}$ is commonly known as the \emph{score function}. Through simple calculation, we can obtain the analytical forms of the proposed transition parameterization from the following proposition.
\begin{corollary}
Consider the entropy parametric transition~\eqref{inner_para1}, the corresponding score function is
\begin{align*}
\frac{\partial\log p^{\bm{\xi}}_{sas'}}{\partial \theta_{i}} &~=~ \frac{\phi_{i}(s')}{\bm{\lambda}^{\top}\bm{\zeta}(s,a)} - \sum_{j}p^{\bm{\xi}}_{saj}\cdot\frac{\phi_{i}(j)}{\bm{\lambda}^{\top}\bm{\zeta}(s,a)},\\
\frac{\partial\log p^{\bm{\xi}}_{sas'}}{\partial \lambda_{sa}} &~=~ \sum_{j}p^{\bm{\xi}}_{saj}\cdot\frac{\bm{\theta}^{\top}\bm{\phi}(j)\cdot \zeta_{i}(s,a)}    {(\bm{\lambda}^{\top}\bm{\zeta}(s,a))^{2}}-\frac{\bm{\theta}^{\top}\bm{\phi}(s')\cdot \zeta_{i}(s,a)}{(\bm{\lambda}^{\top}\bm{\zeta}(s,a))^{2}}.
\end{align*}
For the Gaussian mixture transition paramterization~\eqref{inner_para2}, the corresponding score functions is
\begin{align*}
\frac{\partial\log p^{\bm{\xi}}_{sas'}}{\partial \omega_{m}} &~=~ \frac{\frac{1}{\sigma_{m} \sqrt{2 \pi}} \exp \left(-\frac{(s'-\bm{\theta}^{\top}_{m} \bm{\phi}(s,a))^2}{2 \sigma_{m}^2}\right)}{p^{\bm{\xi}}_{sas'}},\\
\frac{\partial\log p^{\bm{\xi}}_{sas'}}{\partial \bm{\theta}_{m}} &~=~ \omega_{m}\cdot\frac{\frac{1}{\sigma_{m} \sqrt{2 \pi}} \exp \left(-\frac{(s'-\bm{\theta}^{\top}_{m} \bm{\phi}(s,a))^2}{2 \sigma_{m}^2}\right)}{p^{\bm{\xi}}_{sas'}}\cdot \frac{(s'-\bm{\theta}^{\top}_{m} \bm{\phi}(s,a))\bm{\phi}(s,a)}{\sigma_{m}^2}.
\end{align*}
\end{corollary}
With exact transition gradient, a simple variant of TMA in Algorithm~\ref{alg:transition_gradient_para} is derived, which directly takes the mirror ascent step on the inner parameter. Compared to our conference work~\citep{wang2023policy}, this generalized TMA is more versatile to incorporate different types of Bregman divergence, encompassing the previously used projected gradient ascent update in~\cite{wang2023policy}. Notably, a concurrent work by~\cite{li2023policy} studies a variant of projected gradient method by leveraging the idea of Langevin dynamics and employs a projected stochastic gradient Langevin method~\citep{lamperski2021projected} to solve the inner problem~\eqref{prob_inRMDP_para}, offering convergence guarantee even without assuming rectangularity.  
\begin{algorithm}[t]
\caption{Generalized Transition Mirror Ascent Method (Generalized TMA)}
\label{alg:transition_gradient_para}
\begin{algorithmic}
\STATE {\bfseries Input:} Target fixed policy with parameter $\bm{\theta}_{k}$, initial transition parameters $\bm{\xi}_{0}$, iteration time $T_{k}$, step size sequence $\{\beta_{t}\}_{t\geq0}$
\FOR{$t = 0,1,\dots,T_{k}-1$}
\STATE Set $\displaystyle\bm{\xi}_{t+1} \gets \argmax_{\bm{\xi}\in\Xi}\left\{\beta_{t}\langle \nabla_{\bm{\xi}}J_{\bm{\rho}}(\bm{\pi}^{\bm{\theta}_{k}},\bm{p}^{\bm{\xi}_{t}}), \bm{\xi}\rangle - D(\bm{\xi},\bm{\xi}_{t})\right\}$.
\ENDFOR
\end{algorithmic}
\end{algorithm}

The above transition gradient theorem serves as the cornerstone for our proposed gradient-based methods that solve the inner problem; however, computing the exact transition gradient is still costly, particularly when dealing with large-scale or continuous problems. To improve the scalability of Algorithm~\ref{alg:transition_gradient_para}, we develop the stochastic variant of generalized TMA, where the exact information of the transition gradient is unknown and it can only be approximated from samples. Specifically, we update the parameter $\bm{\xi}$ on $J_{\bm{\rho}}(\bm{\pi},\bm{p}^{\bm{\xi}})$ with the approximated mirror ascent step 
\begin{equation}\label{eq:stoc_mirror}
    \bm{\xi}_{t+1} ~=~ \mathop{\arg\max}\limits_ {\bm{\xi}\in\Xi}\left\{\beta_{t}\langle \widehat{\nabla_{\bm{\xi}}J_{\bm{\rho}}(\bm{\pi},\bm{p}^{\bm{\xi}_{t}})}, \bm{\xi}\rangle - D(\bm{\xi},\bm{\xi}_{t})\right\},
\end{equation}
where $\widehat{\nabla_{\bm{\xi}}J_{\bm{\rho}}(\bm{\pi},\bm{p}^{\bm{\xi}_{t}})}$ is a stochastic estimate by sampling from the simulator whose expectation approximates the transition gradient. 

While the transition gradient theorem provides an exact expression proportional of the gradient, to implement the stochastic TMA, one is required to obtain samples such that the expectation of the sample gradient is proportional to the actual gradient. Let $\{\tau_{m}\}^{M}_{m=1}$ be $M$ sampling trajectories and define $\tau_{m} = (S_{m,0},A_{m,0},R_{m,1},\cdots,S_{m,T_{m}})$ with length $T_{m}$. We also introduce $G_{\tau_{m}}= \sum^{T_{m}-1}_{i=0}\gamma^{i}R_{m,i}$ as the return of the $m$-th trajectory. Then, the transition gradient can be estimated as
\begin{equation*}
\nabla_{\bm{\xi}}J(\bm{\pi}, \bm{p}^{\bm{\xi}}) = \mathbb{E}^{\bm{p}^{\bm{\xi}}}\left[G_{\tau} \nabla_{\bm{\xi}} \log P^{\bm{\xi}}(\tau)\right]  \approx \frac{1}{M}\sum_{m=1}^{M}\sum_{i=0}^{T_{m}-1} \nabla_{\bm{\xi}} \log p^{\bm{\xi}}_{S_{m,i}A_{m,i}S_{m,i+1}}\cdot G_{m,i},
\end{equation*}
which is similar to the estimation of policy gradient in REINFORCE~\citep{sutton2018reinforcement}. We refer interested reader to Appendix~\ref{app:MCTG} for a brief explanation. Then, a fundamental Monte Carlo algorithm with sample-based transition mirror ascend is outlined in Algorithm~\ref{alg:MC-inner}, named the Monte Carlo Transition Mirror Ascent (MCTMA). Under some standard approximation conditions, this assures improvement in expected performance and convergence to a local optimum.
\begin{algorithm}[t]
\caption{Monte-Carlo Transition Mirror Ascent (MCTMA)}
\label{alg:MC-inner}
\begin{algorithmic}
\STATE {\bfseries Input:} Target fixed policy with parameter $\bm{\theta}_{k}$, initial transition parameters $\bm{\xi}$, step sizes $\{\beta_{m}\}_{m\geq 0}$, number of trajectory $M$
\FOR{$m = 0,1,\dots,M-1$}
 \STATE Sample episode $\tau_{m}$ under $\bm{\pi}^{\bm{\theta}_{k}}$ and $\bm{p}^{\bm{\xi}}$: $S_{m,0},A_{m,0},R_{m,1},\cdots,S_{m,T_{m}}$
\FOR{$i = 0,1,\cdots,T_{m}-1$}
\STATE Set $\bm{\xi}_{i} = \bm{\xi}$
\STATE Define $G_{m,i} = R_{m,i+1} + \gamma R_{m,i+2}+\cdots+\gamma^{T_{m}-i-1}R_{m,T_{m}}$
\STATE Update $\displaystyle\bm{\xi} \gets \argmax_{\bm{z}\in\Xi}\left\{\gamma^{i}\beta_{m}\langle G_{m,i}\nabla_{\bm{\xi}} \log p^{\bm{\xi}_{i}}_{S_{m,i}A_{m,i}S_{m,i+1}}, \bm{z}\rangle - D(\bm{z},\bm{\xi}_{i})\right\}$
\ENDFOR
\ENDFOR
\end{algorithmic}
\end{algorithm} 

The theoretical analysis in Section~\ref{sec:DRPMD} implies that the performance of DRPMD is independent of the specific method chosen to solve the inner problem. For the tabular problem with small state and action spaces, the cost of computing the exact gradient for the policy and transition may be acceptable, making DRPMD aligned with TMA or generalized TMA suitable without stochastic approximation. However, for large-scale or continuous problems where collecting full first-order information is costly, using sample data to estimate the gradient of both policy and transition becomes an attractive alternative. Therefore, we further develop our first robust policy-based learning algorithm outlined in Algorithm~\ref{alg:RMCPMD}, named Robust Monte-Carlo Policy Mirror Descent (RMCPMD). 
\begin{algorithm}[t]
\caption{Robust Monto-Carlo Policy Mirror Descent (RMCPMD)}
\label{alg:RMCPMD}
\begin{algorithmic}
\STATE {\bfseries Input:} Initial policy parameter $\bm{\theta}$ and transition parameter $\bm{\xi}$, step sizes $\{\alpha_{t}\}_{t\geq 0}$ and $\{\beta_{t}\}_{t\geq 0}$, numbers of iteration $T$ and trajectories $M,N$
\FOR{$t = 0,1,\dots,T-1$}
\STATE \texttt{// Update inner parameter}
\FOR{$m = 0,1,\dots, M-1$}
\STATE Sample episode $\tau_{m}$ under $\bm{\pi}^{\bm{\theta}}$ and $\bm{p}^{\bm{\xi}}$: $S_{m,0},A_{m,0},R_{m,1},\cdots,S_{m,T_{m}}$
\FOR{$i = 0,1,\cdots,T_{m}-1$}
\STATE Set $\bm{\xi}_{i} = \bm{\xi}$
\STATE Define $G_{m,i} = R_{m,i+1} + \gamma R_{m,i+2}+\cdots+\gamma^{T_{m}-i-1}R_{m,T_{m}}$
\STATE Update $\displaystyle\bm{\xi} \gets \argmax_{\bm{z}\in\Xi}\left\{\gamma^{i}\beta_{t}\langle G_{m,i}\nabla_{\bm{\xi}} \log p^{\bm{\xi}_{i}}_{S_{m,i}A_{m,i}S_{m,i+1}}, \bm{z}\rangle - D(\bm{z},\bm{\xi}_{i})\right\}$
\ENDFOR
\ENDFOR
\STATE \texttt{// Update outer parameter}
\FOR{$n = 0,1,\dots, N-1$}
\STATE Sample episode $\tau_{n}$ under $\bm{\pi}^{\bm{\theta}}$ and $\bm{p}^{\bm{\xi}}$: $S_{n,0},A_{n,0},R_{n,1},\cdots,S_{n,T_{n}}$
\FOR{$j = 0,1,\cdots,T_{n}-1$}
\STATE Set $\bm{\theta}_{j} = \bm{\theta}$
\STATE Define $G_{n,j} = R_{n,j+1} + \gamma R_{n,j+2}+\cdots+\gamma^{T_{n}-j-1}R_{n,T_{n}}$
\STATE Update $\displaystyle\bm{\theta} \gets  \argmin_ {\bm{x}\in\Theta}\left\{ \gamma^{j}\alpha_{t}\langle G_{n,j}\nabla_{\bm{\theta}} \log \pi^{\bm{\theta}_{j}}_{S_{n,j}A_{n,j}}, \bm{x}\rangle + D(\bm{x},\bm{\theta}_{i})\right\}$
\ENDFOR
\ENDFOR
\ENDFOR
\end{algorithmic}
\end{algorithm}

RMCPMD involves an additional inner sampling process to update the transition parameter $\bm{\xi}$ under current policy, then updates the policy parameter given the estimate gradient. It is worth noting that, while most recent policy-base algorithms for addressing RMDPs are limited to the tabular rectangularity setting~\citep{li2022first,wang2022policy,kumar2023towards,zhou2023natural,kumar2024policy,lin2024single}, RMCPMD serves as a foundational technique in robust policy learning and can naturally handle continuous action spaces. As a stochastic variant of DRPMD, RMCPMD leverages the idea of Monte Carlo sampling, which may suffer from high variance. We leave the study of further algorithmic improvement (\ie, variance reduction, faster speed, ...) for future research. That being said, this proposed RMCPMD method still performs well in our numerical experiments.

\section{Numerical Evaluation}\label{sec:Numerical}

We now demonstrate the convergence and robustness of DRPMD combined with three inner solution methods, including the transition mirror ascent method (Algorithm~\ref{alg:transition_gradient}), the general transition mirror ascent method (Algorithm~\ref{alg:transition_gradient_para}), and the Monte-Carlo transition mirror ascent (Algorithm~\ref{alg:MC-inner}). In our experiment, we use the common SE distance as the choice of Bregman divergence applied in our proposed algorithms. The results were generated on a computer with an i7-11700 CPU with 16GB RAM. The algorithms are implemented in Python and C++, and we use Gurobi 11.0.3 to solve any linear or quadratic optimization problems involved. To facilitate the reproducibility of the domains, the full source code, which was used to generate
them, is available at \url{https://github.com/JerrisonWang/JMLR-DRPMD}. To assess the performance of the proposed algorithms and parameterizations, we consider three standard test problems: the Garnet problem~\citep{archibald1995generation}, the inventory management problem~\citep{porteus2002foundations}, and the cart-pole problem~\citep{lagoudakis2003least,brockman2016openai}. 

\subsection{Garnet Problem}
In this experiment, we first demonstrate the convergence behavior of DRPMD on both $(s,a)$-rectangular and $s$-rectangular RMDPs. Then, we illustrate the effect of the decreasing tolerance sequence on the computational efficiency. We consider the Garnet MDPs as nominal MDPs and use them to construct rectangular RMDPs. MDPs generated by the Generalized Average Reward Non-stationary Environment Test bench (Garnet) are a class of synthetic finite MDPs that can be generated randomly. A general GARNET MDP $\mathcal{G}(S,A,b)$ is characterized by three parameters, where $S$ is the number of states, $A$ is the number of actions, and $b$ is a branching factor which determines the number of possible next states for each state-action pair and controls sparsity of the transition kernel.

In the experimental setup, we generate the cost $c_{sas'}$ corresponding to $(s,a,s')\in\mathcal{S}\times\mathcal{A}\times\mathcal{S}$ randomly from the uniform distribution on $[0,5]$ and the transition kernel randomly as the nominal transition $\bm{p}_{c}\in\mathcal{P}$. We let the discount factor $\gamma=0.95$. Constrained $L_{1}$-norm rectangular ambiguity sets~\citep{petrik2014raam,ghavamzadeh2016safe,ho2021partial} are considered in our simulation. Our $(s,a)$-rectangular ambiguity sets $\mathcal{P}_{sa}$ has the following form
\begin{equation*}
    \mathcal{P}_{s,a}:=\left\{\bm{p}\in\Delta^{S} \mid \|\bm{p}-\bm{p}_{c,sa}\|_{1} \leq \kappa_{sa}\right\},
\end{equation*}
and similarly we use the following $s$-rectangular ambiguity sets $\mathcal{P}_{s}$
\begin{equation*}
    \mathcal{P}_{s}:=\left\{\left(\bm{p}_{s 1}, \ldots, \bm{p}_{s A}\right)\in(\Delta^{S})^{A} \mid \sum_{a\in\mathcal{A}}\|\bm{p}_{sa}-\bm{p}_{c,sa}\|_{1} \leq \kappa_{s}\right\}.
\end{equation*}
We randomly select the ambiguity budget $\kappa_{sa}$ and $\kappa_{s}$ from a uniform distribution $[0.3,0.6]$ for all $s\in\mathcal{S},a\in\mathcal{A}$.

\begin{figure}
\centering
\subfigure[$(s,a)$-rectangular RMDPs]{\label{fig:Garnet_sa}\includegraphics[width=0.4\textwidth]{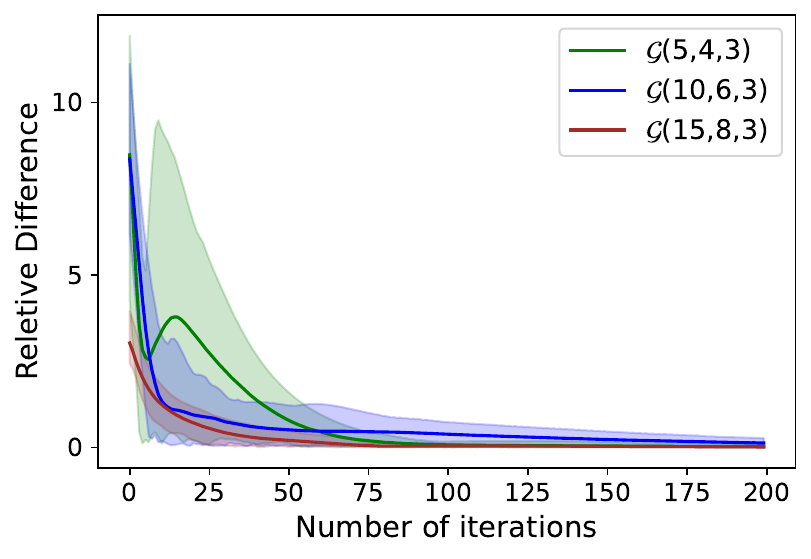}}
\hspace{.5in}
\subfigure[$s$-rectangular RMDPs]{\label{fig:Garnet_s}\includegraphics[width=0.4\textwidth]{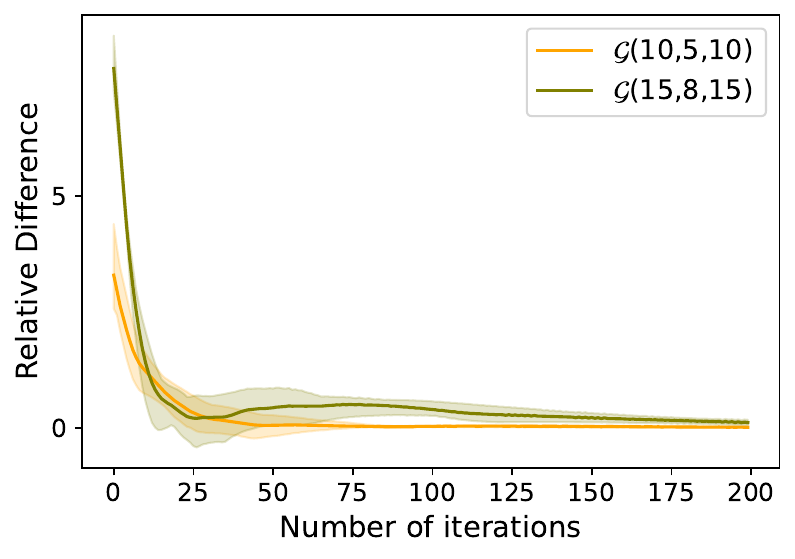}}
\caption{The relative difference of objective values computed by DRPMD and RVI for Garnet problems with different sizes}
\label{fig:mainfig}
\end{figure}
To demonstrate the convergence performance of DRPMD, we implement both our proposed algorithm and the robust value iteration (RVI), the state-of-the-art benchmark for solving rectangular RMDPs through iterative robust Bellman optimality updates~\citep{iyengar2005robust,nilim2005robust,wiesemann2013robust,ho2021partial,nilim2005robust}. In each of our Garnet problems, we begin with an identical initial policy and compare the objective values $J(\bm{\pi}^{1}_{t},\bm{p}^{1}_{t})$ of DRPMD with those values $J(\bm{\pi}^{2}_{t},\bm{p}^{2}_{t})$ computed by RVI at different iterations. To illustrate the effect of the decreasing tolerance sequence $\{\epsilon_{t}\}_{t\geq0}$ on the computational efficiency of DRPMD, we test DRPMD with two tolerance selections, typically referred as:
\begin{itemize}
    \item \emph{DRPMD-Exa}, where DRPMD is employed with exactly solved inner loop (\ie, $\epsilon_{t}=10^{-6},\forall t>0$),
    \item \emph{DRPMD-Dec}, where DRPMD with approximately solved inner loop under a decreasing tolerance sequence $\epsilon_{t+1}=\gamma\epsilon_{t}$ such that $\epsilon_{t+1}\leq\gamma\epsilon_{t}$. 
\end{itemize}
Then, we compare the required computation time of DRPMD-Dec to that of DRPMD-Exa.

For each Garnet problem, we implement both DRPMD and RVI with 200 iterations to solve 50 sample instances. In Figure~\ref{fig:Garnet_sa}, the relative difference (i.e., $|J(\bm{\pi}^{1}_{t},\bm{p}^{1}_{t}) - J(\bm{\pi}^{2}_{t},\bm{p}^{2}_{t})|$) has been shown as decreasing when DRPMD and RVI are performed. The upper and lower envelopes of the curves correspond to the 95 and 5 percentiles of the 50 samples, respectively. As expected, the relative difference decreases to zero as the number of iterations increases, which implies that DRPMD converges to
a nearly identical value computed by RVI, confirming the global convergence behavior of DRPMD. Similar results are observed for the s-rectangular case in Figure~\ref{fig:Garnet_s}.

\begin{table}[h]
\centering
\resizebox{\textwidth}{!}{
\begin{tabular}{|ccc|cc|cc|}
\toprule
\multicolumn{3}{|c|}{} & \multicolumn{2}{|c|}{ $(s,a)$-rectangular } & \multicolumn{2}{|c|}{ $s$-rectangular } \\
\midrule
States & Actions & Branches & DRPMD-Exa & DRPMD-Dec & DRPMD-Exa & DRPMD-Dec \\
\midrule
5 & 4 & 3 & 24.10 & 3.95 & 14.50 & 2.40 \\
10 & 6 & 3 & 181.50 & 24.85 & 226.35 & 23.20 \\
15 & 8 & 3 & $636.65$ & 88.20 & $1,236.50$ & 188.25 \\
20 & 10 & 8 & $1,259.85$ & 194.15 & $2,070.25$ & 514.80 \\
30 & 10 & 10 & $3,938.75$ & 457.50 & $8,597.50$ & $1414.60$ \\
50 & 10 & 10 & $>20,000.00$ & $2,556.65$ & $>20,000.00$ & $4,029.35$ \\
\bottomrule
\end{tabular}
}
\caption{Runtime (in seconds) required by DRPDM with two different tolerance selections.}
\label{tab:DRPMD-tolerance}
\end{table}
Table~\ref{tab:DRPMD-tolerance} reports the runtimes required by the parallelized versions of DRPMD-Exa and DRPMD-Dec to solve our Garnet problems to approximate optimality. To this end, we run our algorithms on each Garnet problem with different problem sizes over $20$ sample instances, recording the average computation time. Both DRPMD-Exa and DRPMD-Dec are terminated when $J(\bm{\pi}_{t},\bm{p}_{t})$ shows minimal change (\ie, $\|J(\bm{\pi}_{t+1},\bm{p}_{t+1}) - J(\bm{\pi}_{t},\bm{p}_{t})\|\leq10^{-5}$) or the runtime exceeds $20,000$ seconds. One can observe that by applying the decreasing tolerance sequence, DRPMD (DRPMD-Dec) significantly decreases the runtimes for both $(s,a)$- and $s$-rectangular ambiguity set when compared to its counterpart (DRPMD-Exa).

\subsection{Inventory Management Problem}
Our second experiment considers the inventory management problem, where a retailer engages in the ordering, storage, and sale of a single product over an infinite time horizon. The states and actions of the MDP represent the inventory levels and the order quantities in any given time period, respectively. The stochastic demands drive the stochastic state transitions. Any items held in inventory incur deterministic per-period holding costs. The retailer's goal is to find a policy that minimizes the total cost without knowing the exact transition kernel.

Here, we consider the inventory management problem with a continuous state space and a discrete action space. We employ the softmax parameterization~\eqref{def:Soft-policy} with the linear approximation 
\begin{equation*}
    \pi^{\bm{\theta}}_{sa} ~:=~ \frac{\exp(\bm{\theta}^{\top}\bm{\zeta}(s,a))}{\sum_{a^{\prime} \in \mathcal{A}} \exp(\bm{\theta}^{\top}\bm{\zeta}(s,a'))}, \quad \forall s \in \mathcal{S}, \forall a\in \mathcal{A}.
\end{equation*}
to mitigate the impact of the action space size and use the Gaussian mixture parametric transition~\eqref{inner_para2} due to the continuous state space. We use the radial-type features~\citep{sutton2018reinforcement} in policy and transition parameterizations with a two-dimensional state feature ($n=2$) 
\begin{align*}
\bm{\zeta}_{i}(s,a) ~:=~ \exp \left(-\frac{\left\|s-[\mu_{\bm{\zeta},s}]_{i}\right\|^2+\left\|a-[\mu_{\bm{\zeta},a}]_{i}\right\|^2}{2 [\sigma_{\bm{\zeta}}]_{i}^2}\right), \quad\forall i\in[n], 
\end{align*}
where we set $[\mu_{\bm{\zeta},s}]_{1}=-1.3$, $[\mu_{\bm{\zeta},s}]_{2}=-0.7$, $[\mu_{\bm{\zeta},a}]_{1} = 0.1$, $[\mu_{\bm{\zeta},a}]_{2} = 0.5$, $[\sigma_{\bm{\zeta}}]_{1} = 10$, $[\sigma_{\bm{\zeta}}]_{2} = 5$. We consider the simplest scenario where $m=1$ in the Gaussian mixture parameterization and a standard $L_{1}$-norm constrained ambiguity set, that is $\Xi:= \{\bm{\eta}\mid\bm{\eta}\in\mathbb{R}^{2},\|\bm{\eta} - \bm{\eta}_{c}\|_{1}\leq\kappa_{\eta}\}$, where we set $\bm{\eta}_{c}:=[-2,3.5]$ and $\kappa_{\eta} = 15$. The initial updating step size for $\bm{\eta}$ on the inner problem is chosen as $\beta_{0}=0.1$. Other parameters are included in the published codes.

In this experiment, we implement the proposed Robust Monto-Carlo Policy Mirror Descent (RMCPMD) method to compute the optimal robust value of the underlying continuous inventory management problem under ambiguity. To illustrate the robustness of RMCPMD, we also apply the non-robust Monte-Carlo policy gradient (MC-PG) method as a benchmark to solve the optimal objective value of the same problem with nominal model. Similar to the experiments on the GARNET problem, we run both RMCPMD with the inner Gaussian mixture parameterization and the non-robust MC-PG over $20$ times and record the objectives $\Phi(\bm{\pi}^{\bm{\theta}^{1}_{t}})$ and $\Phi(\bm{\pi}^{\bm{\theta}^{2}_{t}})$ computed by RMCPMD and non-robust MC-PG at each iteration $t$. We then plot our results in Figure~\ref{RMCPMD-Inventory} with respect to the number of iterations $t$. The
upper and lower envelopes of the curves correspond to the
$95$ and $5$ percentiles of the $20$ curves, respectively. The comparison result shows that our method obtains the policy that outperforms the policy computed by the non-robust policy gradient method, which also showcases both the robustness of our method and the effectiveness of our inner parameterization.

\begin{figure}[t]
\centering
\includegraphics[width=0.5\textwidth]{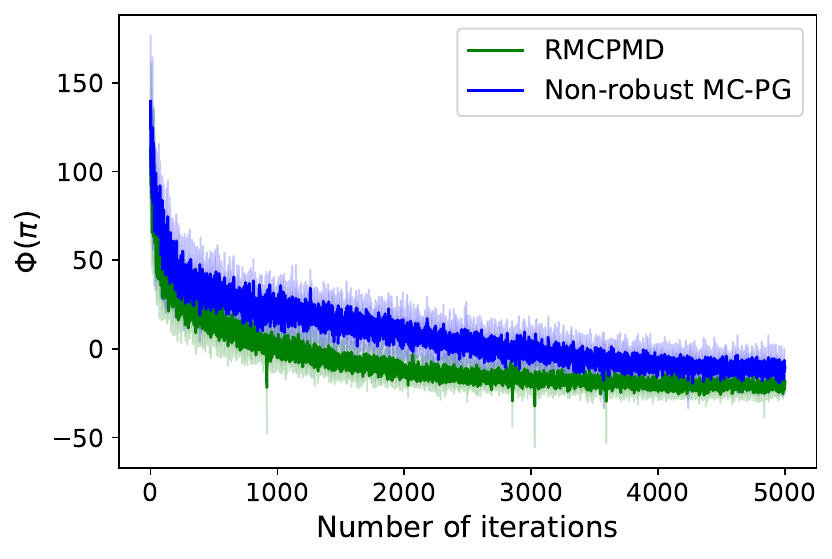}
\caption{RMCPMD with GM parametric transitions v.s. Non-robust MC-PG on the inventory management problem}
\label{RMCPMD-Inventory}
\end{figure}

\subsection{Cart-Pole Problem}
We now broaden our empirical investigation by considering a more general neural policy parameterization in our experiment. We proceed to evaluate our algorithm in the cart-pole problem, where a pole has to be balanced upright on top of a cart that moves along a single dimension. At any point in time, the state of the system is described by
four continuous quantities: the cart’s position and velocity, as well as the pole’s angle and angular velocity. To balance the pole, one needs to apply a force to the cart from the left or right. 

To address this problem with continuous state space and discrete action space, we typically consider the softmax policy parameterization with a neural network function approximation and inner Gaussian mixture parametric transitions. Specifically, the cart-pole dynamics in this experiment are described by the Gaussian mixture parameterized transition kernel with $m=1$ as follows: 
\begin{equation*}
   p_{\bm{s}a\bm{s}'} ~:=~\frac{1}{ (2\pi)^{2}|\Sigma|^{\frac{1}{2}}} \exp \left( -\frac{1}{2}(\bm{s}'-(1+\delta) \bar{\bm{s}})^{\top}\Sigma^{-1}(\bm{s}'-(1+\delta) \bar{\bm{s}})\right), 
\end{equation*}
where $\bar{\bm{s}}$ is the deterministic next state while the action $a$ is taken at the state $\bm{s}$ in the ideal cart-pole physical environment, and $\delta$ represents the disturbance factor, indicating the potential unexpected deviation from the correct next state $\bm{s}$.

We trained two policies: (1) the non-robust policy, trained using the standard policy gradient method, and (2) the robust policy, trained using RMCPMD, in the disturbed experiment with the ambiguity $\delta\in[0,\kappa_{0}]$. To demonstrate the robustness of these two policies against environmental perturbations, we test them in 11 new perturbed environments with the form $\|\delta\|\leq\kappa$. We introduce the ``Uncertainty Level'' to quantify the extent of perturbation imposed on the environment: the zero uncertainty level corresponds to the ideal scenario with no perturbation ($\kappa = 0$), while higher uncertainty levels correspond to greater perturbations (as $\kappa$ increasing). The maximum number of steps is set to $10,000$, and we define ``Success'' in a cart-pole game as maintaining an upright cart position for the entire $10,000$ step duration. We repeat this game 1000 times across different uncertainty levels and record the steps before termination by both policies in each environment.

\begin{table}[h!]
\centering
\resizebox{\textwidth}{!}{
\begin{tabular}{ccccccccccccc}
\toprule
\multicolumn{2}{c}{Uncertainty Level} & 0 & 1 & 2 & 3 & 4 & 5 & 6 & 7 & 8 & 9 & 10 \\
\midrule
\multirow{4}{*}{\centering MDP Policy} & Min & 10000 & 352 & 115 & 151 & 97 & 97 & 113 & 108 & 81 & 103 & 78 \\
                          & Max & 10000 & 10000 & 10000 & 10000 & 4523 & 2442 & 1553 & 1763 & 1423 & 1113 & 1310 \\
                          & Mean & 10000 & 9542 & 6118 & 1961 & 902 & 615 & 470 & 413 & 373 & 360 & 356 \\
                          & Success Rate & 100 & 90 & 31.1 & 0.1 & 0 & 0 & 0 & 0 & 0 & 0 & 0 \\
\midrule
\multirow{4}{*}{\centering RMDP Policy} & Min & 10000 & 10000 & 6625 & 406 & 480 & 73 & 166 & 90 & 125 & 62 & 81 \\
                          & Max & 10000 & 10000 & 10000 & 10000 & 10000 & 10000 & 10000 & 10000 & 10000 & 10000 & 10000 \\
                          & Mean & 10000 & 10000 & 9994 & 9971 & 9893 & 9626 & 9157 & 8519 & 8021 & 7221 & 7121 \\
                          & Success Rate & 100 & 100 & 99.8 & 99.2 & 97.5 & 92 & 83.7 & 73.3 & 62.1 & 50 & 48.5 \\
\bottomrule
\end{tabular}
}
\caption{Comparative Results in the Cart-Pole Experiment}
\label{tab:Cartpole}
\end{table}

The results of our numerical study on the cart-pole problem are provided in Table~\ref{tab:Cartpole}. 
In the ideal environment with zero uncertainty level, both the robust and non-robust policies perform well, achieving a $100\%$ success rate.   However, as the uncertainty level increases, the performance of the non-robust policy deteriorates significantly, exhibiting extremely poor performance in perturbed environments with an uncertainty level greater than three, where it achieves a zero success rate. In contrast, the robust policy shows resilience against the inherent randomness, maintaining a success rate around $50\%$ in our most perturbed environment. As expected, the robust policy computed using the DRMCPMD method significantly outperforms the non-robust policy derived through the MC policy gradient method in a perturbed environment. This observation highlights the robustness of our approach.

\section{Conclusion}
We proposed a new policy optimization algorithm, DRPMD, to solve RMDPs over general compact ambiguity sets. Our algorithm converges to the globally optimal robust policy under both direct and softmax policy parametrizations. Our established iteration complexities match or improve upon the best-known convergence rate in the literature of policy gradient methods applied to RMDPs. Moreover, we developed the first gradient-based solution method, TMA, and its stochastic variant, MCTMA, for solving the inner maximization with two novel transition parameterizations. In our experiments, our numerical results demonstrate the global convergence of DRPMD, its reliable performance against the non-robust approach, and its applicability to broader policy classes.



\bibliography{reference.bib}

\newpage
\appendix
\vskip 0.2in
\section{Auxiliary Lemmas}
\subsection{Properties of Markov Decision Process}\label{app:MDP-pre}
\begin{theorem}(Policy gradient theorem)\label{the:appA1-1}
Fix a map $\bm{\theta}\mapsto\pi^{\bm{\theta}}_{sa}$ that for any $(s,a)\in\mathcal{S}\times\mathcal{A}$ is differentiable and fix an initial distribution $\bm{\rho}\in\Delta^{S}$. Then
\begin{equation}
    \nabla_{\bm{\theta}}J_{\bm{\rho}}(\bm{\pi}^{\bm{\theta}},\bm{p}) ~=~ \frac{1}{1-\gamma}    \mathbb{E}_{s\sim\bm{d}^{\bm{\pi}^{\bm{\theta}},\bm{p}}_{\rho}}\left[\sum_{a\in\mathcal{A}}\nabla_{\bm{\theta}}\pi^{\bm{\theta}}_{sa}\cdot q^{\bm{\pi}^{\bm{\theta}},\bm{p}}_{sa}\right].
\end{equation}
\end{theorem}
\begin{proof}
See, for example, Theorem 1 in \cite{sutton2018reinforcement}.
\end{proof}
The following lemmas are helpful throughout the convergence analysis of policy optimization.
\begin{lemma}(First performance difference lemma~\citep{kakade2002approximately})\label{lem_sec3_3}
For any $\bm{\pi},\bm{\pi}'\in\Pi$, $\bm{p}\in\mathcal{P}$ and  $\bm{\rho}\in\Delta^{S}$, we have
\begin{equation}
   J_{\bm{\rho}}(\bm{\pi},\bm{p})-J_{\bm{\rho}}(\bm{\pi}',\bm{p}) ~=~\frac{1}{1-\gamma} \sum_{s, a} d_{\bm{\rho}}^{\bm{\pi},\bm{p}}(s) \pi_{sa} \left(q^{\bm{\pi}',\bm{p}}_{sa} - v^{\bm{\pi}',\bm{p}}_{s}\right).
\end{equation}
\end{lemma}
\begin{proof}
By the definition of $J_{\bm{\rho}}(\bm{\pi},\bm{p})$ in~\eqref{prob_RMDP}, we have
\begin{equation*}
    J_{\bm{\rho}}(\bm{\pi},\bm{p})-J_{\bm{\rho}}(\bm{\pi}',\bm{p}) = \sum_{s\in\mathcal{S}}\rho_{s}\left( v^{\bm{\pi},\bm{p}}_{s}-v^{\bm{\pi}',\bm{p}}_{s}\right).
\end{equation*}
Observe that, for any $s\in\mathcal{S}$,
\begin{align*} 
&v^{\bm{\pi},\bm{p}}_{s}-v^{\bm{\pi}',\bm{p}}_{s} \\
&= v^{\bm{\pi},\bm{p}}_{s}-  \sum_{a\in\mathcal{A}}\pi_{sa}\sum_{s'}p_{sas'}\left(c_{sas'}+\gamma v^{\bm{\pi}',\bm{p}}_{s'}\right)
+\sum_{a\in\mathcal{A}}\pi_{sa}\sum_{s'}p_{sas'}\left(c_{sas'}+\gamma v^{\bm{\pi}',\bm{p}}_{s'}\right)-
v^{\bm{\pi}',\bm{p}}_{s}\\
&=\sum_{a\in\mathcal{A}}\pi_{sa}\sum_{s'}p_{sas'}\left(c_{sas'}+\gamma v^{\bm{\pi},\bm{p}}_{s'}\right)-  \sum_{a\in\mathcal{A}}\pi_{sa}\sum_{s'}p_{sas'}\left(c_{sas'}+\gamma v^{\bm{\pi}',\bm{p}}_{s'}\right)\\
&\;\;\;\;+\sum_{a\in\mathcal{A}}\pi_{sa}\underbrace{\sum_{s'}p_{sas'}\left(c_{sas'}+\gamma v^{\bm{\pi}',\bm{p}}_{s'}\right)}_{q^{\bm{\pi}',\bm{p}}_{sa}}-
v^{\bm{\pi}',\bm{p}}_{s}.
\end{align*}
We notice that $q^{\bm{\pi}',\bm{p}}_{sa} = \sum_{s'}p_{sas'}\left(c_{sas'}+\gamma v^{\bm{\pi}',\bm{p}}_{s'}\right)$, then we have
\begin{align*}
v^{\bm{\pi},\bm{p}}_{s}-v^{\bm{\pi}',\bm{p}}_{s} &= \gamma\sum_{a\in\mathcal{A}}\pi_{sa}\sum_{s'}p_{sas'}\left(v^{\bm{\pi},\bm{p}}_{s'}-v^{\bm{\pi}',\bm{p}}_{s'} \right) + \sum_{a\in\mathcal{A}}\pi_{sa} \left(q^{\bm{\pi}',\bm{p}}_{sa}- v^{\bm{\pi}',\bm{p}}_{s}\right)\\
&= \gamma\sum_{a\in\mathcal{A}}\pi_{sa}\sum_{s'}p_{sas'}\left(v^{\bm{\pi},\bm{p}}_{s'}-v^{\bm{\pi}',\bm{p}}_{s'} \right) + \sum_{a\in\mathcal{A}}\pi_{sa} \psi^{\bm{\pi}',\bm{p}}_{sa}.     
\end{align*}
Using the recursion result and condensing the notations as
\begin{align*}
    \sum_{\hat{a}}\pi_{s\hat{a}} p_{s\hat{a}s'} &= p^{\bm{\pi}}_{ss'}(1),\\
    \sum_{s'\in\mathcal{S}}p^{\bm{\pi}}_{ss'}(t-1)\cdot\sum_{a\in\mathcal{A}}&\pi_{s'a} p_{s'as''} = p^{\bm{\pi}}_{ss''}(t),
\end{align*}
we then obtain
\begin{align*}
v^{\bm{\pi},\bm{p}}_{s}-v^{\bm{\pi}',\bm{p}}_{s} &= \gamma\sum_{s'} p^{\bm{\pi}}_{ss'}(1)\left(\gamma\sum_{s''}p^{\bm{\pi}}_{s's''}(1)\left(v^{\bm{\pi},\bm{p}}_{s''}-v^{\bm{\pi}',\bm{p}}_{s''} \right)+\sum_{a'}\pi_{s'a'} \psi^{\bm{\pi}',\bm{p}}_{s'a'} \right) + \sum_{a\in\mathcal{A}}\pi_{sa} \psi^{\bm{\pi}',\bm{p}}_{sa}\\
&= \sum_{a\in\mathcal{A}}\pi_{sa} \psi^{\bm{\pi}',\bm{p}}_{sa} + \gamma\sum_{s'} p^{\bm{\pi}}_{ss'}(1)\sum_{a'}\pi_{s'a'} \psi^{\bm{\pi}',\bm{p}}_{s'a'} + \gamma^{2}\sum_{s'} p^{\bm{\pi}}_{ss'}(2)\left(v^{\bm{\pi},\bm{p}}_{s'}-v^{\bm{\pi}',\bm{p}}_{s'} \right)\\
&= \cdots\\
&= \sum_{t=0}^{\infty}\gamma^{t}\sum_{s'}p^{\bm{\pi}}_{ss'}(t)\left( \sum_{a'}\pi_{s'a'} \psi^{\bm{\pi}',\bm{p}}_{s'a'}\right).
\end{align*}
Finally, we can prove this lemma
\begin{align*}
    J_{\bm{\rho}}(\bm{\pi},\bm{p})-J_{\bm{\rho}}(\bm{\pi}',\bm{p}) &= \sum_{s\in\mathcal{S}}\rho_{s}\left( v^{\bm{\pi},\bm{p}}_{s}-v^{\bm{\pi}',\bm{p}}_{s}\right)\\
    &= \sum_{s\in\mathcal{S}}\rho_{s}\sum_{t=0}^{\infty}\gamma^{t}\sum_{s'}p^{\bm{\pi}}_{ss'}(t)\left( \sum_{a'}\pi_{s'a'} \psi^{\bm{\pi}',\bm{p}}_{s'a'}\right)\\
    &= \sum_{s'}\left(\sum_{s\in\mathcal{S}}\sum_{t=0}^{\infty}\gamma^{t}\rho_{s}p^{\bm{\pi}}_{ss'}(t)\right)\left( \sum_{a'}\pi_{s'a'} \psi^{\bm{\pi}',\bm{p}}_{s'a'}\right)\\
    &= \frac{1}{1-\gamma} \sum_{s, a} d_{\bm{\rho}}^{\bm{\pi},\bm{p}}(s) \pi_{sa} \psi^{\bm{\pi}',\bm{p}}_{sa}.
\end{align*}
The last equality is obtained by the definition of state occupancy measure (See Definition~\ref{def:occu}).
\end{proof}

\begin{lemma}(Second performance difference lemma~\citep{mei2020global})\label{lem:2nd-performance-diff}
For any $\bm{\pi},\bm{\pi}'\in\Pi$, $\bm{p}\in\mathcal{P}$ and  $\bm{\rho}\in\Delta^{S}$, we have
\begin{equation}
   J_{\bm{\rho}}(\bm{\pi}',\bm{p})-J_{\bm{\rho}}(\bm{\pi},\bm{p}) ~=~\frac{1}{1-\gamma} \sum_{s\in\mathcal{S}} d_{\bm{\rho}}^{\bm{\pi},\bm{p}}(s)\sum_{a\in\mathcal{A}}\left(\pi^{\prime}_{sa}-\pi_{sa}\right)\cdot q^{\bm{\pi}',\bm{p}}_{sa}
\end{equation}    
\end{lemma}
\begin{proof}
    See Lemma 20 in~\cite{mei2020global}.
\end{proof}

\subsection{Standard Definition in Optimization}\label{app:Opt-results}
In this subsection, we present some standard optimization definitions~\citep{ghadimi2016accelerated,beck2017first}, which are used in our work. Consider the following optimization problem
\begin{equation}\label{eq:app-A5-1}
    \min_{\bm{x}\in\mathcal{X}} h(\bm{x})
\end{equation}
with $\mathcal{X}$ being a nonempty closed and convex set and $h\colon\mathbb{R}^{d}\rightarrow\mathbb{R}$ being proper, closed and $\ell$-smooth. We first introduce the crucial definitions of smoothness and Lipschitz continuity.
\begin{definition}\label{def:Lip}
A function $h: \mathcal{X}\rightarrow\mathbb{R}$ is \emph{$L$-Lipschitz} if for any $\bm{x}_{1},\bm{x}_{2}\in\mathcal{X}$, we have that $\|h(\bm{x}_{1})-h(\bm{x}_{2})\|\leq L\|\bm{x}_{1}-\bm{x}_{2}\|$, and \emph{$\ell$-smooth} if for any $\bm{x}_{1},\bm{x}_{2}\in\mathcal{X}$, we have $\|\nabla h(\bm{x}_{1})-\nabla h(\bm{x}_{2})\|\leq \ell\|\bm{x}_{1}-\bm{x}_{2}\|$.
\end{definition}
Another common definition we need to clarify is the indicator function.
\begin{definition}[Indicator functions]\label{def:indi}
    For any subset $\mathcal{X}\subseteq\mathbb{R}^{d}$, the indicator function of $\mathcal{X}$ is defined to be the extended real-valued function given by
\begin{equation*}
\mathbb{I}_{\mathcal{X}}(\bm{x})= \begin{cases}0, & \bm{x} \in \mathcal{X}, \\ \infty, & \bm{x} \notin \mathcal{X}.\end{cases}
\end{equation*}    
\end{definition}

\section{Deferred Proofs for Section~\ref{sec:DRPMD}}\label{app:tech-proof-sec3}
\subsection{Policy Gradient in Direct Parameterization}
Here, we provide the following result, which derives the analytical form of the direct parameterization policy gradient.
\begin{lemma}\label{lem:appC1-1}
    Direct parameterization policy gradient w.r.t. $\pi_{sa}$ for any $(s,a)\in\mathcal{S}\times\mathcal{A}$ is 
\begin{equation*}
    \frac{\partial J_{\bm{\rho}}(\bm{\pi},\bm{p})}{\partial \pi_{sa}} ~=~ \frac{1}{1-\gamma} \cdot d_{\bm{\rho}}^{\bm{\pi},\bm{p}}(s)\cdot q^{\bm{\pi},\bm{p}}_{sa}.
\end{equation*}
\end{lemma}
\begin{proof}
Notice that,
\begin{equation*}
    \frac{\partial J_{\bm{\rho}}(\bm{\pi},\bm{p})}{\partial \pi_{sa}} = \sum_{\hat{s} \in \mathcal{S}} \frac{\partial v^{\bm{\pi},\bm{p}}_{\hat{s}}}{\partial \pi_{sa}} \rho_{\hat{s}}.
\end{equation*}
Then, we discuss $\frac{\partial v^{\bm{\pi},\bm{p}}_{\hat{s}}}{\partial \pi_{sa}}$ over two cases: $\hat{s}\neq s$ and $\hat{s}= s$
\begin{align*}
    \frac{\partial v^{\bm{\pi},\bm{p}}_{\hat{s}}}{\partial \pi_{sa}}\Big\vert_{\hat{s}\neq s} &= \frac{\partial}{\partial \pi_{sa}}\left[\sum_{\hat{a}}\pi_{\hat{s}\hat{a}}\sum_{s' \in \mathcal{S}} p_{\hat{s}\hat{a}s'}\left(c_{\hat{s}\hat{a}s'}+\gamma v^{\bm{\pi},\bm{p}}_{s'}\right)\right]
    = \gamma\sum_{\hat{a}}\pi_{\hat{s}\hat{a}}\sum_{s' \in \mathcal{S}} p_{\hat{s}\hat{a}s'} \frac{\partial v^{\bm{\pi},\bm{p}}_{s'}}{\partial \pi_{sa}}, \\
    \frac{\partial v^{\bm{\pi},\bm{p}}_{\hat{s}}}{\partial \pi_{sa}}\Big\vert_{\hat{s}= s} &= \frac{\partial}{\partial \pi_{sa}}\left[\sum_{\hat{a}}\pi_{s\hat{a}}\underbrace{\sum_{s' \in \mathcal{S}} p_{s\hat{a}s'}\left(c_{s\hat{a}s'}+\gamma v^{\bm{\pi},\bm{p}}_{s'}\right)}_{q^{\bm{\pi},\bm{p}}_{s\hat{a}}}\right]
    = q^{\bm{\pi},\bm{p}}_{sa}+\gamma\sum_{\hat{a}}\pi_{s\hat{a}}\sum_{s' \in \mathcal{S}} p_{s\hat{a}s'} \frac{\partial v^{\bm{\pi},\bm{p}}_{s'}}{\partial \pi_{sa}}. 
\end{align*}
Condense the notation as
\begin{align}\label{inter_no_1}
    \sum_{\hat{a}}\pi_{s\hat{a}} p_{s\hat{a}s'} = p^{\bm{\pi}}_{ss'}(1),\quad
    \sum_{s'}p^{\bm{\pi}}_{ss'}(t-1)\cdot\sum_{a\in\mathcal{A}}\pi_{s'a} p_{s'as''} = p^{\bm{\pi}}_{ss''}(t),
\end{align}
and then, combining these two equations, we can obtain,
\begin{align*}
    \frac{\partial v^{\bm{\pi},\bm{p}}_{\hat{s}}}{\partial \pi_{sa}}\Big\vert_{\hat{s}\neq s} 
    &= \gamma\sum_{s'\neq s}p^{\bm{\pi}}_{\hat{s}s'}(1)\frac{\partial v^{\bm{\pi},\bm{p}}_{s'}}{\partial \pi_{sa}}+ \gamma\sum_{s'= s}p^{\bm{\pi}}_{\hat{s}s'}(1)\frac{\partial v^{\bm{\pi},\bm{p}}_{s'}}{\partial \pi_{sa}}\\
    &=\gamma^{2}\sum_{s'\neq s}p^{\bm{\pi}}_{\hat{s}s'}(1)\sum_{\hat{a}}\pi_{s'\hat{a}}\sum_{s'' \in \mathcal{S}} p_{s'\hat{a}s''} \frac{\partial v^{\bm{\pi},\bm{p}}_{s''}}{\partial \pi_{sa}}\\
    &+ \gamma p^{\bm{\pi}}_{\hat{s}s}(1)\left(q^{\bm{\pi},\bm{p}}_{sa}+\gamma\sum_{\hat{a}}\pi_{s\hat{a}}\sum_{s' \in \mathcal{S}} p_{s\hat{a}s'} \frac{\partial v^{\bm{\pi},\bm{p}}_{s'}}{\partial \pi_{sa}}\right)
    \\
    &= \gamma p^{\bm{\pi}}_{\hat{s}s}(1)q^{\bm{\pi},\bm{p}}_{sa}+ \gamma^{2}\sum_{s'}p^{\bm{\pi}}_{\hat{s}s'}(2)\frac{\partial v^{\bm{\pi},\bm{p}}_{s'}}{\partial \pi_{sa}}\\
    &= \gamma p^{\bm{\pi}}_{\hat{s}s}(1)q^{\bm{\pi},\bm{p}}_{sa}+ \gamma^{2}p^{\bm{\pi}}_{\hat{s}s}(2)q^{\bm{\pi},\bm{p}}_{sa}+\gamma^{3}\sum_{s'}p^{\bm{\pi}}_{\hat{s}s'}(3)\frac{\partial v^{\bm{\pi},\bm{p}}_{s'}}{\partial \pi_{sa}}\\
    &= \cdots\\
    &= \sum_{t=1}^{\infty}\gamma^{t}p^{\bm{\pi}}_{\hat{s}s}(t)q^{\bm{\pi},\bm{p}}_{sa} = \sum_{t=0}^{\infty}\gamma^{t}p^{\bm{\pi}}_{\hat{s}s}(t)q^{\bm{\pi},\bm{p}}_{sa}.
\end{align*}
The last equality is from the initial assumption $\hat{s} \neq s$, \ie, $p^{\bm{\pi}}_{\hat{s}s}(0)=0$, and similarly for the case $\hat{s}= s$ we have,
\begin{equation*}
    \frac{\partial v^{\bm{\pi},\bm{p}}_{\hat{s}}}{\partial \pi_{sa}}\Big\vert_{\hat{s}= s} = \sum_{t=0}^{\infty}\gamma^{t}p^{\bm{\pi}}_{ss}(t)q^{\bm{\pi},\bm{p}}_{sa}.
\end{equation*}
Hence, the partial derivative is obtained
\begin{align*}
    \frac{\partial J_{\bm{\rho}}(\bm{\pi},\bm{p})}{\partial \pi_{sa}} &= \left(\frac{\partial v^{\bm{\pi},\bm{p}}_{s}}{\partial \pi_{sa}} \rho_{s}+\sum_{\hat{s} \neq s} \frac{\partial v^{\bm{\pi},\bm{p}}_{\hat{s}}}{\partial \pi_{sa}} \rho_{\hat{s}}\right)
    = \frac{1}{1-\gamma}\left(\underbrace{(1-\gamma) \sum_{\hat{s}\in\mathcal{S}}\sum_{t=0}^{\infty} \gamma^{t} \rho_{\hat{s}}p^{\bm{\pi}}_{\hat{s}s}(t)}_{d_{\bm{\rho}}^{\bm{\pi},\bm{p}}(s)}\right)q^{\bm{\pi},\bm{p}}_{sa}.
\end{align*}
\end{proof}

\subsection{Linear Convergence Rate of DRPMD Under Rectangularity}\label{subapp:DRPMD-linear}
In this subsection, we discuss the linear convergence of DRPMD under different rectangularity conditions. For completeness, we here recall the definition of two common rectangular ambiguity sets from Section~\ref{sec:Inner-Rec}. An ambiguity set $\mathcal{P}$ is $(s,a)$-rectangular~\citep{iyengar2005robust,nilim2005robust,le2007robust} if it is a Cartesian product of sets $\mathcal{P}_{s,a}\subseteq\Delta^{S}$ for each state $s\in\mathcal{S}$ and action $a\in\mathcal{A}$, \ie,
\begin{equation*}
\mathcal{P}~:=~\{\bm{p}\in(\Delta^{S})^{S\times A}\mid\bm{p}_{s,a}\in\mathcal{P}_{s,a},\;\forall s\in\mathcal{S},a\in\mathcal{A}\},
\end{equation*}
whereas an ambiguity set $\mathcal{P}$ is $s$-rectangular~\citep{wiesemann2013robust} if it is defined as a Cartesian product of
sets $\mathcal{P}_{s}\subseteq(\Delta^{S})^{A}$, \ie,
\begin{equation*}
\mathcal{P}~:=~\{\bm{p}\in(\Delta^{S})^{S\times A}\mid\bm{p}_{s}=(\bm{p}_{s,a})_{a\in\mathcal{A}}\in\mathcal{P}_{s},\;\forall s\in\mathcal{S}\}.
\end{equation*}

A fundamental result for analyzing first-order methods addressing rectangular RMDPs is the \textit{robust performance difference lemma}~\citep{li2022first,kumar2023towards}. We then summarize different forms of robust performance difference lemma within the following lemma. 
\begin{lemma}(Robust performance difference under rectangularity)\label{lem:robust-performance-diff}
Consider any pair of policies $\bm{\pi}_{1},\bm{\pi}_{2}\in\Pi$ and their corresponding inner worst-case transition kernel
\begin{equation*}
\bm{p}_{1}:= 
\mathop{\arg\max}\limits_ {\bm{p}\in\mathcal{P}}J_{\bm{\rho}}(\bm{\pi}_{1},\bm{p}),\quad\bm{p}_{2}:= \mathop{\arg\max}\limits_ {\bm{p}\in\mathcal{P}}J_{\bm{\rho}}(\bm{\pi}_{2},\bm{p}).    
\end{equation*}
Then, for RMDPs with the $(s,a)$-rectangular ambiguity set, we have
\begin{equation*}
    \Phi(\bm{\pi}_{2})-\Phi(\bm{\pi}_{1})\leq \frac{1}{1-\gamma} \sum_{s\in\mathcal{S}} d_{\bm{\rho}}^{\bm{\pi}_{2},\bm{p}_{2}}(s)\sum_{a\in\mathcal{A}}\left(\pi_{2,sa}-\pi_{1,sa}\right)\cdot q^{\bm{\pi}_{1},\bm{p}_{1}}_{sa}.
\end{equation*}
For RMDPs with the $s$-rectangular ambiguity set, we have
\begin{equation*}
    \Phi(\bm{\pi}_{2})-\Phi(\bm{\pi}_{1})\leq \sum_{s\in\mathcal{S}} d_{\bm{\rho}}^{\bm{\pi}_{2},\bm{p}_{1}}(s)\sum_{a\in\mathcal{A}}\left(\pi_{2,sa}-\pi_{1,sa}\right)\cdot q^{\bm{\pi}_{1},\bm{p}_{1}}_{sa}.    
\end{equation*}
\end{lemma}
\begin{proof}
See Lemma 2.6 of~\cite{li2022first} for RMDPs with the $(s,a)$-rectangular set and Lemma 2 of~\cite{kumar2023towards} for RMDPs with the $s$-rectangular set.
\end{proof}
The robust performance difference lemma introduces a characterization that upper bounds the difference of robust values for two policies and will be helpful in the convergence analysis. We are now ready to provide the complete proof of Theorem~\ref{the:linear-srec}, which establishes the convergence guarantee of DRPMD for $s$-rectangular RMDPs while assuming the inner maximization problem could be solved exactly.
\begin{proof}[\textbf{of Theorem~\ref{the:linear-srec}}]
We still denote $\bm{\pi}^{\star} = (\bm{\pi}^{\star}_{s})_{s\in\mathcal{S}}$ be the globally optimal policy. An essential inequality in the convergence analysis of Theorem~\ref{the:sublinear-dir-para} is
\begin{equation}\label{eq:subappC2-mid1}
        \underbrace{\langle \bm{q}^{\bm{\pi}_{t},\bm{p}_{t}}_{s}, \bm{\pi}_{t,s}-\bm{\pi}^{\star}_{s}\rangle}_{\text{(A)}}+\underbrace{\langle \bm{q}^{\bm{\pi}_{t},\bm{p}_{t}}_{s}, \bm{\pi}_{t+1,s}-\bm{\pi}_{t,s}\rangle}_{\text{(B)}}  \leq  \frac{1}{\alpha_{t}}B(\bm{\pi}^{\star}_{s},\bm{\pi}_{t,s}) - \frac{1}{\alpha_{t}}B(\bm{\pi}^{\star}_{s},\bm{\pi}_{t+1,s}).
    \end{equation}
    To handle term (A), we use the existing result 
\begin{align*}
        \Phi(\bm{\pi}_{t}) - \Phi(\bm{\pi}^{\star}) &= 
\max_{\bm{p}\in\mathcal{P}}J_{\bm{\rho}}(\bm{\pi}_{t},\bm{p})-\max_{\bm{p}\in\mathcal{P}}J_{\bm{\rho}}(\bm{\pi}^{\star},\bm{p})\\
   &\leq J_{\bm{\rho}}(\bm{\pi}_{t},\bm{p}_{t})-J_{\bm{\rho}}(\bm{\pi}^{\star},\bm{p}_{t})\\
   &= \frac{1}{1-\gamma} \sum_{s\in\mathcal{S}} d_{\bm{\rho}}^{\bm{\pi}^{\star},\bm{p}_{t}}(s)\sum_{a\in\mathcal{A}}\left(\pi_{t,sa}-\pi^{\star}_{sa}\right)\cdot q^{\bm{\pi}_{t},\bm{p}_{t}}_{sa},
\end{align*}
where $\bm{p}_{t}:= 
\mathop{\arg\max}\limits_ {\bm{p}\in\mathcal{P}}J_{\bm{\rho}}(\bm{\pi}_{t},\bm{p})$ for any $t\geq0$. Different from the analysis of Theorem~\ref{the:sublinear-dir-para} on the term (B), we apply Lemma~\ref{lem:robust-performance-diff} to obtain that
\begin{align*}
    \Phi(\bm{\pi}_{t+1})-\Phi(\bm{\pi}_{t})&\leq \sum_{s\in\mathcal{S}} d_{\bm{\rho}}^{\bm{\pi}_{t+1},\bm{p}_{t}}(s)\sum_{a\in\mathcal{A}}\left(\pi_{t+1,sa}-\pi_{t,sa}\right)\cdot q^{\bm{\pi}_{t},\bm{p}_{t}}_{sa}\\
    &= \sum_{s\in\mathcal{S}} \frac{d_{\bm{\rho}}^{\bm{\pi}_{t+1},\bm{p}_{t}}(s)}{d^{\bm{\pi}^{\star},\bm{p}_{t}}_{\bm{\rho}}(s)}d^{\bm{\pi}^{\star},\bm{p}_{t}}_{\bm{\rho}}(s)\sum_{a\in\mathcal{A}}\left(\pi_{t+1,sa}-\pi_{t,sa}\right)\cdot q^{\bm{\pi}_{t},\bm{p}_{t}}_{sa}.
\end{align*}
We notice that, by plugging in $\bm{y} = \bm{\pi}_{t}$ in~\eqref{eq:3-point-descent}, we have
\begin{equation*}
\sum_{a\in\mathcal{A}}\left(\pi_{t+1,sa}-\pi_{t,sa}\right)\cdot q^{\bm{\pi}_{t},\bm{p}_{t}}_{sa} \leq -B(\bm{\pi}_{t+1,s},\bm{\pi}_{t,s}) - B(\bm{\pi}_{t,s},\bm{\pi}_{t+1,s}) \leq 0.
\end{equation*}
Then, $\Phi(\bm{\pi}_{t+1})-\Phi(\bm{\pi}_{t})$ can be bounded as follow:
\begin{align*}
\Phi(\bm{\pi}_{t+1})-\Phi(\bm{\pi}_{t})&\leq \left(\min_{s\in\mathcal{S}}\left(\frac{d_{\bm{\rho}}^{\bm{\pi}_{t+1},\bm{p}_{t}}(s)}{d^{\bm{\pi}^{\star},\bm{p}_{t}}_{\bm{\rho}}(s)}\right)\right)  \sum_{s\in\mathcal{S}}d^{\bm{\pi}^{\star},\bm{p}_{t}}_{\bm{\rho}}(s)\sum_{a\in\mathcal{A}}\left(\pi_{t+1,sa}-\pi_{t,sa}\right)\cdot q^{\bm{\pi}_{t},\bm{p}_{t}}_{sa} \\
&= \left\|\frac{\bm{d}^{\bm{\pi}^{\star},\bm{p}_{t}}_{\bm{\rho}}}{\bm{d}_{\bm{\rho}}^{\bm{\pi}_{t+1},\bm{p}_{t}}}\right\|^{-1}  \sum_{s\in\mathcal{S}}d^{\bm{\pi}^{\star},\bm{p}_{t}}_{\bm{\rho}}(s)\sum_{a\in\mathcal{A}}\left(\pi_{t+1,sa}-\pi_{t,sa}\right)\cdot q^{\bm{\pi}_{t},\bm{p}_{t}}_{sa}\\
&\leq \frac{1-\gamma}{M} \sum_{s\in\mathcal{S}}d^{\bm{\pi}^{\star},\bm{p}_{t}}_{\bm{\rho}}(s)\sum_{a\in\mathcal{A}}\left(\pi_{t+1,sa}-\pi_{t,sa}\right)\cdot q^{\bm{\pi}_{t},\bm{p}_{t}}_{sa}.
\end{align*}
Then, by aggregating the above inequality across different states with weights set as $d^{\bm{\pi}^{\star},\bm{p}_{t}}_{\bm{\rho}}(s)$ on~\eqref{eq:subappC2-mid1} and incorporating the above results, we have
\begin{align*}
\frac{M}{1-\gamma}\left(\Phi(\bm{\pi}_{t+1})-\Phi(\bm{\pi}_{t})\right) + (1-\gamma)\left(\Phi(\bm{\pi}_{t}) - \Phi(\bm{\pi}^{\star})\right) \leq  \frac{1}{\alpha_{t}}\left(B_{\bm{d}^{\bm{\pi}^{\star},\bm{p}_{t}}_{\bm{\rho}}}(\bm{\pi}^{\star},\bm{\pi}_{t}) - B_{\bm{d}^{\bm{\pi}^{\star},\bm{p}_{t}}_{\bm{\rho}}}(\bm{\pi}^{\star},\bm{\pi}_{t+1})\right)
\end{align*}
By recursively applying the above inequality from $t=0$ to $t=k-1$, we obtain
\begin{align*}
\Phi(\bm{\pi}_{k}) - \Phi(\bm{\pi}^{\star}) &\leq (1-\frac{(1-\gamma)^{2}}{M})^{k}\left(\Phi(\bm{\pi}_{0}) - \Phi(\bm{\pi}^{\star})\right)\\
&\;\;+\frac{1-\gamma}{M}\sum^{k-1}_{t=0}\left[(1-\frac{(1-\gamma)^{2}}{M})^{k-1-t}\frac{1}{\alpha_{t}}B_{\bm{d}^{\bm{\pi}^{\star},\bm{p}_{t}}_{\bm{\rho}}}(\bm{\pi}^{\star},\bm{\pi}_{t})\right]\\
&\;\;- \frac{1-\gamma}{M}\sum^{k}_{t=1}\left[(1-\frac{(1-\gamma)^{2}}{M})^{k-t}\frac{1}{\alpha_{t-1}}B_{\bm{d}^{\bm{\pi}^{\star},\bm{p}_{t-1}}_{\bm{\rho}}}(\bm{\pi}^{\star},\bm{\pi}_{t})\right].
\end{align*}
Rearrange the above inequality and we can obtain
\begin{align*}
\Phi(\bm{\pi}_{k}) - \Phi(\bm{\pi}^{\star}) &\leq (1-\frac{(1-\gamma)^{2}}{M})^{k}\left(\Phi(\bm{\pi}_{0}) - \Phi(\bm{\pi}^{\star})\right)+\frac{1-\gamma}{M}(1-\frac{(1-\gamma)^{2}}{M})^{k-1}B_{\bm{d}^{\bm{\pi}^{\star},\bm{p}_{0}}_{\bm{\rho}}}(\bm{\pi}^{\star},\bm{\pi}_{0})\\
&\;\;+\frac{1-\gamma}{M}\sum^{k-1}_{t=1}(1-\frac{(1-\gamma)^{2}}{M})^{k-1-t}\frac{1}{\alpha_{t}}B_{\bm{d}^{\bm{\pi}^{\star},\bm{p}_{t}}_{\bm{\rho}}}(\bm{\pi}^{\star},\bm{\pi}_{t})\\
&\;\;-\frac{1-\gamma}{M}\sum^{k-1}_{t=1}(1-\frac{(1-\gamma)^{2}}{M})^{k-t}\frac{1}{\alpha_{t-1}}B_{\bm{d}^{\bm{\pi}^{\star},\bm{p}_{t-1}}_{\bm{\rho}}}(\bm{\pi}^{\star},\bm{\pi}_{t}).
\end{align*}
Given the definition of step sizes $\{\alpha_{t}\}_{t\geq0}$ in~\eqref{def:linear-stepsize-s}, we have
\begin{equation*}
(1-\frac{(1-\gamma)^{2}}{M})^{k-1-t}\frac{1}{\alpha_{t}}B_{\bm{d}^{\bm{\pi}^{\star},\bm{p}_{t}}_{\bm{\rho}}}(\bm{\pi}^{\star},\bm{\pi}_{t})-(1-\frac{(1-\gamma)^{2}}{M})^{k-t}\frac{1}{\alpha_{t-1}}B_{\bm{d}^{\bm{\pi}^{\star},\bm{p}_{t-1}}_{\bm{\rho}}}(\bm{\pi}^{\star},\bm{\pi}_{t})\leq 0.    
\end{equation*}
Therefore, we can reach the desired result that when~\eqref{def:linear-stepsize-s} holds,
\begin{align*}
\Phi(\bm{\pi}_{k}) - \Phi(\bm{\pi}^{\star}) &\leq (1-\frac{(1-\gamma)^{2}}{M})^{k}\left(\Phi(\bm{\pi}_{0}) - \Phi(\bm{\pi}^{\star})\right)+\frac{1-\gamma}{M}(1-\frac{(1-\gamma)^{2}}{M})^{k-1}B_{\bm{d}^{\bm{\pi}^{\star},\bm{p}_{0}}_{\bm{\rho}}}(\bm{\pi}^{\star},\bm{\pi}_{0})\\
&= (1-\frac{(1-\gamma)^{2}}{M})^{k}\left(\Phi(\bm{\pi}_{0}) - \Phi(\bm{\pi}^{\star}) + \frac{(1-\gamma)B_{\bm{d}^{\bm{\pi}^{\star},\bm{p}_{0}}_{\bm{\rho}}}(\bm{\pi}^{\star},\bm{\pi}_{0})}{M-(1-\gamma)^{2}}\right).
\end{align*}
\end{proof}
For completeness, we also provide the linear convergence result for $(s,a)$-rectangular RMDPs as well. A similar discussion can be found in~\cite{li2022first}.
\begin{theorem}(Linear convergence for $(s,a)$-rectangular RMDPs)\label{the:linear-sarec}
Suppose the step sizes $\{\alpha_{t}\}_{t\geq0}$ satisfy
\begin{equation}\label{def:linear-stepsize}
    \alpha_{t}\geq\frac{M}{1-\gamma}\cdot\left(1-\frac{1-\gamma}{M}\right)^{-1}\cdot\alpha_{t-1},\quad\forall t\geq1.
\end{equation}
Then, for any iteration $k$ and an initial tolerance $\epsilon_{0}\geq0$, DRPMD produces policy $\bm{\pi}_{k}$ for $(s,a)$-rectangular RMDPs satisfying
\begin{align*}
\Phi(\bm{\pi}_{k}) - \Phi(\bm{\pi}^{\star})\leq (1-\frac{1-\gamma}{M})^{k}\left(\Phi(\bm{\pi}_{0}) - \Phi(\bm{\pi}^{\star}) + (1-\frac{1-\gamma}{M})^{-1}\cdot\frac{B_{\bm{d}^{\bm{\pi}^{\star},\bm{p}_{0}}_{\bm{\rho}}}(\bm{\pi}^{\star},\bm{\pi}_{0})}{M}\right).
\end{align*}   
\end{theorem}
\begin{proof}
We also let $\bm{\pi}^{\star} = (\bm{\pi}^{\star}_{s})_{s\in\mathcal{S}}$ be the optimal robust policy and obtain the following inequality from the proof of Theorem~\ref{the:sublinear-dir-para}
\begin{equation}\label{eq:subappC2-mid2}
        \underbrace{\langle \bm{q}^{\bm{\pi}_{t},\bm{p}_{t}}_{s}, \bm{\pi}_{t,s}-\bm{\pi}^{\star}_{s}\rangle}_{\text{(A)}}+\underbrace{\langle \bm{q}^{\bm{\pi}_{t},\bm{p}_{t}}_{s}, \bm{\pi}_{t+1,s}-\bm{\pi}_{t,s}\rangle}_{\text{(B)}}  \leq  \frac{1}{\alpha_{t}}B(\bm{\pi}^{\star}_{s},\bm{\pi}_{t,s}) - \frac{1}{\alpha_{t}}B(\bm{\pi}^{\star}_{s},\bm{\pi}_{t+1,s}).
    \end{equation}
    As for the term (A), we have
\begin{align*}
        \Phi(\bm{\pi}_{t}) - \Phi(\bm{\pi}^{\star}) &\leq \frac{1}{1-\gamma} \sum_{s\in\mathcal{S}} d_{\bm{\rho}}^{\bm{\pi}^{\star},\bm{p}_{t}}(s)\sum_{a\in\mathcal{A}}\left(\pi_{t,sa}-\pi^{\star}_{sa}\right)\cdot q^{\bm{\pi}_{t},\bm{p}_{t}}_{sa},
\end{align*}
where $\bm{p}_{t}:= 
\mathop{\arg\max}\limits_ {\bm{p}\in\mathcal{P}}J_{\bm{\rho}}(\bm{\pi}_{t},\bm{p})$ for any $t\geq0$. As for the term (B),  by applying the result of Lemma~\ref{lem:robust-performance-diff} and the definition of $M:=\sup_{\bm{\pi}\in\Pi,\bm{p}\in\mathcal{P}}\left\|\nicefrac{\bm{d}_{\bm{\rho}}^{\bm{\pi},\bm{p}}}{\bm{\rho}}\right\|_{\infty}$, we have
\begin{align*}
    \Phi(\bm{\pi}_{t+1})-\Phi(\bm{\pi}_{t})&\leq \frac{1}{1-\gamma} \sum_{s\in\mathcal{S}} d_{\bm{\rho}}^{\bm{\pi}_{t+1},\bm{p}_{t+1}}(s)\sum_{a\in\mathcal{A}}\left(\pi_{t+1,sa}-\pi_{t,sa}\right)\cdot q^{\bm{\pi}_{t},\bm{p}_{t}}_{sa}\\
    &= \frac{1}{1-\gamma} \sum_{s\in\mathcal{S}} \frac{d_{\bm{\rho}}^{\bm{\pi}_{t+1},\bm{p}_{t+1}}(s)}{d^{\bm{\pi}^{\star},\bm{p}_{t}}_{\bm{\rho}}(s)}d^{\bm{\pi}^{\star},\bm{p}_{t}}_{\bm{\rho}}(s)\sum_{a\in\mathcal{A}}\left(\pi_{t+1,sa}-\pi_{t,sa}\right)\cdot q^{\bm{\pi}_{t},\bm{p}_{t}}_{sa}\\
    &\leq \frac{1}{M} \sum_{s\in\mathcal{S}}d^{\bm{\pi}^{\star},\bm{p}_{t}}_{\bm{\rho}}(s)\sum_{a\in\mathcal{A}}\left(\pi_{t+1,sa}-\pi_{t,sa}\right)\cdot q^{\bm{\pi}_{t},\bm{p}_{t}}_{sa}.
\end{align*}
Then, by taking expectation with respect to $s\sim\bm{d}^{\bm{\pi}^{\star},\bm{p}_{t}}_{\bm{\rho}}$ on~\eqref{eq:subappC2-mid2} and aggregating the above results, we have
\begin{align*}
M\left(\Phi(\bm{\pi}_{t+1})-\Phi(\bm{\pi}_{t})\right) + (1-\gamma)\left(\Phi(\bm{\pi}_{t}) - \Phi(\bm{\pi}^{\star})\right) \leq  \frac{1}{\alpha_{t}}\left(B_{\bm{d}^{\bm{\pi}^{\star},\bm{p}_{t}}_{\bm{\rho}}}(\bm{\pi}^{\star},\bm{\pi}_{t}) - B_{\bm{d}^{\bm{\pi}^{\star},\bm{p}_{t}}_{\bm{\rho}}}(\bm{\pi}^{\star},\bm{\pi}_{t+1})\right)
\end{align*}
By recursively applying the above inequality from $t=0$ to $t=k-1$, we obtain
\begin{align*}
\Phi(\bm{\pi}_{k}) - \Phi(\bm{\pi}^{\star}) &\leq (1-\frac{1-\gamma}{M})^{k}\left(\Phi(\bm{\pi}_{0}) - \Phi(\bm{\pi}^{\star})\right)\\
&\;\;+\frac{1}{M}\sum^{k-1}_{t=0}\left[(1-\frac{1-\gamma}{M})^{k-1-t}\frac{1}{\alpha_{t}}B_{\bm{d}^{\bm{\pi}^{\star},\bm{p}_{t}}_{\bm{\rho}}}(\bm{\pi}^{\star},\bm{\pi}_{t})\right]\\
&\;\;- \frac{1}{M}\sum^{k}_{t=1}\left[(1-\frac{1-\gamma}{M})^{k-t}\frac{1}{\alpha_{t-1}}B_{\bm{d}^{\bm{\pi}^{\star},\bm{p}_{t-1}}_{\bm{\rho}}}(\bm{\pi}^{\star},\bm{\pi}_{t})\right].
\end{align*}
Rearrange the above inequality to obtain
\begin{align*}
\Phi(\bm{\pi}_{k}) - \Phi(\bm{\pi}^{\star}) &\leq (1-\frac{1-\gamma}{M})^{k}\left(\Phi(\bm{\pi}_{0}) - \Phi(\bm{\pi}^{\star})\right)+\frac{1}{M}(1-\frac{1-\gamma}{M})^{k-1}B_{\bm{d}^{\bm{\pi}^{\star},\bm{p}_{0}}_{\bm{\rho}}}(\bm{\pi}^{\star},\bm{\pi}_{0})\\
&\;\;+\frac{1}{M}\sum^{k-1}_{t=1}(1-\frac{1-\gamma}{M})^{k-1-t}\left[\frac{1}{\alpha_{t}}B_{\bm{d}^{\bm{\pi}^{\star},\bm{p}_{t}}_{\bm{\rho}}}(\bm{\pi}^{\star},\bm{\pi}_{t})-(1-\frac{1-\gamma}{M})\frac{1}{\alpha_{t-1}}B_{\bm{d}^{\bm{\pi}^{\star},\bm{p}_{t-1}}_{\bm{\rho}}}(\bm{\pi}^{\star},\bm{\pi}_{t})\right].    
\end{align*}
Given the definition of step sizes $\{\alpha_{t}\}_{t\geq0}$ in~\eqref{def:linear-stepsize}, the last term in the right-hand side of the above inequality should be non-positive. Therefore, we can reach the desired result that when~\eqref{def:linear-stepsize} holds,
\begin{align*}
\Phi(\bm{\pi}_{k}) - \Phi(\bm{\pi}^{\star}) &\leq (1-\frac{1-\gamma}{M})^{k}\left(\Phi(\bm{\pi}_{0}) - \Phi(\bm{\pi}^{\star})\right)+\frac{1}{M}(1-\frac{1-\gamma}{M})^{k-1}B_{\bm{d}^{\bm{\pi}^{\star},\bm{p}_{0}}_{\bm{\rho}}}(\bm{\pi}^{\star},\bm{\pi}_{0})\\
&\leq (1-\frac{1-\gamma}{M})^{k}\left(\Phi(\bm{\pi}_{0}) - \Phi(\bm{\pi}^{\star}) + (1-\frac{1-\gamma}{M})^{-1}\cdot\frac{B_{\bm{d}^{\bm{\pi}^{\star},\bm{p}_{0}}_{\bm{\rho}}}(\bm{\pi}^{\star},\bm{\pi}_{0})}{M}\right).
\end{align*}
\end{proof}

\subsection{Policy Gradient in Softmax Parameterization}
Here, we give the form of policy gradient under the softmax parameterization for completeness.
\begin{lemma}\label{lem:sec4_2_1}
    Softmax parameterization gradient w.r.t. $\bm{\theta}$ is 
\begin{equation*}
    \frac{\partial J_{\bm{\rho}}(\bm{\pi}^{\bm{\theta}},\bm{p})}{\partial \theta_{sa}} ~=~ \frac{1}{1-\gamma} \cdot d_{\bm{\rho}}^{\bm{\pi}^{\bm{\theta}},\bm{p}}(s) \cdot\pi^{\bm{\theta}}_{sa}\cdot \psi^{\bm{\pi}^{\bm{\theta}},\bm{p}}_{sa}.
\end{equation*}    
\end{lemma}
\begin{proof}
    See for example Lemma 1 of~\cite{mei2020global} or Lemma C.1 of~\cite{agarwal2021theory}.
\end{proof}
A crucial technical lemma is introduced as follow, playing an important role in our analysis of DRPMD with softmax parameterization.
\begin{lemma}\label{lem:PG_seq}
    Consider the optimization problem $\min_{\bm{x}\in\mathbb{R}^{d}}f(\bm{x})$ where the continuously differentiable function $f\colon \mathbb{R}^{d}\rightarrow\mathbb{S}\in\mathbb{R}$ is $L$-Lipschitz and $\ell$-smooth. Then, given an arbitrary $\hat{\bm{x}}\in\mathbb{R}^{d}$, for $k\geq 0$, there exists a sequence $\{\bm{x}_{n}\}_{n\geq0}$, generated by the gradient descent update $\bm{x}_{n+1}=\bm{x}_{n} - \alpha\nabla f(\bm{x}_{n})$ with constant step size $\alpha>0$, while the $k$-th iterative point $\bm{x}_{k}$ exactly equals to $\hat{\bm{x}}$.
\end{lemma}
\begin{proof}
We prove this lemma by induction. When $k=0$, it implies that $\hat{\bm{x}}$ is assumed to be the initial point, so we can generate a sequence $\{\bm{x}_{n}\}_{n\geq1}$ with $\bm{x}_{0} = \hat{\bm{x}}$ using the updating rule $\bm{x}_{n+1}=\bm{x}_{n} - \alpha\nabla f(\bm{x}_{n})$. Next, we assume that, for $k=t-1$, there exists a sequence $\{\bm{y}_{n}\}_{n\geq0}$ such that $\bm{y}_{t-1}=\hat{\bm{x}}$, and consider the case $k=t$. To generate a sequence $\{\bm{z}_{n}\}_{n\geq0}$ such that $\bm{z}_{t}=\hat{\bm{x}}$, we consider the following scheme. First, it is natural to generate a subsequence $\{\bm{z}_{n}\}_{n> t}$ for $n> t$, step by step from $\bm{z}_{t}$ using the update rule $\bm{z}_{n+1}=\bm{z}_{n} - \alpha\nabla f(\bm{z}_{n})$. Then, we  generate another subsequence $\{\bm{z}_{n}\}_{n< t}$ for $n<t$. To do so, we begin to find $\bm{z}_{t-1}\in\operatorname{dom}(f)$ which satisfies the following update rule 
\begin{equation*}
    \bm{z}_{t}
  ~=~
  \bm{z}_{t-1} - \alpha\cdot\nabla f(\bm{z}_{t-1}).    
\end{equation*}
Note that $\bm{z}_{t}$ in the above equation is fixed as $\hat{\bm{x}}$. Let us define two functions $h_{1}(\bm{x})= \nabla f(\bm{x})\colon \mathbb{R}^{d}\rightarrow\mathbb{R}$ and $h_{2}(\bm{x})= \frac{1}{\alpha}\left(\bm{x}-\bm{z}_{t}\right) \colon \mathbb{R}^{d}\rightarrow\mathbb{R}$. Since $f$ is $L$-Lipschitz and $\ell$-smooth, it implies that $h_{1}(\bm{x})$ is a bounded and $\ell$-Lipschitz continuous function, while $h_{2}(\bm{x})$ is a simple linear function. Therefore, we can obtain that there exists at least one point $\bm{x}^{\star}:= \bm{z}_{t-1}$ such that $h_{1}(\bm{x}^{\star}) = h_{2}(\bm{x}^{\star})$. Once we obtain $\bm{z}_{t-1}$, by using the assumption $k=t-1$, we can immediately obtain a sequence, denoted as $\{\hat{\bm{y}}_{n}\}_{n\geq 0}$, satisfying $\hat{\bm{y}}_{t-1}:= \bm{z}_{t-1}$. Therefore, the sequence $\{\bm{z}_{n}\}_{n\geq 0}$ becomes
\begin{equation*}
    \{\bm{z}_{n}\}_{n\geq 0}:=\{\hat{\bm{y}}_{0},\cdots,\hat{\bm{y}}_{t-2}, \bm{x}^{\star} = \bm{z}_{t-1}, \hat{\bm{x}} = \bm{z}_{t},\cdots\}
\end{equation*}
so that the proof is completed.
\end{proof}

\vskip 0.2in
\section{Entropy Transition Parameterization}\label{app:entropy_trans}
We now discuss the relationship between the proposed entropy transition parameterization and the solution to the policy evaluation problem of the $(s,a)$-rectangular RMDPs with the KL-divergence constraint. In \cite{nilim2005robust}, the RMDP aims to compute the worst-case performance with respect to the ambiguous transition lying in a $(s,a)$-rectangular ambiguity set
\begin{equation*}
\min_{\bm{\pi}\in\Pi}\max_{\bm{p}\in\mathcal{P}} \, \left\{J_{\bm{\rho}}(\bm{\pi},\bm{p})~:=~ \mathbb{E}_{\bm{\pi},\bm{p},\tilde{s}_0 \sim\bm{\rho}}\left[\sum_{t=0}^{\infty} \gamma^{t}\cdot  c_{\tilde{s}_{t} \tilde{a}_{t}}\right]\right\}.  
\end{equation*}
The fundamental solution method to find the optimal robust value function is the \emph{robust value iteration}, which repeatedly computes the sequence $\bm{v}_{t+1} = \mathcal{T}\bm{v}_{t}$ with any initial values $\bm{v}_{0}$. Here, $\mathcal{T}$ is called the robust Bellman optimality operator~\citep{ho2021partial} and defined for each state $s\in\mathcal{S}$ as
\begin{equation*}
    (\mathcal{T}\bm{v})_{s} ~:=~ \min_{a\in\mathcal{A}}\left(c_{sa}+\gamma\max_{\bm{p}_{sa}\in\mathcal{P}_{sa}}\bm{p}_{sa}^{\top}\bm{v}\right),
\end{equation*}
Each step of the robust value iteration involves the solution of an optimization problem, referred to in this paper as the inner robust evaluation problem. This problem is decoupled across state-action pairs and has the following form
\begin{equation*}
\max_{\bm{p}\in\mathcal{P}_{sa}}\bm{p}^{\top}\bm{v},\quad\forall (s,a)\in\mathcal{S}\times\mathcal{A}.
\end{equation*}
Consider the $(s,a)$-rectangular ambiguity set with KL-divergence constraint
\begin{equation*}
    \mathcal{P}_{sa}~:=~\{\bm{p}\in\Delta^{S} \mid D\left(\bm{p}\|\bm{p}_{c}\right) \leq \kappa\},
\end{equation*}
where $\kappa>0$ is fixed, $\bm{q}$ is a given distribution, and $D\left(\bm{p}\|\bm{q}\right)$denotes the Kullback-Leibler (KL) divergence from $\bm{q}$ to $\bm{p}$, \ie,
\begin{equation*}
D\left(\bm{p}\|\bm{q}\right) ~:=~\sum_{s\in\mathcal{S}} p_{s}\log\left(\frac{p_{s}}{q_{s}} \right). 
\end{equation*}
Then, by standard duality arguments, the inner problem is equivalent to its dual
\begin{equation*}
\min_{\lambda>0, \mu} \mu+\kappa \lambda+\lambda \sum_{s\in\mathcal{S}} q_{s} \exp \left(\frac{v_{s}-\mu}{\lambda}-1\right).    
\end{equation*}
Then, the optimal distribution is
\begin{equation*}
  p\opt
  =\frac{q_{s} \exp ( \frac{v_{s}}{\lambda})}{\sum_{s\in\mathcal{S}} q_{s} \exp (\frac{v_{s}}{\lambda})}.
\end{equation*}
While we introduce the linear approximation on $v_{s}$ and $\lambda$, we can immediately have the entropy transition parameterization form.  

\vskip 0.2in
\section{Details of Monte-Carlo Transition Gradient Method}\label{app:MCTG}

    In this case, we fix the policy $\bm{\pi}$ to choose action and consider the gradient of the inner problem under the episodic case. Denote trajectory $\tau = (S_{0},A_{0},R_{1},S_{1},\cdots,S_{T})$ and the total reward of this trajectory as $G_{\tau}= \sum^{T}_{i=1}\gamma^{i}R_{i}$. Then, for the inner problem, the value of transition $\bm{p}^{\bm{\xi}}$ is
\begin{equation*}
    J_{\bm{\rho}}(\bm{\pi}, \bm{p}^{\bm{\xi}}) ~=~ \mathbb{E}_{\bm{\pi},\bm{p}^{\bm{\xi}}} \left[G_{\tau}\right] ~=~ \sum_{\tau}P^{\bm{\pi},\bm{p}^{\bm{\xi}}}(\tau)G_{\tau},
\end{equation*}
where $P^{\bm{\pi},\bm{p}^{\bm{\xi}}}(\tau)$ is the probability of having trajectory $\tau$ under policy $\bm{\pi}$ and transition $\bm{p}^{\bm{\xi}}$. The gradient of the inner is then 
\begin{align*}
      \nabla_{\bm{\xi}} J(\bm{\pi}, \bm{p}^{\bm{\xi}}) &=  \nabla_{\bm{\xi}}\sum_{\tau}G_{\tau} P^{\bm{\pi},\bm{p}^{\bm{\xi}}}(\tau) \\
     &=\sum_{\tau}G_{\tau}\nabla_{\bm{\xi}} P^{\bm{\pi},\bm{p}^{\bm{\xi}}}(\tau)\\
     &= \sum_{\tau} P^{\bm{\pi},\bm{p}^{\bm{\xi}}}(\tau) \cdot G_{\tau}\nabla_{\bm{\xi}} \log P^{\bm{\pi},\bm{p}^{\bm{\xi}}}(\tau).
\end{align*}
We notice that the gradient is estimated by the expectation of $G_{\tau}\nabla_{\bm{\xi}} \log P^{\bm{\pi},\bm{p}^{\bm{\xi}}}(\tau)$ over all trajectories. To compute $\nabla_{\bm{\xi}} \log P^{\bm{\pi},\bm{p}^{\bm{\xi}}}(\tau)$, we have
\begin{align*}
    \nabla_{\bm{\xi}} \log P^{\bm{\pi},\bm{p}^{\bm{\xi}}}(\tau) &= \nabla_{\bm{\xi}} \log \left[\rho(S_{0})\prod_{t=0}^{T-1}\pi(A_{t}|S_{t})\cdot p^{\bm{\xi}}_{S_{t}A_{t}S_{t+1}}\right] \\
    &= \nabla_{\bm{\xi}} \log\rho(S_{0}) + \sum_{t=0}^{T-1}\left[\nabla_{\bm{\xi}}\log\pi(A_{t}|S_{t}) + \nabla_{\bm{\xi}} \log p^{\bm{\xi}}_{S_{t}A_{t}S_{t+1}}\right]\\
    &= \sum_{t=0}^{T-1}\nabla_{\bm{\xi}} \log p^{\bm{\xi}}_{S_{t}A_{t}S_{t+1}}
\end{align*}
In particular, one can approximate the gradient using $m$ sample trajectories without the model
\begin{align*}
    \nabla_{\bm{\xi}} J(\bm{\pi}, \bm{p}^{\bm{\xi}}) &= \mathbb{E}_{\tau\sim P^{\bm{\pi},\bm{p}^{\bm{\xi}}}}\left[G_{\tau} \nabla_{\bm{\xi}} \log P^{\bm{\pi},\bm{p}^{\bm{\xi}}}(\tau)\right]\\
    &\approx \frac{1}{m}\sum_{i=1}^{m}  G_{\tau_{i}} \nabla_{\bm{\xi}} \log P^{\bm{\pi},\bm{p}^{\bm{\xi}}}(\tau_{i})\\
    &=  \frac{1}{m}\sum_{i=1}^{m}  G_{\tau_{i}} \sum_{t=0}^{T_{i}-1}\nabla_{\bm{\xi}} \log p^{\bm{\xi}}_{S_{t}A_{t}S_{t+1}}
\end{align*}
Notice that we have
\begin{align*}
    \nabla_{\bm{\xi}} J(\bm{\pi}, \bm{p}^{\bm{\xi}}) = \nabla_{\bm{\xi}}\mathbb{E}_{\bm{\pi},\bm{p}^{\bm{\xi}}} \left[G_{\tau}\right] &= \mathbb{E}_{\tau\sim P^{\bm{\pi},\bm{p}^{\bm{\xi}}}}\left[G_{\tau} \nabla_{\bm{\xi}} \log P^{\bm{\pi},\bm{p}^{\bm{\xi}}}(\tau)\right] \\
    &= \mathbb{E}_{\tau\sim P^{\bm{\pi},\bm{p}^{\bm{\xi}}}}\left[\left(\sum^{T}_{i=1}\gamma^{i}R_{i}\right)\left(\sum_{t=0}^{T-1}\nabla_{\bm{\xi}} \log p^{\bm{\xi}}_{S_{t}A_{t}S_{t+1}}\right)\right]
\end{align*}
Alternatively, using the same argument, one can show that for any fixed $t'$, we have
\begin{equation*}
    \nabla_{\bm{\xi}} \mathbb{E}_{\bm{\pi},\bm{p}^{\bm{\xi}}}\left[R_{t'}\right] = \mathbb{E}_{\tau\sim P^{\bm{\pi},\bm{p}^{\bm{\xi}}}}\left[R_{t'}\left(\sum_{t=0}^{t'-1}\nabla_{\bm{\xi}} \log p^{\bm{\xi}}_{S_{t}A_{t}S_{t+1}}\right)\right]
\end{equation*}
Then, summing over all $t'$ with discounted factor, we have
\begin{align*}
    \nabla_{\bm{\xi}} \mathbb{E}_{\bm{\pi},\bm{p}^{\bm{\xi}}}\left[G_{\tau}\right] &= \nabla_{\bm{\xi}} \mathbb{E}_{\bm{\pi},\bm{p}^{\bm{\xi}}}\left[\sum^{T}_{i=1}\gamma^{i}R_{i}\right]\\
    &= \sum^{T}_{i=1}\gamma^{i}\cdot\nabla_{\bm{\xi}} \mathbb{E}_{\bm{\pi},\bm{p}^{\bm{\xi}}}\left[R_{i}\right]\\
    &= \mathbb{E}_{\tau\sim P^{\bm{\pi},\bm{p}^{\bm{\xi}}}}\left[\sum^{T}_{t'=1}\gamma^{t'}R_{t'}\left(\sum_{t=0}^{t'-1}\nabla_{\bm{\xi}} \log p^{\bm{\xi}}_{S_{t}A_{t}S_{t+1}}\right)\right]
\end{align*}
Rearranging the terms, we have
\begin{align*}
\mathbb{E}_{\tau\sim P^{\bm{\pi},\bm{p}^{\bm{\xi}}}}\left[\sum^{T}_{t'=1}\gamma^{t'}R_{t'}\left(\sum_{t=0}^{t'-1}\nabla_{\bm{\xi}} \log p^{\bm{\xi}}_{S_{t}A_{t}S_{t+1}}\right)\right] &= \mathbb{E}_{\tau\sim P^{\bm{\pi},\bm{p}^{\bm{\xi}}}}\left[\sum^{T-1}_{t=0}\nabla_{\bm{\xi}} \log p^{\bm{\xi}}_{S_{t}A_{t}S_{t+1}}\left(\sum_{t'=t+1}^{T}\gamma^{t'}R_{t'}\right)\right]\\
&= \mathbb{E}_{\tau\sim P^{\bm{\pi},\bm{p}^{\bm{\xi}}}}\left[\sum^{T-1}_{t=0}\nabla_{\bm{\xi}} \log p^{\bm{\xi}}_{S_{t}A_{t}S_{t+1}}G_{t}\right]
\end{align*}
where $G_{t} = \sum^{T}_{t}\gamma^{i}R_{i}$. Therefore, we have
\begin{equation*}
    \nabla_{\bm{\xi}} J(\bm{\pi}, \bm{p}^{\bm{\xi}}) = \mathbb{E}_{\tau\sim P^{\bm{\pi},\bm{p}^{\bm{\xi}}}}\left[G_{\tau} \nabla_{\bm{\xi}} \log P^{\bm{\xi}}(\tau)\right]  \approx \frac{1}{m}\sum_{i=1}^{m}\sum_{t=0}^{T_{i}-1} \nabla_{\bm{\xi}} \log p^{\bm{\xi}}_{S_{i,t}A_{i,t}S_{i,t+1}}\cdot G_{i,t}
\end{equation*}

\end{document}